\documentclass[twoside,11pt]{article}
\usepackage[preprint]{jmlr2e}
\usepackage[toc,page,header]{appendix}
\usepackage{minitoc}
\usepackage{enumitem}

\usepackage{amsmath,amsfonts,bm}
\usepackage{amssymb}
\usepackage{pifont}
\usepackage{mathtools}
\usepackage[normalem]{ulem}
\usepackage{verbatim}
\usepackage{subcaption}
\usepackage[utf8]{inputenc} 
\usepackage[T1]{fontenc}    
\usepackage{hyperref}       
\usepackage{url}            
\usepackage{booktabs}       
\usepackage{color,graphicx,float,wrapfig}
\usepackage{caption}
\usepackage{nicefrac}       
\usepackage{microtype}  
\usepackage{xcolor}   
\definecolor{darkgreen}{RGB}{0,120,0}
\usepackage{caption}
\usepackage{subcaption}
\usepackage{mathtools}
\usepackage{booktabs}
\usepackage{multirow}
\usepackage{tikz}
\usepackage{soul}

\newcommand{\xmark}{\ding{55}}
\newlength{\NOTskip} 
\def\NOT#1{\settowidth{\NOTskip}{\ensuremath{#1}}%
            \hspace{0.5\NOTskip}\mathclap{\not}\hspace{-0.5\NOTskip}#1}

\newtheorem{assumption}{Assumption}
\newtheorem{prop}[theorem]{Proposition}
\newcommand{\ie}{i.e.}
\newcommand{\eg}{e.g.}
\newcommand{\T}{\theta}

\DeclareMathOperator\supp{supp}
\setcitestyle{numbers,square,comma} 
\newcommand{\real}{\mathbb{R}}

\newcommand{\bT}{\boldsymbol{\theta}}
\newcommand{\bnu}{\boldsymbol{\nu}}
\newcommand{\bDelta}{\boldsymbol{\Delta}}
\newcommand{\bX}{\mathbf{X}}

\newcommand{\Yao}[1]{{\color{darkgreen}{\small\bf\sf [Yao: #1]}}}
\newcommand\sbullet[1][.5]{\mathbin{\vcenter{\hbox{\scalebox{#1}{$\bullet$}}}}}
\usepackage[labelfont=scriptsize, font = scriptsize]{caption}
\jmlrheading{1}{2021}{1-48}{4/00}{10/00}{Yao}{Authors}

\DeclareMathOperator*{\argmin}{arg\,min}
\setcitestyle{authoryear,open={(},close={)}} 
\let\originallesssim\lesssim
\let\originalgtrsim\gtrsim

\DeclareRobustCommand{\lesssim}{%
  \mathrel{\mathpalette\lowersim\originallesssim}%
}
\DeclareRobustCommand{\gtrsim}{%
  \mathrel{\mathpalette\lowersim\originalgtrsim}%
}

\makeatletter
\newcommand{\lowersim}[2]{%
  \sbox\z@{$#1<$}%
  \raisebox{-\dimexpr\height-\ht\z@}{$\m@th#1#2$}%
}
\makeatother
\makeatletter
\renewcommand*{\thanks}[1]{%
  \footnotemark
  \protected@xdef\@thanks{\@thanks
    \protect\footnotetext[\arabic{footnote}]{#1}}%
}
\makeatother

\title{Distributed Sparse Regression via Penalization}

\editor{unknown}
\allowdisplaybreaks
\firstpageno{1}

\begin{document}

\author{\name Yao Ji \email jiyao@purdue.edu \\
       \addr School of Industrial Engineering\\
      Purdue University\\
      West Lafayette, IN 47906, USA
       \AND
      \name Gesualdo Scutari\thanks{Equal contribution.} \email gscutari@purdue.edu \\
       \addr School of Industrial Engineering\\
      Purdue University\\
      West Lafayette, IN 47906, USA
      \AND
       \name Ying Sun$^\ast$ \email ysun@psu.edu \\
       \addr School of Electrical Engineering and Computer Science\\
      The Pennsylvania State University\\
     State College, PA 16802, USA
       \AND
      \name Harsha Honnappa \email honnappa@purdue.edu \\
       \addr School of Industrial Engineering\\
      Purdue University\\
      West Lafayette, IN 47906, USA
      }

\editor{unknown}

\maketitle

\faketableofcontents 
\begin{abstract}
  We study sparse linear regression over a  network of agents, modeled as an undirected graph (with no centralized node). The estimation problem is formulated as the minimization of the sum of the  local LASSO loss functions plus a quadratic penalty of the consensus constraint—the latter being instrumental to obtain  distributed solution methods.   While  penalty-based consensus  methods have been extensively studied in the optimization literature,  their {\it statistical} and computational  guarantees in the   {\it high dimensional} setting remain unclear.  
  This work provides an answer to this open problem. 
 Our contribution is two-fold. First, we establish statistical consistency of the   estimator: under  a suitable choice of the penalty parameter, the optimal solution of the penalized problem achieves 
 {\it near   optimal minimax  {rate  $\mathcal{O}(s \log d/N)$  }} in  $\ell_2$-loss, where $s$ is the sparsity value, $d$ is the ambient dimension, and $N$ is the \emph{total} sample size in the network--this matches centralized sample rates. Second, 
 we show that the proximal-gradient algorithm  applied to the penalized problem, which naturally leads to distributed implementations, converges linearly up to a tolerance of the order of the  centralized statistical error--the rate scales  as  {$\mathcal{O}(d)$}, revealing an unavoidable  speed-accuracy dilemma.   
 Numerical results 
 demonstrate 
 the tightness of the derived sample rate and convergence rate scalings.  
\end{abstract}

\section{Introduction}\label{intro}
We study   high-dimensional   sparse estimation over a network of $m$ agents, modeled as an undirected graph. No centralized agent is assumed in the network; agents can communicate only with their immediate neighbors. Each agent $i$   owns a data set $(y_i, X_i)$, generated according to the linear model
\begin{equation}\label{eq:linear-model}
    y_i = X_i \theta^* + w_i,
\end{equation}
where $y_i\in \mathbb{R}^n$ is the vector of $n$ observations, $X_i \in \mathbb{R}^{n\times d}$ is the design matrix, $w_i \in \real^n$ is observation noise, and $\T^*\in \mathbb{R}^d$ is the unknown $s$-sparse parameter {\it common} to all local models. In the   high-dimensional setting, as postulated here,  
the ambient dimension $d$ is larger than the total sample size $N=n\cdot m$ and $s<<d$.  

A standard approach to estimate $\T^*$ from $\{(y_i, X_i)\}_{i = 1}^m$ is to solve the   LASSO problem, whose Lagrangian form reads  
       \begin{equation}\label{original problem}
    \hat{\theta}\in\argmin\limits_{\theta\in\mathbb{R}^d}\frac{1}{m}\sum\limits_{i=1}^m \frac{1}{2n}\lVert y_i-X_i\theta\rVert^2 +\lambda\|\theta\|_1,
       \end{equation} 
where  $\lambda > 0$  controls the sparsity of the solution $\hat{\theta}$.  Since the objective function involves the entire data  set $\{(y_i, X_i)\}_{i = 1}^m$ across  the network,  and  routing local data to other agents is infeasible (e.g., due to privacy issues) or highly inefficient, 
Problem~(\ref{original problem}) cannot be solved   by each agent $i$ independently. This calls for the design of distributed   algorithms whereby agents  alternate    computations, based on available local information, with communications  with  neighboring nodes.  
To this end,  a widely adopted  approach   is to decompose (\ref{original problem})   by  introducing local estimates $\theta_i$'s of the common variable $\theta$, each one controlled by one agent,  and forcing consensus among  the agents 
(\eg, \citep{Nedic_Olshevsky_Rabbat2018}):
\begin{equation}\label{p:dist-lasso}
        \min_{\bT\in \real^{md}}  \frac{1}{m}\sum\limits_{i=1}^m{\frac{1}{2n}\lVert y_i-X_i\theta_i\rVert^2} + \frac{\lambda}{m} \| \bT\|_1,\quad \textnormal{subject to}\quad   V  \bT = \bf{0}, 
\end{equation} 
where  
$\boldsymbol{\T} = [\theta_1^\top,\ldots,\theta_m^\top]^\top$  is the `stack vector' of all the local copies $\theta_i$'s, and $V$ is a positive semidefinite consensus-enforcing matrix, \ie,  $V  \bT = \bf{0}$ if and only if all $\T_i$'s are equal. 

The objective function in (\ref{p:dist-lasso}) is now  (additively) separable in the agents' variables; 
however, there is still a   coupling across the $\theta_i$'s,   due to the consensus constraint $V\bT=\bf{0}$. To resolve  this coupling, 
{a widely   adopted strategy in the literature of distributed optimization} is to employ an inexact penalization  of 
  the  constraint via a quadratic function. This leads to the following relaxed formulation:  %
\begin{equation}
  \label{Change problem}
   { \hat{\bT}}\in\argmin\limits_{\substack{\bT\in\mathbb{R}^{m d}}}\frac{1}{m}\sum\limits_{i=1}^m  {\frac{1}{2n}\lVert y_i-X_i\theta_i\rVert^2}  + \frac{1}{2m \gamma}\|\bT\|_{V}^2+ \frac{\lambda}{m } \| \bT\|_1, 
\end{equation}
where   $\|\bT\|_{V}^2\triangleq \bT^\top V \bT$, and   $\gamma>0$ is a free parameter 
controlling the violation of the consensus constraint $V\bT=\bf{0}$. Invoking standard results of penalty  methods (see, e.g., \cite{nesterov2018lectures}), it is not difficult to check that, when $\gamma\downarrow 0$,  every limit point of the resulting sequence $\hat{\bT}=\hat{\bT}(\gamma)$ is a solution of (\ref{p:dist-lasso}). This justifies the use of (\ref{Change problem})  as an approximation of (\ref{p:dist-lasso}) (for sufficiently small $\gamma$).

Problem~(\ref{Change problem}) unlocks distributed solution methods.  Here, we consider the  proximal gradient algorithm  \citep{nesterov2018lectures} that,  
based upon a suitable choice of the matrix $V$, is readily implementable over the network. 
This resembles 
the renowned Distributed Gradient Descent algorithms (DGD) (see~Sec.~\ref{sec:literature_review}), {which are among the most studied distributed schemes in the literature}. { Motivated by the popularity of the penalized formulation (\ref{Change problem}) and associated DGD algorithms,  the goal of this paper is to study the statistical properties of the estimator (\ref{Change problem}) as well as computational guarantees of the aforementioned    DGD algorithm.}

\begin{figure}[ht!]
 \centering
     \includegraphics[height = 7cm]{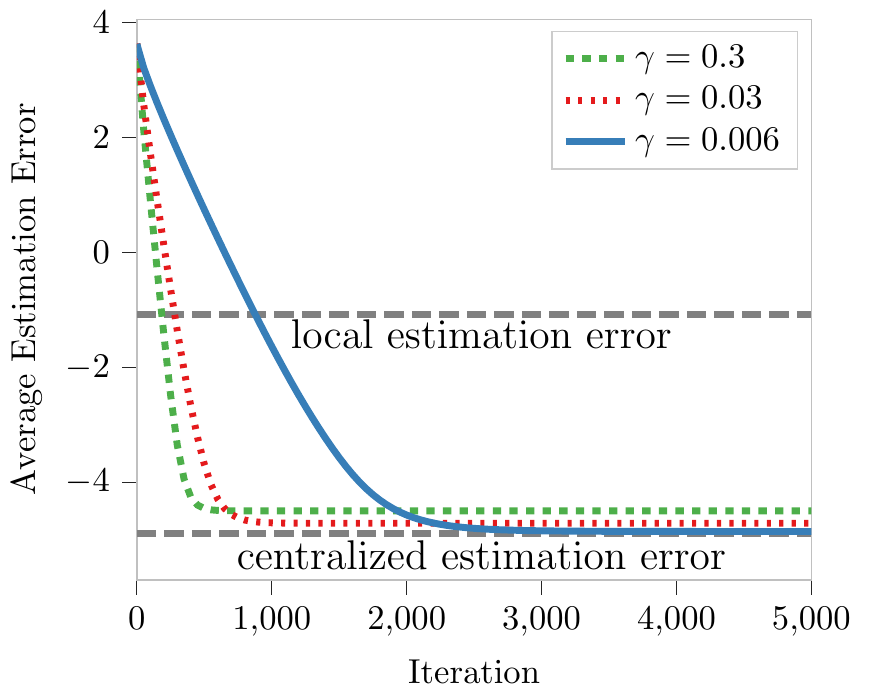}
     \caption{Proximal gradient in the high-dimensional setting (\ref{Change problem}): linear convergence up to some tolerance; different curves refer to different values of the penalty parameter $\gamma$.  Notice the speed-accuracy dilemma. \label{fig:DGD}}\vspace{-0.3cm}
\end{figure} 
\subsection{Challenges and open problems}  While penalty-based  formulations like~(\ref{Change problem}) and related solution methods 
have been extensively studied in the optimization literature,  
the statistical properties    of the solution $\hat{\bT}$  in the {\it high-dimensional} setting ($d>>N$) remain unknown, and so are the  convergence guarantees of the proximal gradient algorithm applied to~(\ref{Change problem}). Postponing   to  Sec.~\ref{sec:literature_review}   a detailed review of the literature, here we point out the following. {\bf  Statistics:}  classical 
sample complexity analysis  of LASSO error $\|\hat{\T}-\T^*\|^2$ for~(\ref{original problem}) (\eg,  \citep{Wainwright-book}) is not directly applicable to the penalized problem (\ref{Change problem})--for instance, it is unclear whether each agent's  error   $\|\hat{\T}_i-\T^*\|^2$  
can match centralized sample complexity. 
{\bf Distributed optimization:}  When it comes to algorithms for solving~(\ref{Change problem}),   existing studies are of pure optimization type, lacking of statistical guarantees. 
{If  nevertheless   invoked to predict   convergence of 
 the proximal  gradient algorithm applied to (\ref{Change problem}), they  
would  certify  {\it sublinear} convergence of the optimization error, since the objective function in (\ref{Change problem}) is not strongly convex on the entire space (recall $d>N$). }
This results in a pessimistic prediction,  as shown by the exploratory experiment in  Fig.~\ref{fig:DGD}:  
the average estimation error   $(1/m)\cdot \sum_{i=1}^m  \|\T_i^t-\T^* \|^2$      decreases {\it linearly} up to a tolerance (floor);  different curves refer to different values of the penalty parameter $\gamma$. The figure also plots the (square) estimation error achieved  {by} solving \eqref{original problem}--termed  {as} centralized estimation error--and the average of the (square) estimation  errors achieved by each agent solving the LASSO problem using only its local data--termed as local estimation error.  The experiment seems to suggest that statistical error comparable to centralized ones are still achievable where data are distributed over a network. However this requires a sufficiently small $\gamma$, and thus results in  slow convergence rates.   

To the best of our knowledge,   a theoretical understanding of these phenomena remains an open problem;  questions are abundant, such as: (i) Is  centralized statistical consistency  (quantified by sample complexity $N = o(s \log d )$) provably achievable when data are distributed across the network? What is the role/impact of the network? (ii) Is it feasible for the distributed proximal gradient method to yield statistically optimal solutions while maintaining linear convergence? (iii)  How do sample and convergence rates of the algorithm interact with model   parameters, specifically $\gamma$, $d$, $N$, and network configurations?     

 
\sloppy
\subsection{Major contributions}  \label{sec:contributions}  This work addresses the above questions--our contributions can be summarized as follows. \begin{itemize}
    \item[\bf 1)] { \bf Statistical analysis of the penalized LASSO problem  \eqref{Change problem}: } 
  We establish non-asymptotic error bounds on the estimation error   averaged over the agents, $(1/m) \sum_{i  = 1}^m \|\hat{\T}_i-\T^*\|^2$, under proper tuning  of $\lambda$ and $\gamma$. 
  Our results are of two types. {\bf (i)} A    deterministic    bound, 
  under   a strong convexity requirement on the objective in (\ref{Change problem}) restricted to certain directions containing    the augmented LASSO error $\hat{\bT}-1_m\otimes\T^*$ (cf.~Theorem~\ref{solution err bound})--this bound sheds light on the  role of the network and consensus errors (via $\gamma$) into the estimation process; and  
{\bf (ii)} 
  a (sample) convergence rate    $(1/m) \sum_{i  = 1}^m \|\hat{\T}_i-\T^*\|^2=\mathcal{O}(s\log d/N)$ 
  (cf.~Theorem~\ref{statistical optimization error result}), which holds with high probability (w.h.p.) under standard   Gaussian data generation models (cf.~Assumption~\ref{Random Gaussian model}). 
  This
   matches  the  statistical error   of the   LASSO estimator (\ref{original problem}), thereby unveiling for the first time that   statistical consistency over networks is feasible under a similar order of sample size $N$ as employed in the centralized setting, {\it even when the number of local samples $n$ is insufficient}. 
   \item[\bf 2)] {\bf Algorithmic  guarantees:}  To  compute such estimators in a distributed fashion,  we leverage the   proximal-gradient algorithm applied to  (\ref{Change problem}), and   study its convergence and statistical properties  
   (cf.~Theorems~\ref{deterministic optimization error result} and ~\ref{Th_stat_convergence}).   
    A major result is proving that, in the setting (ii) above,
    the algorithm  converges   {\it linearly} up to a fixed tolerance  which can be  driven below  the statistical precision of the centralized LASSO problem (\ref{original problem}). Specifically,  to enter an $\varepsilon$-neighborhood of a statistically optimal solution, it takes $$\mathcal{O} \left( \frac{1}{1-\rho}\cdot\frac{\lambda_{\max}(\Sigma)}{\lambda_{\min}(\Sigma)} \cdot d\,  {m}\,\log m \cdot \log \frac{1}{\varepsilon} \right)$$ number of communications (iterations), where $\rho\in [0,1)$ is a measure of the connectivity of the network (the smaller $\rho$ is, the more connected the graph); and  $\lambda_{\max}(\Sigma)/\lambda_{\min}(\Sigma)$ is  the restricted condition number of the LASSO loss function   [see \eqref{original problem}], with $\lambda_{\max}$ (resp. $\lambda_{\min}$) denoting the largest (resp. smallest) eigenvalue of the covariance matrix $\Sigma$ of the data (cf.~Assumption~\ref{Random Gaussian model}). This shows that centralized statistical accuracy  is achievable over a given network (without moving data) but at the price  of a linear rate (and thus communication cost) that scales as  $\mathcal{O}(d)$. This `speed-accuracy dilemma' is confirmed by our experiments (cf.~Sec.~\ref{Numerical Experiments}). A similar phenomenon has been observed previously in low-dimensional settings (strongly convex losses) {\citep[Theorem 3]{Yuan_2016}}.  However, our results demonstrating this dilemma in the high-dimensional setting as well, imply that this appears to be a ``scarlet letter'' of DGD-like algorithms.
\end{itemize}

   
\subsection{Related works}\label{sec:literature_review}
 
{\bf Statistical analysis:} Statistical properties of  the LASSO solution $\hat{\T}$ of   (\ref{original problem}) along with  several other regularized  M-estimators 
have been extensively studied in the literature (see, \eg, \citep{Tibshirani96,10.5555/2834535,Wainwright-book}  and references therein). 
Introducing  suitably   restricted notions of  strong convexity of the loss--\eg, \citep{10.1214/08-AOS620,4016283,Buhlmann09,Negahban_2012,Wainwright-book}--(nonasymptotic) error bounds and sample  complexity   for such estimators 
under high-dimensional scaling are established. For instance, for the LASSO estimator  (\ref{original problem}),   statistical errors  read  $\|\hat{\T}-\T^*\|^2=\mathcal{O}(s\log d/N)$. 

These conditions and results for (\ref{original problem}) 
do not transfer directly to the  {\it lifted, penalized} formulation (\ref{Change problem})--it is not even clear the relation between $\hat{\T}$ [cf.~(\ref{original problem})] and   $\hat{\T}_1,\ldots, \hat{\T}_m$ [cf.~(\ref{Change problem})]. A new solution and statistical analysis is needed for the ``augmented'' LASSO estimator $\hat{\bT}$ (\ref{Change problem}), possibly revealing the role of the network on the statistical properties of $\hat{\bT}$. \smallskip 

\noindent  {\bf Centralized optimization algorithms:} Referring to solution methods  
for {\it centralized} sparse linear regression problems, 
  several studies are available in the literature,   including \citep{BeckerCandesNESTA11,BeckTeboulleFISTA,Bredies08,Hale08,TsengYun09,ZhouSo17,Wen17,Bolte09}   and \citep{agarwal2012fast}. Since (\ref{original problem}) is not strongly convex in a global sense, classical (accelerated) first-order methods like  \citep{BeckerCandesNESTA11,BeckTeboulleFISTA}  are known to converge at sublinear rate; others \citep{Bredies08,Hale08}  are proved to achieve  linear convergence if initialized in a neighborhood of the solution of (\ref{original problem}); and  \citep{TsengYun09,ZhouSo17,Wen17,Bolte09}  showed  linear convergence (in particular) of the proximal-gradient algorithm, invoking global regularity conditions of the loss   (\ref{original problem}), such as  the  Luo-Tseng's bound \citep{LuoTseng93} or the  KL property \citep{Bolte09,Pan18}. These studies are of pure optimization type--
 \eg, convergence  focuses on   iteration complexity of the optimization error, no statistical analysis of the limit points is provided. Furthermore,  they are not suitable for the high-dimensional regime (\ie, ``$d,N$ growing''). 
 A closer related work is \citep{agarwal2012fast}, which establishes global linear convergence of the proximal-gradient algorithm  for  (\ref{original problem}) up to the statistical precision of the model, under a restricted strong convexity (RSC) and  restricted smoothness (RSM) assumption. {The method is not directly implementable over mesh networks, because of the lack of a centralized node. Furthermore, 
   it is unclear whether {RSC/RSM} conditions hold for the penalized sum-loss in (\ref{Change problem}). On the other hand, a naive application of the RSC/RSM to      {\it each} agent's loss  $f_i(\theta_i)=(1/2n) \|y_i-X_i \theta_i\|^2$ in (\ref{Change problem}) (without accounting for the penalty, coupling term $(1/{\gamma})\|\bT\|^2_V$),   would require a local sample scaling   $n=\mathcal{O}(s\log d)$ to hold.} 
  This conclusion is  unsatisfactory because it would state that the centralized minimax error bound $\|\hat{\T}-\T^*\|^2=\mathcal{O}(s\log d/N)$ is not achievable over networks--a fact that is confuted by our theoretical findings and  experiments.\smallskip 
  
 \noindent  {\bf Divide and Conquer (D\&C) methods: } { When it comes to decomposition methods for statistical estimation and inference, the statistical community   is best acquainted with   D\&C methods. D\&C algorithms postulate the existence of a node  in the network  (a.k.a. {\it master} node)  connected to all the others (termed {\it worker} nodes), which  combines  the  estimators produced by each worker using its local data set.    
   D\&C algorithms for   $M$-estimation in   {\it low-dimension}, covering  the asymptotics   $d,N\rightarrow\infty$ while $d/N\rightarrow c\in[0,1)$,    have been extensively studied in the literature; representative examples include    \cite{rosenblatt2016optimality,wang2018giant,chen2021first,Bao2021OneRoundCE,Fan2021}. More relevant to this work are the 
  D\&C methods applicable to sparse linear regression in  {\it high-dimension}, i.e.,  $d>N$  and $d/N\to \infty$,  which include \cite{lee2015communication,Fan2018,wang2017efficient,jordan2018communication}. \cite{lee2015communication,Fan2018} devised a one-shot approach averaging  at the master node ``debiased'' local LASSO estimators.  \cite{wang2017efficient,jordan2018communication} independently  improved the sample complexity of \cite{lee2015communication} hinging on ideas from    \cite{shamir2014communication}--Table~\ref{table: local sample condition1} provides the sample and communication complexity of these methods, which can be summarized as follows.
  By performing a single round of communication from the workers to the master node, resulting in a $\mathcal{O}(d)$  communication cost, these algorithms achieve the centralized statistical error $\mathcal{O}(s\log d/N)$  as long as the local sample size is sufficiently large, i.e.,      $n=\Omega(ms^2\log d)$  
  (see Table.~\ref{table: local sample condition1}). Alternatively, for fixed $n$, this imposes a constraint on the maximum number of workers, \ie, $m=\mathcal{O}(n/(s^2\log d)),$ which limits the range of applicability of these methods to small-size (star) networks. 
  The dependence of $n$ on $m$   can be removed    at the cost of multiple  communication rounds; to our knowledge, the state of the art is \cite{wang2017efficient} showing that $n=\Omega(s^2\log d)$ suffices under $\log m$ communication rounds, resulting in  a total communication cost of  $\mathcal{O}(d\log m)$. 
  None of these methods is directly implementable over mesh networks, because of the lack of a centralized node. Naive attempts of decentralizing D\&C  methods over mesh networks by replacing the exact average at the master node with local consensus updates fail to achieve centralized statistical consistency.}

    
  \begin{table}[h!]
  \begin{center}\resizebox{0.8\columnwidth}{!}{
 \begin{tabular}{c|c|c}
\hline
\multirow{2}{*}{D\&C Methods}          & $n\gtrsim m\,s^2\log d$  & $  m\,s^2\log d \gtrsim   n \gtrsim s^2\log d$ \\ \cline{2-3} 
                                       & Communication   Cost  (one round)      & Communication Cost (multiple rounds)                               \\ \hline
Avg-Debias \cite{lee2015communication} & $d$                    & \xmark                                       \\ \hline
\cite{Fan2018}                         & $d$                    & \xmark                                       \\ \hline

CSL  \cite{jordan2018communication}    & $d$                    &   \xmark$^{1}$                         \\ \hline
EDSL \cite{wang2017efficient}          & $d$                    & $d\log m $                                   \\ \hline
\end{tabular}}
\caption{D\&C algorithms   for sparse linear regression in the  high-dimensional, $d>N$  and $d/N\to \infty$: local sample size and communication cost to achieve  the centralized statistical error $\mathcal{O}(s\log d/N).$ For a single    communication round, all methods require a condition on the minimum  local sample size $n;$ multiple communication rounds can reduce the condition on local sample size $n.$  $^{1}$CSL \cite{jordan2018communication} can be extended to multiple rounds of communication to reduce the local sample size using the similar argument as in EDSL \cite{wang2017efficient}. }\label{table: local sample condition1}\vspace{-0.3cm}
\end{center}

\end{table}
  {In contrast to D\&C methods, 
  the DGD-like algorithm  studied in this paper  to solve (\ref{Change problem})   provably  achieves (near) optimal minimax rates with {\it no conditions on the local sample size},    at a total communication cost however of $\mathcal{O}(d^2)$.    This raises the question  whether communication costs of $\mathcal{O}(d)$ are achievable in high-dimension over mesh networks by other distributed optimization algorithms, yet with no conditions on the local sample size.   Motivated by this work, the study of other methods in high-dimension is the subject of current investigation; see, e.g., the companion work \cite{NetLASSO}. In fact, as discussed next, there is no study of any other existing distributed algorithm in high-dimension.}\smallskip 
  
 \noindent \textbf{Distributed optimization algorithms:}   Solving the LASSO problem~\eqref{original problem}  over  mesh networks falls under the umbrella of  distributed optimization.  The literature of distributed optimization methods is vast;   given the focus of the paper,  we comment next only relevant   works on decentralization of the (proximal) gradient method  over mesh networks modeled as undirected graphs.   
Distributed Gradient Descent (DGD) algorithms, including those   derived by penalizing consensus constraints as in (\ref{Change problem}), have been extensively studied in the literature; see, e.g.,   \citep{Nedic2009,Nedic2010,Chen-Sayed,Sayed,chen2012fast,Yuan_2016,WYin_ncvxDGD_SIAM2018,Daneshmand20,Nedic_Olshevsky_Rabbat2018}.  
Among all, the most relevant distributed scheme to this paper is \citep{WYin_ncvxDGD_SIAM2018}, a proximal gradient algorithm. When  applied to (\ref{Change problem}), under the additional assumption of bounded (sub)gradient of  the  agents' losses (a fact that is not guaranteed),  {\it sublinear} convergence (on the objective value) to the optimal solutions of (\ref{Change problem})   would be certified (recall that agents' losses  are not strongly convex globally). Furthermore, the connection between the solution of  the penalized problem  (\ref{Change problem}) and that of the LASSO formulation (\ref{original problem}) remains unclear.  

{While different and not derived   directly from (\ref{Change problem}), the other   DGD-like algorithms can be roughly commented as  follows:   
(i)  when the agents' loss functions are strongly convex (or the centralized loss satisfies  the KL property \citep{WYin_ncvxDGD_SIAM2018, Daneshmand20}), differentiable, and  there are no constraints,  DGD-like schemes   equipped  with a constant stepsize,  converge (only) to  a   neighborhood of the solution at linear rate  {\citep{Yuan_2016,WYin_ncvxDGD_SIAM2018,Yuan2020}}. Convergence (in objective value) to the exact solution   is achieved only using diminishing stepsize rules, thus  at the   slower   sublinear rate  (see, e.g.,   \citep{WYin_ncvxDGD_SIAM2018, Jakoveti2014FastDG}).} This speed-accuracy dilemma can be overcomed by correcting explicitly the local gradient direction so that a constant stepsize can be used still preserving convergence to the exact solution; examples include: gradient tracking methods \citep{qu2016harnessing,nedich2016achieving,Xu-TAC:hs,LorenzoScutari-J'16,sun2019distributed}  and primal-dual schemes  \citep{jakovetic2011cooperative,shi2014linear,Jakovetic:da,6926737,shi2015extra,shi2015proximal}, just to name a few.\smallskip  

The above review of the literature 
shows  that there is no  study of  
  statistical/computational guarantees  in the {\it high-dimensional} regime. Our comments   on  centralized optimization algorithms apply here for all the aforementioned distributed ones: 
  all the convergence results are of pure optimization type and    are confuted by our experiments (see Fig.~\ref{fig:DGD}). {A new analysis is needed to understand the behaviour of distributed algorithms in  the high-dimensional regime.  This paper represents the first study of a DGD-like algorithm   towards this direction. }
 

  \smallskip

\subsection{Notation and paper organization} \label{notation}
%
The rest of the paper is organized as follows. Sec.~\ref{Setup and Background}  introduces the   assumptions on the data model     
 and network  
  along with some   consequences. Solution analysis of the penalized LASSO (\ref{Change problem}) is addressed in Sec.~\ref{sec:sol-analysis}--a deterministic error bound,  based on a notion of restricted strong convexity, is first established; then  {near} optimal centralized  sample complexity is proved under standard data generation models (Sec.~\ref{Bounds On the l2 Error for the the Distributed Lasso}).  The  (distributed) proximal gradient algorithms applied to (\ref{Change problem}) is studied  in Sec.~\ref{Geometric Convergence of Distributed Gradient Descent Based Algorithms}.   
  Finally, Sec.~\ref{Numerical Experiments} provides some experiments validating our theoretical findings while Sec.~\ref{sec_conclusions} draws some conclusions. All the proofs of the presented results are relegated to the appendix.   \smallskip 

\noindent \textbf{Notation:} Let $[m]\triangleq\{1,\dots,m\}$, with $m\in \mathbb{N}_{++}$; $1$ is the 
    vector of all ones; $e_i$ is the $i$-th canonical vector;      $I_d$ is the $d\times d$ identity matrix (when unnecessary, we  omit the subscript); and $\otimes$ denotes the Kronecker product. Given $x_1,\ldots, x_m \in \mathbb{R}^d$,  the bold symbol       $\mathbf{x}=[x_1^{\top}, \dots, x_m^{\top}]^{\top} \in \mathbb{R}^{md}$ denotes the stack vector; for any   $\mathbf{x}=[x_1^{\top}, \dots, x_m^{\top}]^{\top}$, 
  we define its block-average as $x_{\textnormal{av}}\triangleq ({1}/{m})\sum_{i=1}^m x_i$, and the disagreement vector  {$\mathbf{x}_{\perp}\triangleq[x_{\perp 1}^{\top}, \dots, x_{\perp m}^{\top}]^{\top}$}, with each   $x_{\perp i}\triangleq x_i-x_{\textnormal{av}}$. Similarly, given the matrices $X_1, \ldots, X_m\in \mathbb{R}^{n\times d}$, we use bold notation for the stacked matrix  { $\mathbf{X}=[X_1^\top,\ldots, X^\top_m]^\top$}. We order the eigenvalues of any symmetric matrix $A\in \mathbb{R}^{m\times m}$ in nonincreasing fashion, \ie, $\lambda_{\max}(A)=\lambda_1(A)\geq \ldots \geq  \lambda_m(A)=\lambda_{\min}(A)$. 
  We use $\|\cdot\|$ to denote  the Euclidean norm; when other norms are used, \eg, $\ell_1$-norm  and $\ell_{\infty}$, we will append the associate subscript to $\|\cdot\|$, such as $\|\cdot\|_1,$ and $\|\cdot\|_{\infty}$. Consistently, when applied to matrices,   $\|\cdot\|$ denotes the operator norm  induced by $\|\cdot\|.$  Furthermore, we write $\|x\|_{A} \triangleq(x^{\top}Ax)^{1/2}$, for any symmetric positive semidefinite matrix.   Given   $\mathcal{S}\subseteq[d]$ and  $y\in\mathbb{R}^{d}$, we denote by     $|\mathcal{S}|$  the cardinality of $\mathcal {S}$ and by $y_{\mathcal {S}}$ the $|\mathcal{S}|$-dimensional vector containing the entries of  $y$ indexed by the elements of $\mathcal {S}$; $\mathcal{S}^c$ is the   complement of $\mathcal{S}$. All the $\log$ in the paper are intended    natural logarithms, unless otherwise stated. Given two univariate random variables $X$ and $Y$, we say that $Y$ has stochastic dominance over $X$ if  $X \preceq^{\text{st}} Y$, meaning    $\mathbb{P}(X\leq t)\geq\mathbb{P}(Y\leq t),$ for all $t\in \mathbb{R}$ \citep[p.~694]{marshall11}. 
  
 \section{Setup and Background}\label{Setup and Background} In this section we introduce the main assumptions on the data model     
 and network setting 
 underlying our analysis along with some related consequences. 
 
\subsection{Problem setting} 
The following quantities associated with  (\ref{eq:linear-model}) will be used throughout  the paper:  \begin{equation}\label{L max}
  \mathcal{S} \triangleq\supp\{\T^*\},\quad s=|\mathcal{S}|,\quad {L_{\max}}\triangleq\max\limits_{  i\in [m]}\lambda_{\max}\bigg(\frac{X_i^{\top} X_i}{n}\bigg). 
\end{equation} We collect all the local data $\{(y_i,X_i)\}_{i=1}^m$ into the stacked vector measures $\mathbf{y}=[y_1^\top,\ldots, y_m^\top]^\top\in \mathbb{R}^N$ and matrix  $\mathbf{X}=[X_1^\top,\ldots, X_m^\top]^\top\in \mathbb{R}^{N\times d}$. The quadratic losses of the centralized LASSO problem (\ref{original problem}) and of the penalized one (\ref{Change problem}) are denoted respectively by 
\begin{equation}\label{eq:losses}
    F(\T)\triangleq \frac{1}{2N} \|\mathbf{y}-\mathbf{X}\T\|^2\quad \text{and} \quad   {L_\gamma}(\bT)\triangleq \frac{1}{2N}\sum\limits_{i=1}^m  {\underbrace{\lVert y_i-X_i\theta_i\rVert^2}_{{\triangleq f_i(\T_i)}}}  + \frac{1}{2m \gamma}\|\bT\|_{V}^2.
\end{equation} 
We recall next the main path/assumptions used to bound the LASSO error $\|\hat{\T}-\T^*\|^2$ in the centralized setting  (\ref{original problem}) (\eg, \citep{Wainwright-book}). In the regime $d>N$, $F$ is not strongly convex--the $d\times d$ Hessian matrix   $\mathbf{X}^\top\mathbf{X}$ has at most rank $N$. 
Nevertheless,      $\|\hat{\T}-\T^*\|^2$ can be well-controlled requiring  strong convexity of $F$ to hold along a subset of directions. The {\it Restricted Eigenvalue } (RE) condition suffices \citep{10.1214/08-AOS620,4016283,Wainwright-book}
 \begin{equation}\label{RE}
     \frac{1}{N}\lVert \mathbf{X}\Delta\rVert^2\geq \delta_c\lVert\Delta\rVert^2, \quad \forall \Delta\in \mathbb{C}(\mathcal{S})\triangleq \{\Delta\in\mathbb{R}^d\text{ }|\text{ }\|\Delta_{\mathcal{S}^c}\|_1\leq 3\|\Delta_{\mathcal{S}}\|_1\},
 \end{equation}
 where $\delta_c>0$ is the curvature parameter, and $\mathbb{C}(\mathcal{S})$ captures the set of ``sparse'' directions of interests. The rationale  behind (\ref{RE}) is that,  since     $\hat{\T}-\T^*$ can be proved to belong to  $\mathbb{C}(\mathcal{S})$, if   $F$ is strongly convex  on $\mathbb{C}(\mathcal{S})$--as requested by (\ref{RE})--
 then small differences on the loss will translate into bounds on $\|\hat{\T}-\T^*\|^2$.
 
  The RE  (\ref{RE}) imposes conditions on the design matrix $\mathbf{X}$. The following RSC 
  implies (\ref{RE}). 
\begin{lemma}\label{RSC-centralized-deterministic}
Suppose that   $F$ satisfies the following RSC condition with curvature $\mu>0$  and tolerance $\tau>0$:  \begin{equation}\label{Arsc}
    \frac{1}{N}\lVert \mathbf{X}\Delta\rVert^2\geq \frac{\mu}{2}\lVert\Delta\rVert^2-\frac{\tau}{2}\lVert\Delta\rVert_1^2, \quad \forall \Delta\in \mathbb{R}^d.
\end{equation} 
Under $\mu/2-16s\tau>0,$ the RE (\ref{RE}) holds with $\delta_c=\mu/2-16s\tau$.
\end{lemma}
 The practical utility of the RSC condition (\ref{Arsc}) vs. the RE  is that  it can be certified with high probability by a variety of random design matrices $\bX$. Here we consider the following.  
 \begin{assumption}[Random Gaussian model]\label{Random Gaussian model} \label{ass_X_random} The design matrix $\mathbf{X}\in \mathbb{R}^{N\times d}$  satisfies the following: \textnormal{(i)}  the rows of $\mathbf{X}$ are i.i.d. $\mathcal{N}(0,\Sigma)$; and \textnormal{(ii)}   $\Sigma$ is positive definite, with minimum eigenvalue $\lambda_{\min}(\Sigma)>0$.
  \end{assumption}

\begin{lemma}[{\citep[Theorem 1]{Ras}}]\label{Global RE}
Let  $\mathbf{X}\in \mathbb{R}^{N\times d}$  be a design matrix satisfying Assumption~\ref{ass_X_random}. Then, 
there exist universal  constants $c_0, c_1>0$ such that,  with probability at least $1-\exp(-c_0N),$
 the RSC condition (\ref{Arsc}) holds   with   parameters \vspace{-.3cm}\begin{equation}\label{zeta}
     \mu= \lambda_{\min}(\Sigma) \quad \text{and}\quad \tau=2c_{1} \zeta_{\Sigma}\frac{\log d}{N}, \quad \text{with}\,\,\,\zeta_{\Sigma}\triangleq\max\limits_{i\in [d]}\Sigma_{ii}.
 \end{equation}
\end{lemma}
 
\subsection{Network setting} 
 We model the network of $m$ agents  as an undirected graph $\mathcal{G}=(\mathcal{V},\mathcal{E}),$ where $\mathcal{V}=[m]$ is the set of agents, and $\mathcal{E}$ is the set of the edges; $\{i,j\}\in\mathcal{E}$ if and only if there is a communication link between agent $i$ and agent $j.$ We make the   blanket assumption that $\mathcal G$ is connected,  which is necessary for the   convergence of distributed algorithms to a consensual solution.  
 
To solve (\ref{Change problem}) 
over $\mathcal{G}$ via gradient descent, each agent should be able to compute the gradient of the objective (w.r.t. its own local variable $\T_i$) using only information from its immediate neighbours. This imposes some extra conditions on the sparsity pattern of the matrix  $V$. We will use the following  widely adopted structure for $V$.  
\begin{assumption}\label{W}
The matrix   $V=(I_m-W)\otimes I_d$, where $W\triangleq (w_{i j})_{i,j=1}^m$ satisfies the following: \begin{itemize}
    \item[(a)] 
       It is compliant with $\mathcal{G},$ that is, (i) $w_{ii}>0, \forall i\in [m];$ 
(ii) $w_{i j}>0,$ if $\{i,j\}\in \mathcal{E};$ and (iii)  $w_{i j}=0$ otherwise; and 
 \item[(b)]
  It is symmetric and   stochastic, that is,    $W1=1$ (and thus also  $1^{\top}W=1^{\top}$).\end{itemize} 
\end{assumption}

  It follows from the connectivity of $\mathcal{G}$ and   Assumption~\ref{W} that $$V \bT=\mathbf{0} \quad \text{iff} \quad \T_i=\T_j, \,\forall i\neq j\in [m],$$ and 
\begin{equation}\label{connectivity number}  \rho\triangleq\max\{\rvert\lambda_2(W)\rvert, \lvert\lambda_{\min}(W)\rvert\}<1. 
 \end{equation} 
 Roughly speaking, $\rho$  measures  how fast the network mixes information (the smaller, the faster). If $\mathcal{G}$ is   complete  graph or   a star, {one can choose} $W=1 1^\top/m$, 
 resulting in  $\rho=0$.

\section{Solution Analysis and Statistical Guarantees }\label{sec:sol-analysis}
 This section presents the solution analysis of the penalized LASSO problem  (\ref{Change problem}),  establishing   nonasymptotic   bounds of $(1/m)\sum_{i=1}^m\|\hat{\T}_i-\T^*\|^2$. Our study builds on the following  steps. \begin{itemize}
     \item[\bf 1)]  We first determine a suitable restricted set of directions $\mathbb{C}_{\gamma}(\mathcal S)$ [cf.~(\ref{nu ave sparsity})] which contains   the augmented LASSO error $\hat{\bT}-1_m\otimes\T^*$ under certain conditions on the sparsity-enhancing parameter $\lambda$ [cf.~Proposition~\ref{proximal nu average sparsity}]--the set $\mathbb{C}_{\gamma}(\mathcal S)$ plays similar role as  $\mathbb{C}(\mathcal S)$ [cf.~\eqref{RE}] for the centralized LASSO (\ref{original problem}), and sheds light on the role of the penalty parameter $\gamma$ (and thus the consensus errors) on the sparsity pattern of $\hat{\bT}$;
     \item[\bf 2)]  We then  determine a RSC-like condition [cf.~(\ref{ARE eq})] ensuring that,   under a suitable choice of   $\gamma$ controlling the consensus error, the subset $\mathbb{C}_{\gamma}(\mathcal S)$ is well-aligned with the curved directions of the  loss $L_\gamma$ of (\ref{Change problem});
     \item[\bf 3)]  Results in the previous steps  will translate into bounds on $(1/m)\sum_{i=1}^m\|\hat{\T}_i-\T^*\|^2$ [cf. Theorem~\ref{solution err bound}]. Quite interesting, our RSC condition  
 holds w.h.p. under the random model in Assumption~\ref{Random Gaussian model}  (cf.~Lemma~\ref{Global ARE}), which   yields    {\it centralized} sample complexity $(1/m)\sum_{i=1}^m\|\hat{\T}_i-\T^*\|^2=\mathcal{O}(s\log d/N)$   (cf.~Theorem~\ref{statistical optimization error result}).
 \end{itemize}

\subsection{The set of (almost) sparse average   directions}\label{sec:sparse_directions}

 For each given $ \gamma\in(0,1),$ define the set 
    \begin{equation}\label{nu ave sparsity}
        \mathbb{C}_{\gamma}(\mathcal{S})\triangleq\{\bDelta\in\mathbb{R}^{md}\text{ }|\text{ }\lVert (\Delta_{\textnormal{av}})_{\mathcal{S}^c}\rVert_1\leq3\lVert (\Delta_{\text{av}})_{\mathcal{S}}\rVert_1+h(\gamma,\lVert\bDelta_{\perp}\rVert)\},
    \end{equation}
    where 
    \begin{equation}\label{h}
       h(\gamma,\lVert\bDelta_{\perp}\rVert)
        \triangleq  
        -\frac{1-\rho}{m\gamma\lambda}\lVert\bDelta_{\perp }\rVert^2+\Big(2\max\limits_{i\in [m]}\lVert w_i^{\top} X_i\rVert_{\infty}/(\lambda n)+2\Big) \sqrt{d/m}\lVert\bDelta_{\perp}\rVert.
    \end{equation}
    The maximum of  $h(\gamma,\sbullet)$  is a decreasing function of $\gamma>0.$
   This    suggests that, the sparsity   of the average component $\Delta_{\text{av}}$  of directions  $\bDelta\in\mathbb{C}_{\gamma}(\mathcal{S})$ can be   controlled   by $\gamma$; in particular,  by decreasing $\gamma$ one can make   $\Delta_{\text{av}}$ arbitrary ``close'' to the cone $\mathbb{C}(\mathcal{S})$ of sparse directions of the centralized LASSO   (\ref{original problem}) [cf.~\eqref{RE}]. 
  The importance of $\mathbb{C}_{\gamma}(\mathcal{S})$  is captured by the following result. 

    \begin{prop}\label{proximal nu average sparsity}
    Under Assumption~\ref{W} and $\lambda$ satisfying
    \begin{equation}\label{distributed lambda}
    \frac{2}{N}\lVert \mathbf{X}^{\top} \mathbf{w}\rVert_{\infty}\leq \lambda,
    \end{equation}
    the augmented LASSO  error   $\hat{\bT}-1_m\otimes\T^*$ lies in  $\mathbb{C}_{\gamma}(\mathcal{S}).$
    \end{prop}
     \begin{proof}
  See Appendix~\ref{Proof of Proposition 1}.
  \end{proof}
Therefore, the  average component of the augmented LASSO  error   
is nearly sparse for sufficiently small $\gamma$ and large $\lambda$. This   will be used to pursue statistically optimal estimates. 

\subsection{In-network  RE condition}\label{In-network  RE condition} We   impose a positive curvature on the loss $  {L_\gamma}$  of  (\ref{Change problem}) [cf.~(\ref{eq:losses})] along   suitable chosen directions in  $\mathbb{C}_{\gamma}(\mathcal{S}).$ The first-order Taylor   expansion of $  {L_\gamma}$ at
 $\bT'$ along   $\bT-\bT'$, denoted by $\mathcal{T}_{  {L_\gamma}}({\bT};\bT')$, can be lower bounded as
\begin{align}\label{def:L-linearization}
    \mathcal{T}_{  {L_\gamma}}({\bT};\bT')&\triangleq     {L_\gamma}(\bT)-  {L_\gamma}(\bT')-\langle\nabla   {L_\gamma}(\bT'), \bT-\bT'\rangle\notag\\
    &\geq  \underbrace{\frac{1}{4}\frac{\|\mathbf{X}(\bT-\bT')_{\text{av}}\|^2}{N}}_{\text{curvature along average}}-\underbrace{\bigg(\frac{L_{\max}}{2m}-\frac{1-\rho}{2m\gamma}\bigg)\|(\bT-\bT')_{\perp}\|^2}_{\text{nonconsensual  component}}. 
\end{align}
The second term on the RHS of (\ref{def:L-linearization}) is due to the disagreement of the $\theta_i$'s, and can be controlled   choosing suitably small $\gamma$. In fact,  we will prove that a curvature condition on  the first term of the RHS of (\ref{def:L-linearization}) along the directions $\bT-\bT'\in  \mathbb{C}_{\gamma}(\mathcal{S})$ is enough to   establish the desired  error bounds on the LASSO error $(1/m)\sum_{i=1}^m\|\hat{\T}_i-\T^*\|^2$. This motivates the following definition of   RSC-like property of  ${L_\gamma}$. 
   
\begin{assumption}[In-network RE]\label{RSC along consensus space}
The loss function ${L_{\gamma}}$ satisfies the following RSC condition with curvature $\delta>0$ and tolerance $\xi>0$:
 \begin{equation}\label{ARE eq}
      {\mathcal{T}_{L}({\bT};\bT')
    \geq \delta\lVert(\bT-\bT')_{\textnormal{av}}\rVert^2-{\xi \,h^2(\gamma,\lVert(\bT-\bT')_{\perp}\rVert), \quad{\forall\bT-\bT'\in\mathbb{C}_{\gamma}(\mathcal{S})}.}}
        \end{equation}
\end{assumption}
The stipulated condition mandates a positive curvature for $L_\gamma$ along consensual directions in $\mathbb{C}{\gamma}(\mathcal S)$. Across the entire space, however, $L\gamma$ need not exhibit strong convexity, attributed largely to the tolerance term accounting for consensus errors.  

The following two results establish sufficient conditions for (\ref{ARE eq}) to hold, for deterministic and random design matrices $\bX$--which match those required for the centralized LASSO (see Lemma~\ref{RSC-centralized-deterministic} and Lemma~\ref{Global RE}).       
\begin{lemma}\label{global-ARE-determin}
Reinstate Lemma~\ref{RSC-centralized-deterministic}, under $\mu/2-16s\tau>0$. Then (\ref{ARE eq}) holds, with {$\delta=\mu/2-16s\tau$ and $\xi=\tau$},  for any   {given $\gamma\in(0,(1-\rho)/L_{\max}]$}. 
\end{lemma} 
\begin{proof}
See Appendix~\ref{Proof of Lemma 4}.
\end{proof}
\begin{lemma}\label{Global ARE} Let  $\mathbf{X}\in \mathbb{R}^{N\times d}$    satisfy Assumption~\ref{ass_X_random}. For any $N$ and $\gamma$ such that
  {
\begin{equation}
     N\geq c_{2} \frac{\zeta_{\Sigma}s \log d}{\lambda_{\min}(\Sigma)}, \ \ \text{and}\ \   {\gamma\in(0,(1-\rho)/L_{\max}]},
\end{equation}
}
it holds 
\begin{align}\label{network slack}
   {\mathcal{T}_{L_\gamma}({\bT};\bT')
    \geq\frac{\lambda_{\min}(\Sigma)}{4}\lVert(\bT-\bT')_{\textnormal{av}}\rVert^2-\frac{\lambda_{\min}(\Sigma)}{64s}  \,h^2\left(\gamma, \lVert(\bT-\bT')_{\perp}\rVert\right)},\,\forall\, \bT,\bT'\,:\,\bT-\bT'\in\mathbb{C}_{\gamma}(\mathcal{S}),
\end{align}
with probability at least $1-\exp(-c_0N)$. Here,    {$c_0,c_2>0$ are} universal constants.
\end{lemma}
 \begin{proof}
See Appendix~\ref{Proof of Lemma 5}.
\end{proof}
\vspace{-0.5cm}
\subsection{Error bounds and statistical consistency of the  LASSO error of (\ref{Change problem})}\label{Bounds On the l2 Error for the the Distributed Lasso} 
We are ready to  establish consistency and convergence rates for the augmented LASSO estimator $\hat{\bT}$. Our first result is a deterministic upper bound on the average error under the In-network RE  condition \eqref{ARE eq}.   
\begin{theorem}\label{solution err bound}
 Consider the augmented LASSO problem (\ref{Change problem}) under Assumptions~\ref{W} and \ref{RSC along consensus space}. For any fixed   
 $\lambda$ and $\gamma$ satisfying respectively
 \begin{align}\label{eq:gamma_cond}
    \frac{2}{N}\lVert \mathbf{X}^{\top} \mathbf{w}\rVert_{\infty}\leq \lambda \quad \text{and} \quad \gamma\leq \frac{2(1-\rho)}{4L_{\max}+\delta},
 \end{align}
 any   solution  $\hat{\bT}=[\hat{\T}_1,\ldots, \hat{\T}_m]^\top$ satisfies %
\begin{align}\label{oracle}
     &\frac{1}{m}\sum\limits_{i=1}^m\lVert\hat{\T}_i-\T^\ast\rVert^2\notag\\
   \leq&\underbrace{\frac{9\lambda^2s}{\delta^2}}_{\text{centralized error}}+\underbrace{\frac{2\xi d^2\gamma^2({\max_{i\in[m]}\lVert w_i^{\top} X_i\rVert_{\infty}}+\lambda n)^4}{\delta\lambda^2n^4(1-\rho)^2}+\frac{4d\gamma({\max_{i\in[m]}\lVert w_i^{\top} X_i\rVert_{\infty}}+\lambda n)^2}{\delta n^2[2(1-\rho)-4L_{\max}\gamma-\delta\gamma]}}_{\text{cost of decentralization}}.
\end{align}
\end{theorem}
\begin{proof}
See Appendix~\ref{Proof of Theorem 2}.
\end{proof}
Theorem~\ref{solution err bound} shows the bound on the LASSO error over the network can be decoupled in two terms--the first one   matches that of the centralized  LASSO error (see, \eg, \citep[Th. 7.13]{Wainwright-book})--while the second one quantifies the price to pay due to the decentralization of the optimization and consequent lack of   consensus. The explicit dependence on $\gamma$   shows that the detriment effect of the consensus errors can be controlled by    $\gamma$: as $\gamma\rightarrow 0$, the error bound above   approaches that of the centralized LASSO solution. There is however no free lunch;  we anticipate that   $\gamma\rightarrow 0$ affects adversarially the convergence rate of the proximal gradient algorithm   applied to  {problem}~\eqref{Change problem}, determining thus a speed-accuracy dilemma. 

The next result provides nonasymptotic  rates for the LASSO error above, under the random Gaussian  model for   $\bX$ and  the noise $\mathbf{w}$ in (\ref{eq:linear-model})--optimal centralized convergence rates are achievable by a proper choice of $\gamma$. 
\begin{theorem}\label{statistical optimization error result}
Consider the augmented LASSO problem (\ref{Change problem}) with $d \geq 2$ under Assumptions~\ref{W}. Suppose that   $\bX$ satisfies   Assumption~\ref{ass_X_random} and  $\mathbf{w}\sim \mathcal{N}(\mathbf{0},\sigma^2I_{N})$;   the sample size satisfies \vspace{-0.1cm}
\begin{equation}\label{stat N condition}
     N\geq c_3\frac{\zeta_{\Sigma}s\log d}{\lambda_{\min}(\Sigma)};
\end{equation}
and the parameters  $\lambda$ and $\gamma$ are chosen according to the following  
 \begin{align}
  &   {\lambda=c_4\sigma\sqrt{\frac{\zeta_{\Sigma}t_0\log d}{N}}},\label{stat_lambda_final}\\
    &{
    \gamma\leq {c_5} \frac{(1-\rho)}{\lambda_{\max}(\Sigma)(d + \log m) +\lambda_{\min}(\Sigma )d{m(\log m+1)}}},\label{stat gamma condition}
 \end{align}  
 for some  $t_0 > 2$. Then,  
any  solution $\hat{\bT} = [\hat{\T}_1,\ldots, \hat{\T}_m]^\top$ of problem (\ref{Change problem}) satisfies
      \begin{equation}\label{er bound 2}
      \frac{1}{m}\sum\limits_{i=1}^m\lVert\hat{\T}_i-\T^*\rVert^2
      \leq {c_6\frac{\sigma^2\zeta_{\Sigma}t_0}{\lambda_{\min}(\Sigma)^2}}\frac{s\log d}{N},
    \end{equation}
with probability at least      \begin{equation}\label{probability}
 {1- c_{7}\exp(-c_{8}\log d)}.
\end{equation}
Here, {$c_3,\dots,c_8$ are universal constants.}
\end{theorem}
\begin{proof}
See Appendix~\ref{Proof of Theorem 3}.
\end{proof}
The    bound   (\ref{er bound 2}) matches  the  statistical error of the centralized  LASSO estimator in (\ref{original problem})--proving   that statistical consistency over networks is achievable   under the same order of the sample size $N$ used in the centralized setting, even  when the local number  $n$ of samples does not suffice.  This is possible because agents   communicate over the network--the computation of such a solution and associated communication overhead is studied in the next section. \vspace{-0.2cm}
\section{Distributed Gradient Descent  Algorithm  }\label{Geometric Convergence of Distributed Gradient Descent Based Algorithms}
To compute the statistically optimal   estimator $\hat{\bT}$ over networks, we employ 
the proximal gradient algorithm applied to the penalized formulation~(\ref{Change problem}), which naturally decomposes across the agents. 
Specifically, at iteration $t$,   $\bT$ is updated by minimizing the first order approximation of the objective function $L_{\gamma}$, which reads 
\begin{align}\label{regularized algorith}
  \bT^{t+1}
  =&\underset{\lVert\theta_i\rVert_1\leq R\text{ }\forall i\in[m]}{\text{argmin}}\text{ }
  L_{\gamma}(\bT^t)+\langle\nabla L_{\gamma}(\bT^t),\bT-\bT^t\rangle+\frac{1}{2  {\beta m}}\lVert\bT-\bT^t\rVert^2+\frac{\lambda}{m}\lVert\bT\rVert_1,
\end{align}
where 
we included an extra constraint $\lVert\theta_i\rVert_1\leq R$ to regularize the iterates, and $\beta>0$ plays the role of the stepsize.   
The following lemma shows that one can find a sufficiently large $R$ so that the solution of  \eqref{Change problem} does not change if we add therein the norm ball constraint $\|\T_i\|_1\leq R$, $i\in [m]$.   
\begin{lemma}\label{solution eq}
Consider Problem~(\ref{Change problem}) under   Assumption~\ref{W}. Further assume that (i)   $\bX$ satisfies the RSC condition~(\ref{Arsc}) with  $\delta=\mu/2-16\,s\tau> 0$; (ii)  $\lambda$ satisfies~(\ref{distributed lambda}); and (iii) $\gamma$ satisfies 
\begin{align}\label{equ-sol-gamma}
\gamma\leq\frac{(1-\rho)}{2L_{\max}+\delta+{128}(d/s)\delta(\max_{  i\in [m]}\lVert w_i^{\top} X_i\rVert_{\infty}/(\lambda n)+2\sqrt{m})^2}.
\end{align}
Then, $\|\hat{\T}_i\|_1 \leq R,$    $i \in [m]$, whenever  $R$ is such that 
\begin{align}\label{c}
    R \geq \max\bigg\{\frac{\lambda s}{\delta(1-r)}{\bigg(13+\frac{1}{32}\sqrt{\frac{2\tau s}{\delta}}\bigg)}, \frac{1}{r}\|\T^*\|_1\bigg\},
\end{align}
  with $r\in (0,1)$.
\end{lemma}
\begin{proof}
See Appendix~\ref{sec:R-bound}. 
\end{proof} 

Therefore, we can focus on  Problem~(\ref{regularized algorith}) without loss of generality. Notice that the problem is separable in $\{\T_i\}_{i \in [m]}$; hence,  it 
can be solved distributively from each agent $i$. Furthermore, the solution can be computed in an explicit  form,  as determined next. 
\begin{lemma}\label{lm-closed _form}
The solution to (\ref{regularized algorith}) 
  reads \begin{equation}\label{composite dgd variant}
       \T_i^{t+1}
        =\left\{
      \begin{array}{ll}
       \textnormal{prox}_{{\beta}   {\lambda\lVert\cdot\rVert_1}}(\psi_i^t), &\quad\text{if }\quad  \left\| \textnormal{prox}_{{\beta}   {\lambda\lVert\cdot\rVert_1}}(\psi_i^t) \right\|_1 \leq  R, \medskip \\
         \Pi_{\mathcal{B}_{\lVert\cdot\rVert_1}(R)}( \psi_i^t), &\quad \text{otherwise}; \\
      \end{array}
      \right.
\end{equation} where  $\mathbb{R}\ni x\mapsto \text{prox}_{h}(x)\in \mathbb{R}$ is the proximal operator (applied to $\psi_i^t$ component-wise), and  $$ \psi_i^t=\left( 1 -   {\frac{\beta}{\gamma}}\right)\T^t_i +   {\frac{\beta}{\gamma}} \left(\sum_{j = 1}^m w_{ij} \T_j^t - \gamma \nabla f_i (\T_i^t)\right).$$\end{lemma} \begin{proof}
See Appendix~\ref{proof_lm_closed_form}.
\end{proof}
Notice that the proximal operation in (\ref{composite dgd variant}) has a closed-form expression via soft-thresholding \citep{382009}  while the projection onto the $\ell_1$-ball can be efficiently computed using the procedure in \citep{duchi2008efficient}.  
To perform the update (\ref{composite dgd variant}), each agent $i$ only needs to receive the local estimates $\theta_j^t$ from its immediate neighbors. 
\subsection{Linear convergence to statistical precision}\label{sub_stat_linear_convergence}
Convergence rate of the optimization error $\bT^t-\hat{\bT}$ is stated under the RSC condition~\eqref{Arsc}, 
in terms of the contraction coefficient $\kappa$ and initial optimality gap ${\eta^0_G}$:
\begin{equation}\label{uav deterministic}
  \mu_{\textnormal{av}}\triangleq\frac{\mu}{8}-8s\tau,\quad  \kappa\triangleq1-  {\frac{\beta\mu_{\textnormal{av}}}{4}},\quad\text{and}\quad \eta^0_G \triangleq \bigg(L_\gamma(\bT^0)+\frac{\lambda}{m}\lVert\bT^0\rVert_1\bigg)-\bigg(L_\gamma(\hat{\bT}){+}\frac{\lambda}{m}\lVert\hat{\bT}\rVert_1\bigg),
\end{equation}  {where $\boldsymbol{\theta}^0$ is a fixed  initialization.}
Further, denote
\begin{align}\label{eq-error-stat}
    \varepsilon_{\textnormal{stat}}^2\triangleq \frac{36}{m}\sum\limits_{i=1}^m\lVert\hat{\T}_i-\T^*\rVert^2+{\frac{\lambda^2s}{1976\mu^2}}.
\end{align}
We can now state our convergence result. 
\begin{theorem}\label{deterministic optimization error result}
Consider the augmented LASSO problem (\ref{Change problem}) under Assumptions~\ref{W}. Suppose the  design matrix $\bX$ satisfies the RSC condition~\eqref{Arsc} with {$\mu \geq {{c}_{9}} s \tau$} for some sufficiently large   constant ${{c}_{9}}>0$; 
and the penalty parameters $\lambda$ and  $\gamma$ satisfies   
\begin{align}
\lambda  &\geq \max\left\{\frac{2\| \bX^\top \mathbf{w}\|_\infty}{N},{64} \tau \| \T^*\|_1 \right\},\label{lambda condition}\\
    \gamma &\leq\frac{1-\rho}{2L_{\max}+(\mu/2-16s\tau)\left(1+128(d/s) (\max_{  i\in [m]}\lVert w_i^{\top} X_i\rVert_{\infty}/(\lambda n)+2\sqrt{m})^2\right)},
    \label{constrained gamma condition2}
\end{align} 
 respectively. Let $\{\T_i^t\}_{i \in [m]}$ be the sequence generated by Algorithm~(\ref{composite dgd variant}) under the following choices of tuning parameters $\beta$ and ${R}$
\begin{equation}\label{beta condidion}
\beta =   {\frac{\gamma}{\gamma L_{\max}+1-\lambda_{\min}(W)}}\end{equation}and \begin{align}\label{R condidion}
\max\bigg\{\frac{56\lambda s}{\mu-32s\tau}, 2\|\T^*\|_1\bigg\}\leq R\leq \frac{\lambda}{32 \tau};
\end{align}
{and  initialization $\theta_i^0\in \mathbb{R}^d$, $i\in [m]$, such that  } \begin{equation}\label{eq:cond_eta_G}\eta_G^0\geq 4s\tau \cdot\varepsilon_{\textnormal{stat}}^2.\end{equation} 
 Then, 
\begin{equation}\label{final bound lagrangian}
   \frac{1}{m}\sum\limits_{i=1}^m\lVert\T_i^t-\hat{\T}_i\rVert^2\leq \frac{\tau s}{\mu_{\textnormal{av}}} \cdot \varepsilon_{\textnormal{stat}}^2  +  \bigg(\frac{\tau s}{ \mu_{\textnormal{av}}} \frac{4 \alpha^4}{\lambda^2s} + \frac{\alpha^2}{\mu_{\textnormal{av}}} \bigg),
\end{equation}
for any tolerance parameter $\alpha^2$ such that 
\begin{equation}\label{determinstic alpha}
   {\min\bigg\{\frac{R\lambda}{4},\eta_G^0\bigg\}}\geq\alpha^2\geq 4s\tau \cdot\varepsilon_{\textnormal{stat}}^2,
\end{equation}
and for all  
{\begin{equation}\label{T1 expression}
   t\geq
  \bigg\lceil\log_2\log_2\left(\frac{R\lambda}{\alpha^2}\right)\bigg\rceil\bigg(1+\frac{L_{\max}\log 2}{\mu_{\textnormal{av}}}+\frac{(1+\rho)\log 2}{\gamma\mu_{\textnormal{av}}} \bigg) +  \left(\frac{ L_{\max}}{\mu_{\textnormal{av}}} + \frac{1+\rho}{\gamma\mu_{\textnormal{av}}}\right) \log\left(\frac{\eta^0_G}{ \alpha^2}\right).  
\end{equation}}
The intervals in \eqref{R condidion} and \eqref{determinstic alpha} are nonempty. 
\end{theorem}
 \begin{proof}
 See Appendix~\ref{Guarantees for distributed regularized Lasso}.
 \end{proof}

The theorem shows that Algorithm~\eqref{composite dgd variant}  converges at a linear rate to an optimal solution $\hat{\bT}$, up to a tolerance term  as specified on  the right hand side of~\eqref{final bound lagrangian}--the first term therein depends on the model parameters, while the second one is controlled by $\alpha^2$.  
Theorem~\ref{Th_stat_convergence} and (see also Corollary~\ref{Th_stat_convergence-alter}) below proves that for the random Gaussian data generation model, the tolerance can be driven below  the statistical precision for sufficiently large $N$.

\begin{remark}
Observe that the condition   \eqref{eq:cond_eta_G} pertaining to the initialization does not  truly impose a substantial constraint. Indeed, any initial point that contravenes \eqref{eq:cond_eta_G} would inherently be situated within the   centralized statistical error ball, i.e., 
$$ 
\eta_G^0  \leq 4s\tau \cdot\varepsilon_{\textnormal{stat}}^2
 \overset{\eqref{constrained gamma condition2},\autoref{solution err bound}}{=} \mathcal{O}\left(s\tau\cdot \lambda^2s\right).
$$ 
\end{remark}
\begin{remark}
Algorithm~\eqref{composite dgd variant} is closely related to the DGD algorithm studied in the literature of distributed optimization (\eg, \citep{WYin_ncvxDGD_SIAM2018,Nedic_Olshevsky_Rabbat2018}). 
In fact, if in (\ref{composite dgd variant}) one  choose   $\beta =   {\gamma/2}$, with $\gamma$ satisfying  \eqref{constrained gamma condition2} [note that this choice of $\beta$ is compatible with {\eqref{beta condidion}}], the gradient step therein reduces to  
\begin{align}\label{eq:DGD-simplified}
    \psi_i^t = \frac{1}{2} \left( \T_i + \sum_{i = 1}^m w_{ij} \T_j^t \right) - \frac{\gamma}{2} \nabla f_i (\T_i^t),
\end{align}
which can be viewed as DGD with weight matrix $\frac{1}{2} (I + W)$ and step size $\gamma/2$.
\end{remark}

\begin{theorem}\label{Th_stat_convergence} 
Consider the augmented LASSO problem (\ref{Change problem}) with $d \geq 2$ under Assumption~\ref{W}. Suppose the design matrix $\bX$ satisfies Assumption~\ref{Random Gaussian model}, $\mathbf{w} \sim \mathcal{N} (\mathbf{0}, \sigma^2 I_{N})$, and  
\begin{align}\label{eq:cond_on_N}
    N\geq c_{10} \frac{\zeta_\Sigma}{\lambda_{\min}(\Sigma)} s \log d.
\end{align}
Choose the penalty parameters  $\lambda$ and $\gamma$  satisfying respectively
\begin{equation}\label{global lambda condition}
         {\lambda\geq c_{11}\max\bigg\{ \sigma\sqrt{\frac{\zeta_{\Sigma}t_0\log d}{N}}, \zeta_{\Sigma}\cdot\frac{s\log d}{N}\bigg\},}  
\end{equation}
and  
 \begin{equation}\label{gamma-final-algo}
    {\gamma\leq c_{12}\frac{(1-\rho)}{\lambda_{\max}(\Sigma)\left({d+\log m}\right)+\lambda_{\min}(\Sigma){dm}\cdot \left({{\log m}}+1\right)}}.
 \end{equation}
for some fixed ${t_0>2}.$
 Let $\{\T_i^t\}_{i \in [m]}$ be the sequence generated by Algorithm~\eqref{composite dgd variant} under the following choices of tuning parameters $\beta$ and $R$
 \begin{align}\label{eq:alg-conditions-random}
     \begin{split}
         \beta=    {\frac{\gamma}{\gamma  c_{13}d/n+1-\lambda_{m}(W)}}, \quad    
              \max\bigg\{\frac{56\lambda s}{\lambda_{\min}(\Sigma)- c_{14}s\zeta_{\Sigma}\log d/N}, 2s\bigg\}\leq R\leq \frac{\lambda N}{ c_{14}\zeta_{\Sigma}\log d},
     \end{split}
 \end{align}
{and initialization $\theta_i^0\in \mathbb{R}^d$, $i\in [m]$, such that  } 
{    
 \begin{equation}\label{initial_gap_Thm12}
     \eta_G^0\geq   c_{15}\frac{ \zeta_{\Sigma}}{\lambda_{\min}(\Sigma)^2}\cdot\frac{s\log d}{N}\cdot\lambda^2s.
 \end{equation}}
   Then, with probability at least \eqref{probability},  
\begin{align}\label{stat_algo_bound_general}
    &\frac{1}{m}\sum\limits_{i=1}^m\lVert\T_i^t-\hat{\T}_i\rVert^2 \leq  
    \frac{ c_{16}}{\lambda_{\min}(\Sigma)}\left[\alpha^2+\zeta_{\Sigma}\cdot\frac{s\log d}{N}\left(\frac{1}{m}\sum\limits_{i=1}^m\|\hat{\T}_i-\T^*\|^2+\frac{\lambda^2s}{\lambda_{\min}(\Sigma)^2}+\frac{\alpha^4}{\lambda^2s}\right)\right],
\end{align}
for any tolerance parameter $\alpha^2$ such that  
\begin{equation}\label{alpha stat range}
{\min\bigg\{\frac{R\lambda}{4},\eta_G^0\bigg\}}\geq\alpha^2\geq { c_{17}\zeta_{\Sigma}}\left(\frac{1}{m}\sum\limits_{i=1}^m\lVert\hat{\T}_i-\T^*\rVert^2+{\frac{\lambda^2s}{\lambda_{\min}(\Sigma)^2}}\right)\cdot\frac{s\log d}{N},
\end{equation}
and for all
     \begin{align}\label{T2 expression}
         t\geq  c_{18}\left[\left\lceil\log_2\log_2\left(\frac{R\lambda}{\alpha^2}\right)\right\rceil+ \log\left(\frac{\eta^0_G}{ \alpha^2}\right)\right]\bigg({\kappa_{\Sigma}}(d + \log m)+\frac{(1+\rho)}{\lambda_{\min}(\Sigma)\gamma} \bigg). 
     \end{align}
 Further, the range of values of $R$ in  \eqref{eq:alg-conditions-random} is nonempty; and the interval in~\eqref{alpha stat range} is nonempty as well, with probability at least \eqref{probability}.
 Here,   {$ c_{10}, \dots,  c_{18}$}  are universal constants.
\end{theorem}
\begin{proof}
 See Appendix~\ref{Statistical results for DGD-CTA first order methods}.
 \end{proof}
A suitable choice of the free parameters above leads to the following simplified result, showing linear convergence up to a tolerance  of a higher order than the statistical error. 

\begin{corollary}\label{Th_stat_convergence-alter} 
Consider the augmented LASSO problem (\ref{Change problem}) with $d \geq 2$ under Assumptions~\ref{W}. Suppose the   $\bX$ satisfies Assumption~\ref{Random Gaussian model}, $\mathbf{w} \sim \mathcal{N} (\mathbf{0}, \sigma^2 I_{N})$ and the sample size satisfies
\begin{equation}\label{global sample size condition-alter}
    N\geq  c_{19}\max\bigg\{\frac{\zeta_{\Sigma}s\log d}{\lambda_{\min}(\Sigma)}, \frac{s^2\zeta_{\Sigma}\log d}{\sigma^2}\bigg\} \quad \text{and} \quad \frac{d + \log m}{n} \geq 1.
\end{equation}
Choose the penalty parameters  $\lambda$ and $\gamma$  satisfying respectively
\begin{equation}\label{lambda}
         \lambda= c_{20}\sigma\sqrt{\zeta_{\Sigma}t_0\cdot\frac{\log d}{N}}
\end{equation}
and 
  \begin{equation}\label{gamma-final-algo-alter}
   { \gamma\leq c_{21}\frac{1-\rho}{\lambda_{\max}(\Sigma)(d+\log m)+\lambda_{\min}(\Sigma){dm}\cdot \left({{\log m}}+1\right)}},
 \end{equation}
 for some fixed ${t_0 \geq 2}$.
 Let $\{\T_i^t\}_{i \in [m]}$ be the sequence generated by Algorithm~\eqref{composite dgd variant} under the following choices of $\beta$ and $R$:
{ \begin{align}\label{eq:alg-conditions-random-alter}
     \begin{split}
         \beta=    {\frac{n\gamma}{\gamma  c_{13}d+n(1-\lambda_{m}(W))}},\quad 
              \max \bigg\{ c_{22}\frac{s\sigma}{\lambda_{\min}(\Sigma)}\sqrt{\frac{t_0\zeta_{\Sigma}\log d}{N}}, 2s\bigg\}\leq R\leq  c_{23}\sigma\sqrt{\frac{t_0N }{\zeta_{\Sigma}\log d}}
     \end{split}
 \end{align}}
{{and initialization $\theta_i^0\in \mathbb{R}^d$, $i\in [m]$, such that}
} 
{
\begin{equation}
    \eta_G^0\geq  c_{24}t_0\left(\frac{\sigma\zeta_{\Sigma}}{\lambda_{\min}(\Sigma)}\right)^2\cdot\left(\frac{s\log d}{N}\right)^2.
\end{equation}}
Then, with probability at least \eqref{probability}, 
{
\begin{align}\label{eq:sat_ball}
 &\frac{1}{m}\sum\limits_{i=1}^m\lVert\T_i^t-\hat{\T}_i\rVert^2\notag\\
  &\leq  \frac{ c_{23}}{\lambda_{\min}(\Sigma)}\left(\alpha^2+\zeta_{\Sigma}\frac{s\log d}{N}\cdot\frac{1}{m}\sum\limits_{i=1}^m\|\hat{\T}_i-\T^*\|^2+\zeta_{\Sigma}\frac{s\log d}{N}\cdot\frac{\zeta_{\Sigma}\sigma^2t_0}{\lambda^2_{\min}(\Sigma)}\frac{s\log d}{N}+\frac{1}{\sigma^2t_0}\alpha^4\right),
\end{align}
}
for any tolerance parameter $\alpha^2$ such that 
{
\begin{equation}\label{alpha stat range-alter}
  {\min\bigg\{ c_{24}{R\sigma}\sqrt{\frac{\zeta_{\Sigma}t_0\log d}{N}},\eta_G^0\bigg\}}\geq\alpha^2\geq  c_{25}\zeta_{\Sigma}\cdot\frac{s\log d}{N}\left(\frac{1}{m}\sum\limits_{i=1}^m\lVert\hat{\T}_i-\T^*\rVert^2+\frac{\sigma^2\zeta_{\Sigma}t_0}{\lambda_{\min}(\Sigma)^2}\cdot\frac{s\log d}{N}\right),
\end{equation}
}
and for all 
  {
\begin{align}\label{T2 expression-alter}
   t\geq&\, c_{26} \cdot    { \kappa_{\Sigma}\cdot \frac{dm\left({\log m+1}\right)}{1 - \rho}}\cdot \left\{\bigg\lceil\log_2\log_2\left(\frac{R\sigma }{\alpha^2}\sqrt{\frac{\zeta_{\Sigma}t_0\log d}{N}}\right)\bigg\rceil+\log\left(\frac{\eta^0_G}{ \alpha^2}\right)\right\}.
     \end{align}
}
{The range of value of $R$ in \eqref{eq:alg-conditions-random-alter} is nonempty; and the   interval in~\eqref{alpha stat range-alter} is nonempty as well, with probability at least \eqref{probability}. } 
\end{corollary}
\begin{proof}
 See Appendix~\ref{sec:pf-cor-13}.
 \end{proof}
 It is not difficult to check that, in the above setting, Theorem~\ref{statistical optimization error result} holds (in particular, \eqref{gamma-final-algo-alter} implies \eqref{stat gamma condition}; hence, by \eqref{er bound 2}, we have $\frac{1}{m}\sum_{i=1}^m\|\hat{\T}_i-\T^*\|^2=\mathcal{O}(\frac{s\log d}{N})$. Therefore, whenever the sample size $N=o(s\log d)$--a condition that is
required for statistical consistency of any centralized method by minimax results  (see, e.g., \citep{Yu2011}),  the (lower bound of the) tolerance $\alpha^2$ in (\ref{alpha stat range-alter}) and thus the overall residual error in (\ref{eq:sat_ball}) is of smaller  order than the   statistical error $\mathcal{O}(\frac{s\log d}{N})$.  Therefore, in this setting, 
a total number of communications (iterations)   of  
\begin{equation}\label{eq:numb_comm}
    \mathcal{O}\left( \kappa_\Sigma \cdot \frac{d\,   {m \, \log m} }{1-\rho}  \cdot  \log \frac{1}{\alpha^2} \right)\end{equation} 
is sufficient to drive the iterates generated by Algorithm~\eqref{composite dgd variant}  within ${\mathcal{O}(\alpha^2)}$ of an optimal  solution $\widehat{\bT}$ (in the sense of (\ref{eq:sat_ball})), and thus  to an estimate of   $\theta^*$ within the statistical error.  This matches centralized statistical accuracy achievable by the LASSO estimator $\hat{\T}$ in (\ref{original problem}). Notice that no conditions on the local sample size $n$ are required. 

 The expression \eqref{eq:numb_comm} sheds light on the impact of the problem and network parameters on the convergence. Specifically, the following comments are in order. 
 
 \begin{itemize}
     \item[{\bf (i)}] \textbf{Network dependence/scaling:} The term  $1/(1-\rho)>1$ captures the effect of the network; as expected,  weakly connected networks (i.e., as $\rho \to 1$) call for more rounds of communication to achieve the prescribed accuracy.   Recall that $\rho=\rho(m)$ is a function of the number of agents $m$ (and the specific topology under consideration). Hence, the term \begin{equation}\label{eq:comm_scaling}
         \frac{  {m \,\log m}}{1-\rho(m)} \end{equation} shows how the number of rounds of communications on the network scales with    $m$, for a given graph topology (determining $\rho(m)$). 
         Our experiments in Sec.~\ref{Numerical Experiments} seem to suggest that the dependence of the communication rounds on $m$ as predicted by (\ref{eq:comm_scaling}) is fairly  tight, for different graph topologies.

         Table~\ref{table:rho_scaling} provides some estimates of the scaling of $1/(1-\rho(m))$ with $m$ for some representative   graphs, when the  Metropolis-Hastings rule is used for the gossip matrix $W$   \citep{Nedic_Olshevsky_Rabbat2018}.  Some graphs, for instance the Erd\H{o}s-R\'{e}nyi, exhibit a more favorable scaling than others, such as line graphs.  Note that equation (\ref{eq:comm_scaling}) does not encapsulate the total communication cost,  which is also contingent on the density of the graph. We defining one channel use as the communication occurring per edge connecting two nodes.   For example, in the case of an Erd\H{o}s-R'{e}nyi graph (with $p=\log m/m$), the total channel uses (across all nodes) per communication round is ${\mathcal{O}}(m\log m)$. 
          In contrast, the complete graph necessitates  ${\mathcal{O}}(m^2)$ channel uses per  communication round, even though both graphs display a scaling of (\ref{eq:comm_scaling}) with $m$, and thus a total number of communication rounds     of the same order.

     \item[({\bf ii})] \textbf{Population condition number:} The ratio $\frac{\lambda_{\max}(\Sigma)}{\lambda_{\min}(\Sigma)}$ is the condition number of the covariance matrix of the data; it  can be interpreted as the restricted condition number of the LASSO loss function $F(\T)$ [see \eqref{eq:losses}]. Therefore, as expected, ill-conditioned problems call for more iterations (communication) to achieve the prescribed accuracy.    
      \item[({\bf iii})] \textbf{Speed-accuracy dilemma:} We proved that centralized statistical accuracy, as for the LASSO estimator   $\hat{\T}$ in (\ref{original problem}),  is achievable over networks via the distributed algorithm~\eqref{composite dgd variant}. This can be accomplished even when individual agents do not possess sufficient data to ensure statistical consistency locally. The crucial factor making this possible is the ``assistive'' role of the network via information mixing, However, equation (\ref{eq:numb_comm}) reveals that regardless of the speed of information propagation within the network (irrespective of how small $\rho$ is), the total number of communication rounds necessary to achieve a predetermined accuracy scales as $\mathcal{O}(d)$. Our forthcoming numerical results will validate the precision of such scaling.  Therefore, we discern that the DGD-like scheme encounters similar speed-accuracy dilemmas in high-dimensional regimes as observed when applied to strongly convex, smooth losses in lower-dimensional cases (e.g., see \citep{Nedic_Olshevsky_Rabbat2018}). This issue seems to be inevitable and a direct consequence of the structural updates implemented by the algorithm. 
      
 \end{itemize}
     
   \begin{table}[t]
 \begin{center}\resizebox{0.9\columnwidth}{!}{
 \begin{tabular}{ccccccc}
                   & path  & $2$-d grid       & complete  &   $p$-Erd\H{o}s-R\'{e}nyi & $p$-Erd\H{o}s-R\'{e}nyi \\[2ex]\hline\\[1ex]
$(1 - \rho)^{-1}$    & $\mathcal O(m^2)$   & $\mathcal O(m \log m)$        & $\mathcal O(1)$         & $\mathcal O(1)$  [$p = \log m/m$]                           & $\mathcal O(1)$  [$p = \mathcal O(1)$]             \\
\\[1ex]\hline                          \end{tabular} } \caption{Scaling of $(1-\rho(m))^{-1}$ with agents' number $m$, for different graph topologies.}\label{table:rho_scaling}  \end{center} \vspace{-.3cm}
\end{table}      
  
  \vspace{-.6cm}

\section{Numerical Results}\label{Numerical Experiments}
In this section, we provide some experiments on synthetic and real data. The former are instrumental to validate our theoretical findings. More specifically,  we validate the following theoretical results.  {\bf (i)}  On the statistical error front, we show that with a proper choice of $\gamma$, the solution of the distributed formulation~\eqref{Change problem}  achieves the statistical accuracy of the centralized LASSO estimator (Theorem~\ref{statistical optimization error result}). We also validate the dependency of $\gamma$ with the dimension $d$ [cf.~\eqref{stat gamma condition}].   On the computational front, {\bf (ii)}  we demonstrate that   DGD-like algorithm \eqref{composite dgd variant} displays linear convergence up to   statistical precision; {\bf (iii)} we also validate the scaling of the communication rounds with the network size $m$, as predicted by (\ref{eq:comm_scaling}).  {\bf (iv)} Finally,   we illustrate the speed-accuracy dilemma, as shown in Theorem~\ref{Th_stat_convergence}. We then proceed to  experiment on real data, showing that statistical accuracy of the centralized LASSO estimator (Theorem~\ref{statistical optimization error result}) is achievable by  the distributed method~\eqref{Change problem}, still at the cost of a convergence rate scaling with $\mathcal{O}(d)$.  All the experiments were run on a server equipped with  Intel(R) Xeon(R) CPU E5-2699A v4 @ 2.40GHz.

  {\bf Experimental setup (synthetic data):}  The ground truth   $\T^*$ is set by randomly sampling a multivariate Gaussian  $\mathcal{N}(0, I_{d})$ and  {thresholding the smallest $d-s$ elements to zero}. The noise vector $\mathbf{w}$ is assumed to be multivariate Gaussian $\mathcal{N}(\mathbf{0},0.25I_{N})$. We construct $\mathbf{X}\in\mathbb{R}^{N\times d}$ by independently generating each row $x_i\in\mathbb{R}^d$,  {adopting the following procedure \citep{agarwal2012fast}: let $z_1,\dots,z_{d-1}$ be i.i.d. standard normal random variables, set $ x_{i,1}=z_1/\sqrt{1-0.25^2}$ and  $x_{i,t+1}=0.25x_{i,t}+z_t$, for $t=1,2,\dots,d-1.$  It can be verified that all the eigenvalues of $\Sigma=\text{cov}(x_i)$ lie within the interval $[0.64,2.85]$.}
 We partition $(\mathbf{X},\mathbf{y})$ as $\mathbf{X}=[X_1^{\top}, X_2^{\top},\dots,X_m^{\top}]^{\top}$ and  $\mathbf{y}=[y_1^{\top},\dots,y_m^{\top}]^{\top},$ and agent $i$ owns the dataset portion $(X_i, y_i)$ and we have $m$ agents in total.  We simulate an undirected graph $\mathcal{G}$ following the Erd\"os-R\'enyi model $G(m,p)$, where $m$ is the number of agents and  $p$ is the probability that an edge is independently included in the graph. The coefficient of the  matrix $W$ are chosen according to the {Lazy Metropolis} rule~\citep{olshevsky2017linear}.  All results are using Monte Carlo with $30$ repetitions. 


  {\bf 1) Statistical accuracy verification (Theorem~\ref{statistical optimization error result}).}  
We set $N=220,m=20, d=400,$ and consider two types of graphs, namely a fully connected and a  weakly connected  graph, the latter generated as Erd\"os-R\'enyi graph with edge probability $p=0.1$, resulting in $ \rho\approx0.973.$   

Fig. \ref{fig:fig4} plots the log-average statistical error $\log\big(\sum_{i=1}^m\|\hat{\T}_i-\T^*\|^2/m\big)$ versus $\lambda$ for the fully- and weakly-connected graphs (resp.), and contrasts it with  the centralized LASSO log-error  $\log\big(\lVert\hat{\T}-\T^*\rVert^2\big)$. The following comments are in order.   {\bf (i)}   A careful choice of $\gamma$ ($=5\times 10^{-4}$) is required to ensure that the distributed penalty LASSO recovers the centralized  $\ell_2$-error; however, with the same choice of $\gamma,$ the solution achieved by the  distributed method~\eqref{Change problem} over weakly connected graph can not recover the the statistical accuracy of the centralized LASSO estimator. {\bf (ii)} The weakly connected graph requires a smaller $\gamma$ ($=1\times 10^{-4}$)  to recover the centralized statistical error; which is consistent with the dependence of $\gamma$ on $\rho$ as  in \eqref{stat gamma condition}. {\bf (iii)} The range of $\lambda$ guaranteeing the minimal $\ell_2$- error in both the centralized and distributed penalty LASSO is {comparable}, as predicted by  condition  \eqref{stat N condition} {on $\lambda$}.
\begin{figure}[H]
\begin{minipage}[t]{0.5\textwidth}
\centering
\includegraphics[width=6.5cm, height = 5cm]{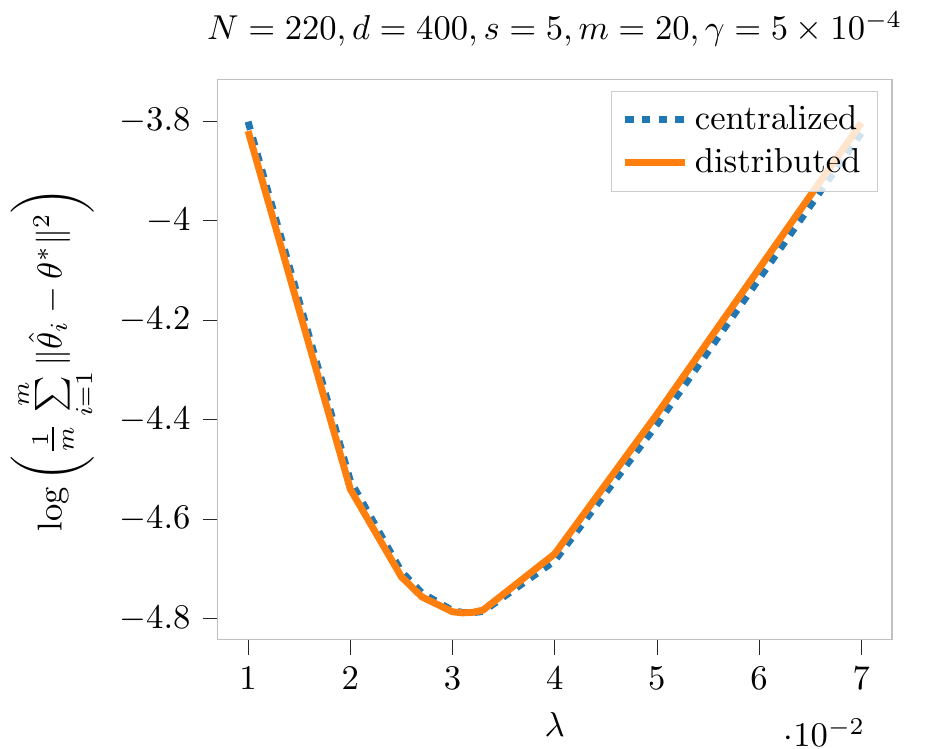}
\end{minipage}
\begin{minipage}[t]{0.5\textwidth}
\centering
\includegraphics[width=6.5cm, height = 5cm]{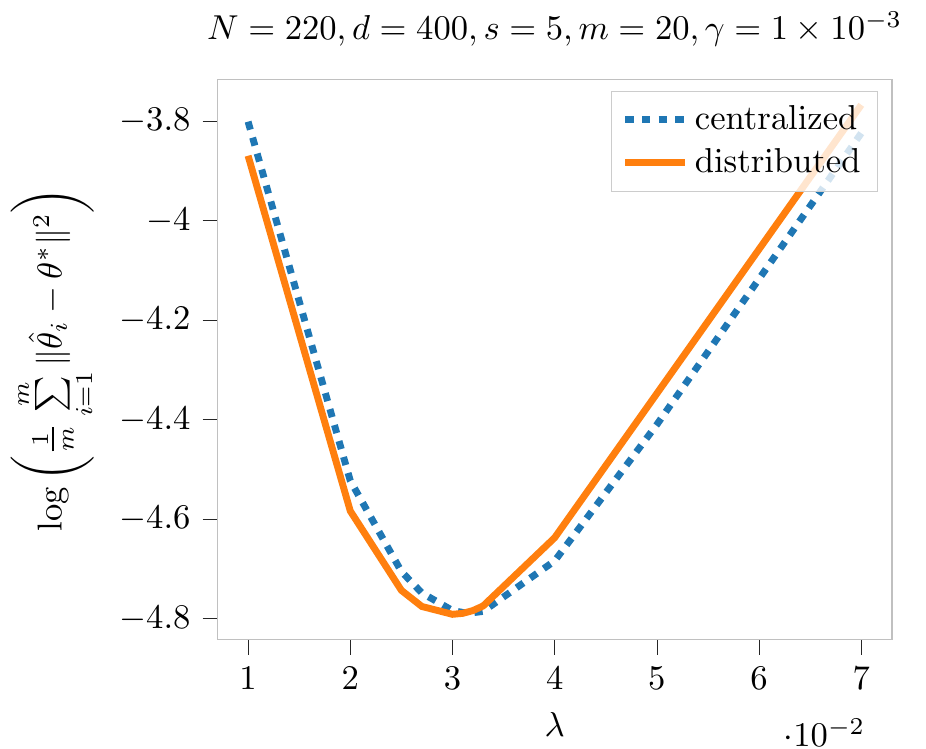}
\captionsetup{justification=centering,margin=2cm}
\end{minipage}
\begin{minipage}[t]{0.5\textwidth}
\centering
\includegraphics[width=6.5cm, height = 5cm]{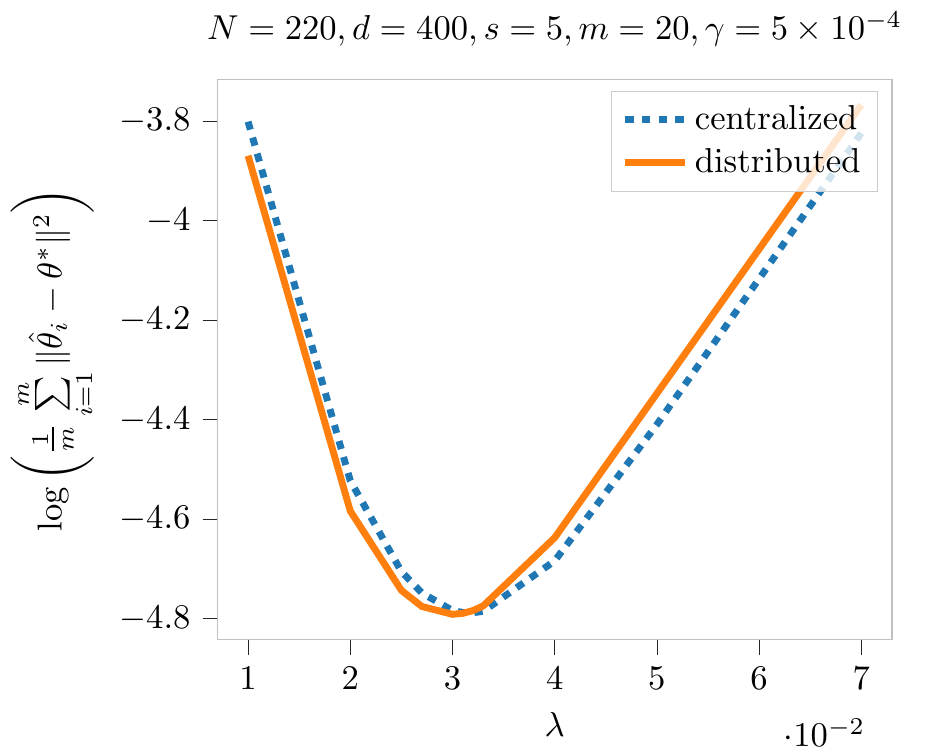}
\captionsetup{justification=centering,margin=2cm}
\end{minipage}
\begin{minipage}[t]{0.5\textwidth}
\centering
\includegraphics[width=6.5cm, height = 5cm]{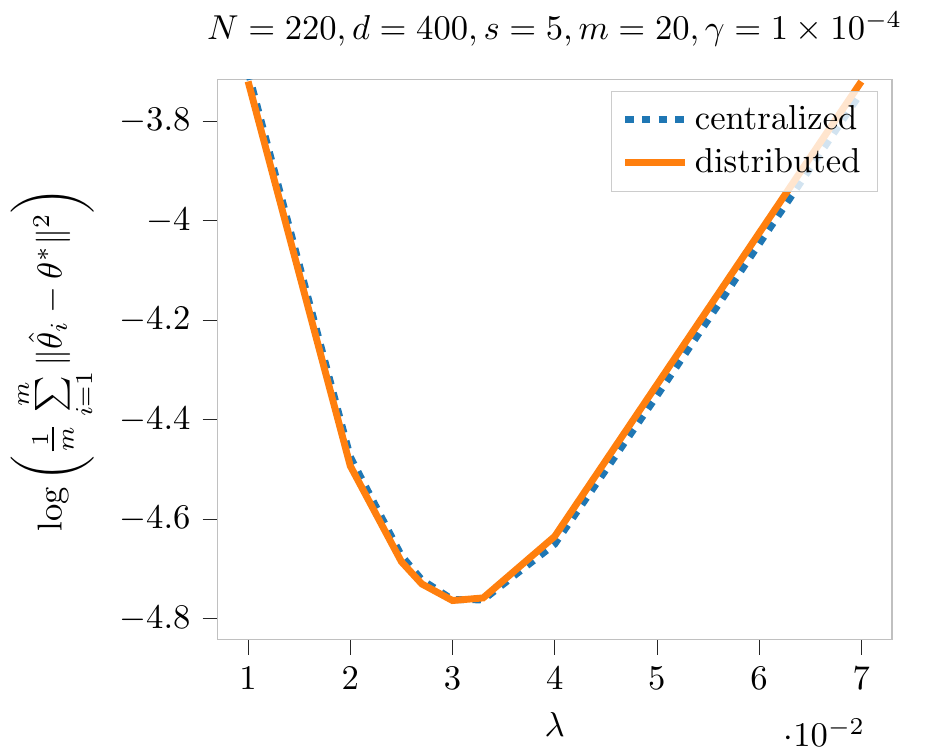}
\captionsetup{justification=centering,margin=2cm}
\end{minipage}
\caption{Statistical error of the estimator $\widehat{\bT}$  [see \eqref{Change problem}] and the centralized LASSO estimator $\hat{\T}$ [see \eqref{original problem}] versus $\lambda$, using synthetic data; \textbf{First row:} fully connected graph (${\rho=0.4897}$);  \textbf{Second row:}.  Erd\"os-R\'enyi graph with $p=0.1,$ ($\rho\approx0.973$).   Notice that our theory explains the behaviour of the curves only for  values of $\lambda\geq 0.033$ (as required by (\ref{stat gamma condition})).  }
\label{fig:fig4}
\end{figure}
\begin{figure}[t]
\begin{minipage}[t]{0.48\linewidth}
\centering
     \includegraphics[height = 5cm]{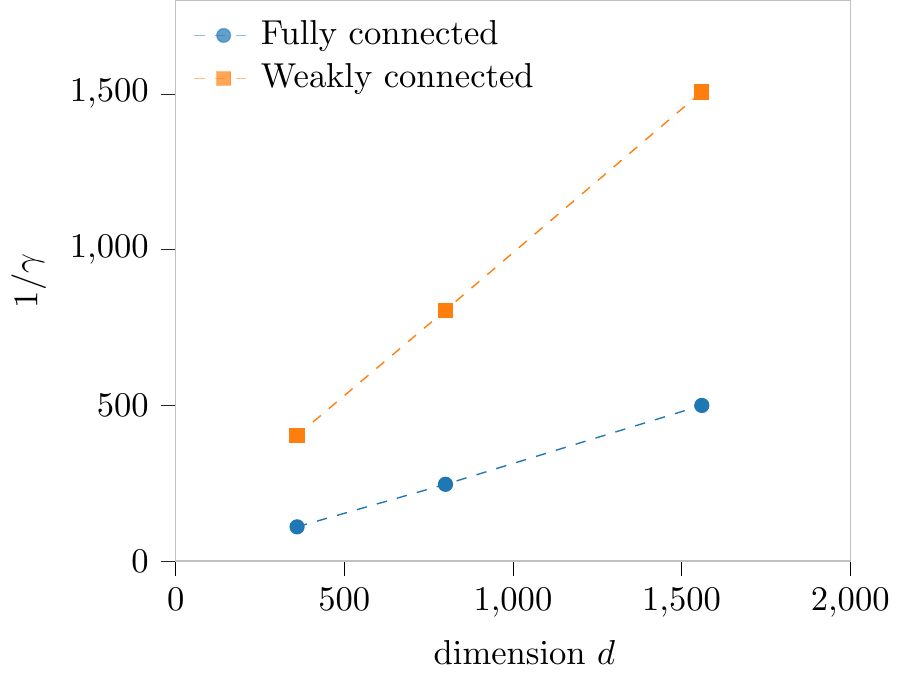}
 \caption{  {Validating  $\gamma=\mathcal{O}((1-\rho)/d)$ as predicted in \eqref{stat gamma condition}: Inverse of critical $\gamma$ (grid-searched) to retain  centralized statistical consistency versus dimension $d$;  $1/\gamma$ scales  linearly on $d.$ }}
 \label{fig:gamma}
\end{minipage}\hfill
\begin{minipage}[t]{0.48\linewidth}
\centering
\includegraphics[height = 5.35cm]{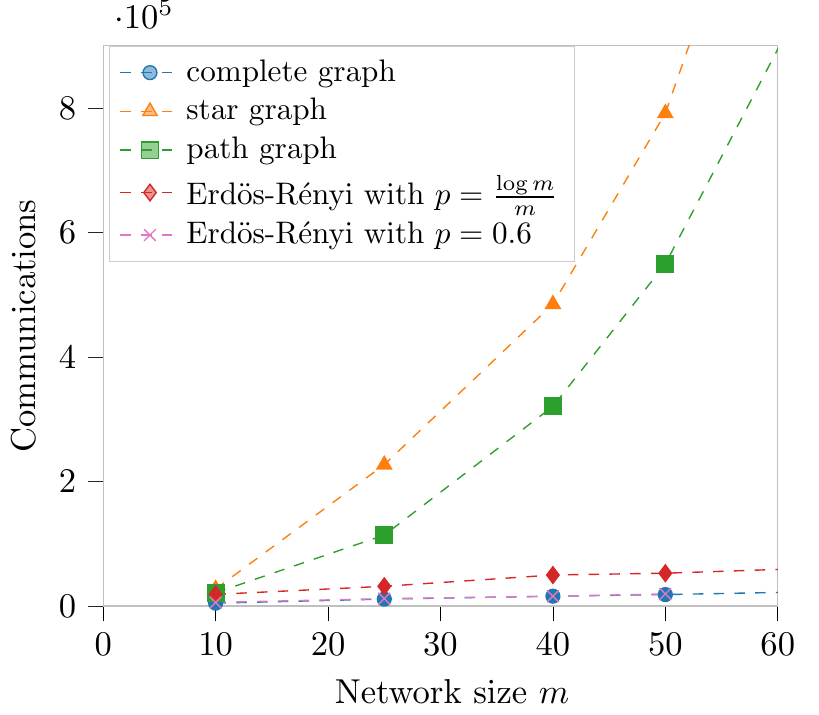} 
     \caption{  {Validating the      scaling $\tilde{\mathcal{O}}(m)$ of the communication complexity as predicted in \eqref{eq:numb_comm}: Communication rounds versus network size $m$ to achieve centralized statistical accuracy. }}
     \label{fig:scaling_m}
\end{minipage}
 
\end{figure}
  {\bf 2) Validating  $\gamma=\mathcal{O}((1-\rho)/d)$ in \eqref{stat gamma condition} (Theorem~\ref{statistical optimization error result}).}   { Fig.~\ref{fig:gamma} plots the inverse of largest $\gamma$ (grid-searched) that guarantees
centralized statistical accuracy versus the dimension, for three choices of  $(N,d,s)$,  corresponding to increasing values of $d,$ $s=\lceil\log d\rceil, m=5$ and adjust $N$ such that roughly  constant  $s\log d/N$ (and so the centralized statistical error).  
The figure shows that, as  $d$ increases, a smaller $\gamma$ is required to preserve   centralized statistical errors. The scaling of such a $\gamma$ is roughly $\mathcal{O}(1/d)$, validating  the dimension-dependence of recovery predicted in  \eqref{stat gamma condition}. Notice also that a weaker  connected graph requires smaller $\gamma$ to recover the  centralized statistical error, and the slope of weakly connected graph (yellow, $\rho=0.9045$) is larger than that of the fully connected (blue, $\rho=0.4897$), 
which is consistent with the dependence of $1/\gamma=\mathcal{O}(d/(1-\rho))$ as proved  in \eqref{stat gamma condition}. }


{\bf 3) Linear convergence and the speed-accuracy dilemma (Theorem \ref{Th_stat_convergence}).} Fig.~\ref{fig:speed accuracy} plots the log average optimization error versus  the number of iterations generated by the distributed proximal-gradient Algorithm~\eqref{composite dgd variant}, in the same setting of Fig.~\ref{fig:gamma}. 
As predicted by Theorem \ref{Th_stat_convergence}, linear convergence within the centralized statistical error is achievable when $d,N\to \infty$ and $s\log d/N=o(1)$, but at a rate scaling with $\mathcal{O}(d)$, revealing the the speed-accuracy dilemma. 
 \begin{figure}[H]
\begin{minipage}[t]{0.5\textwidth}
\centering
\includegraphics[height = 5cm]{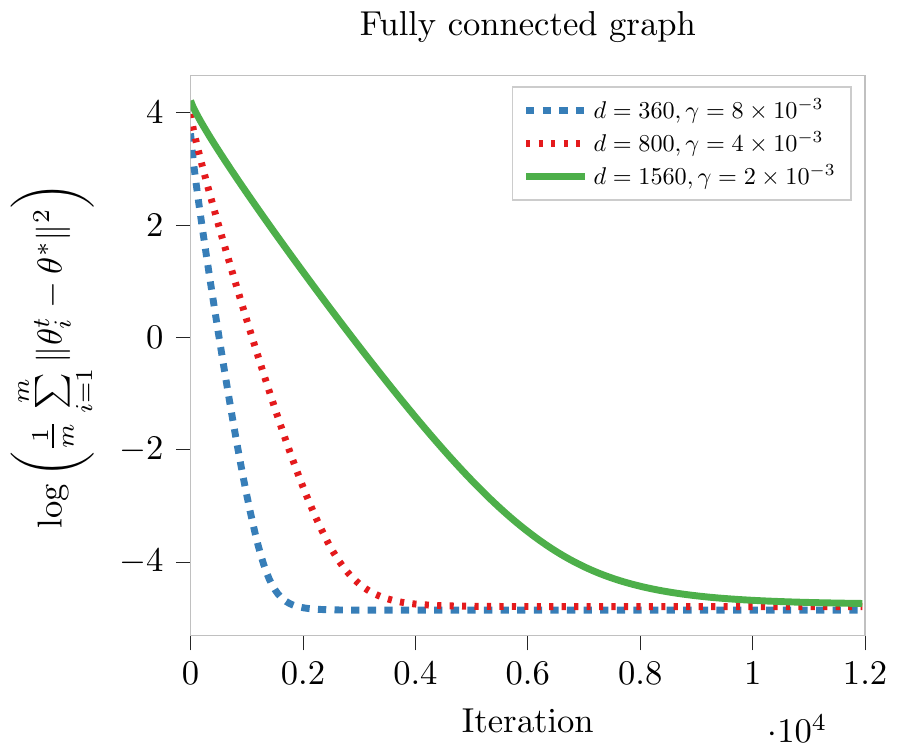} 
\end{minipage} 
\begin{minipage}[t]{0.5\textwidth}
\centering
\includegraphics[height = 5cm]{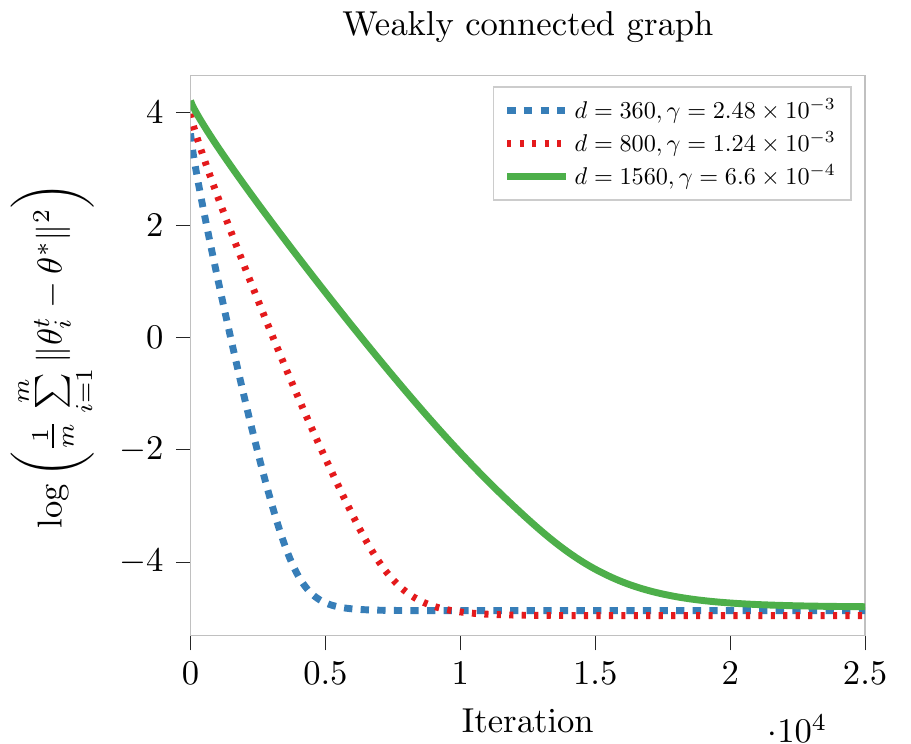} 
\end{minipage}
\caption{ Linear convergence of Algorithm~\eqref{composite dgd variant} up to the centralized statistical error: estimation error generated  by Algorithm~\eqref{regularized algorith} versus iterations (communications), using synthetic data. \textbf{Left panel:} fully connected graph (${\rho=0.4897}$);  \textbf{Right panel:}.  Erd\"os-R\'enyi  graph with $p=0.1,$ ($\rho\approx0.9045$). As predicted by our theory,    the scaling of $\gamma$   to recover centralized statistical consistency is {$\gamma= \Theta(1/d)$}: As  $d$ roughly doubles, going  from $360$ to $800$,  $\gamma$ decreases by half. The same scaling is observed when  $d$ goes  from $800$ to $ 1560$, revealing the the speed-accuracy dilemma.
    \label{fig:speed accuracy}}\vspace{-.3cm}
\end{figure}
{{\bf 4) On the dependence of communication rounds  on network size $m.$}   To   underscore the    aforementioned dependence, we carried out experiments across an array of network topologies. This comprised of three deterministic graphs--the complete graph, path graph, and star graph--and two random graphs--the Erd\"os-R\'enyi graph with $p=\mathcal{O}(1)$ (specifically $p=0.6$) and $p=\mathcal{O}(\log m/m)$ (specifically $p=\log m/m$), resulting in both cases a connectivity $\rho$ roughly constant with $m$ with high-probability \citep[Proposition 5]{Nedic_Olshevsky_Rabbat2018}. In each topology, we progressively augmented the number of nodes, $m$, in increments of $10, 25, 40$, and $50$, while maintaining a consistent total sample size of $200$ and dimension of $400$.  
We sought the largest   $\gamma$ (grid-searched) and the least  number of communications for each pairing of $m$ and graph type  that attain   centralized statistical errors (within $3\%$ accuracy). 
 Results are presented in Fig.~\ref{fig:scaling_m}, where we plotted such a number of communications versus $m$ for the aforementioned topologies.     {   Notice that the communications' scaling is linear with $m$ for the complete graph and  Erd\"os-R\'enyi with $p=\mathcal{O}({\log m}/{m})$ and $p=\mathcal{O}(1)$. Given that we approximately achieve $1/(1-\rho(m))=\mathcal{O}(1)$ for these three topologies \cite{Nedic_Olshevsky_Rabbat2018}, this result confirms the validity of \eqref{eq:numb_comm}, which anticipates $\widetilde{\mathcal{O}}(m)$ under such settings.}
}
 \begin{figure}[H]
 \begin{minipage}[t]{0.5\textwidth}
\centering
\includegraphics[width=6.5cm, height = 5cm]{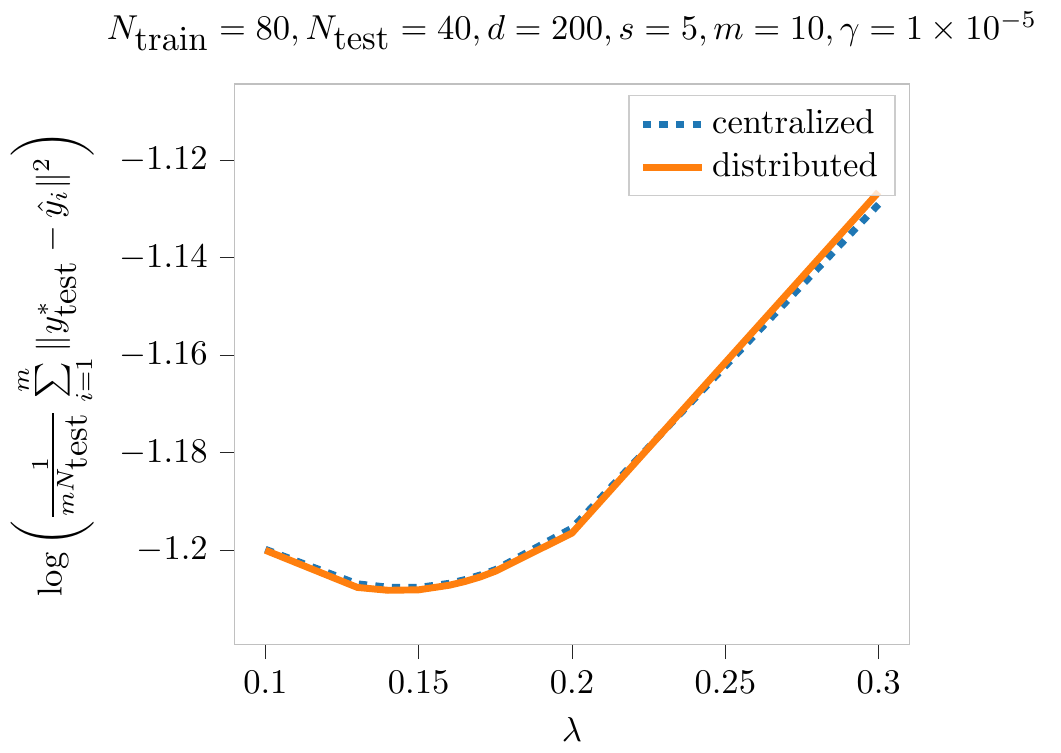}
\end{minipage}
 \begin{minipage}[t]{0.5\textwidth}
\centering
\includegraphics[width=6.5cm, height = 5cm]{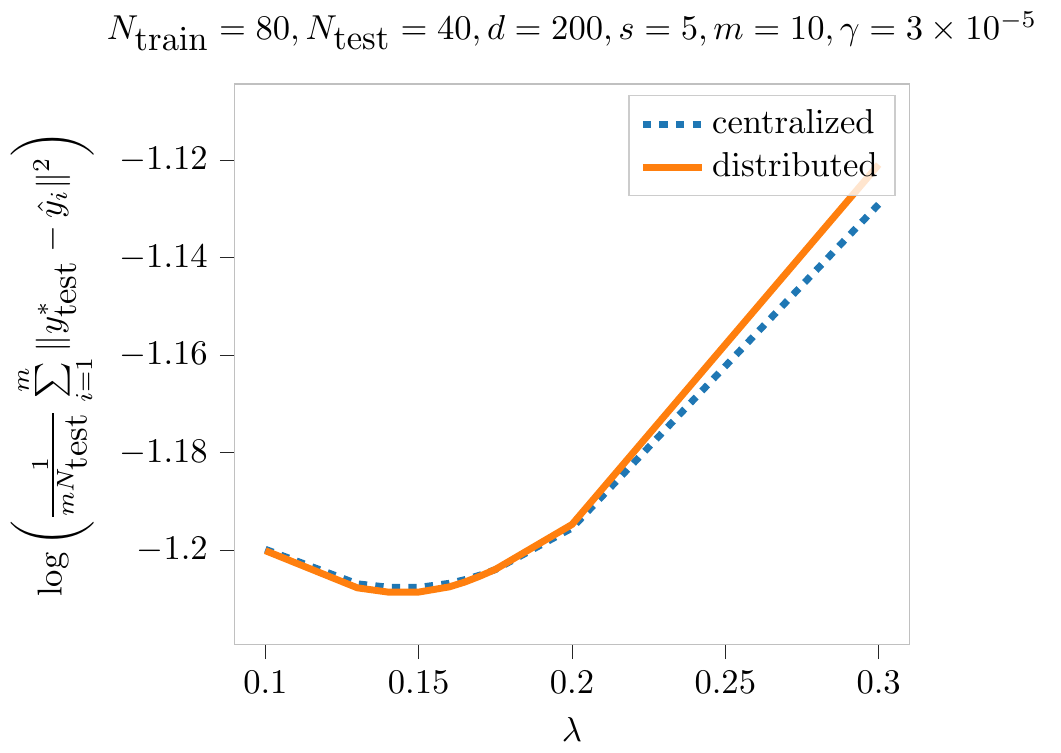}
\end{minipage}\vspace{0.4cm}
\begin{minipage}[t]{0.5\textwidth}
\centering
\includegraphics[width=6.5cm, height = 5cm]{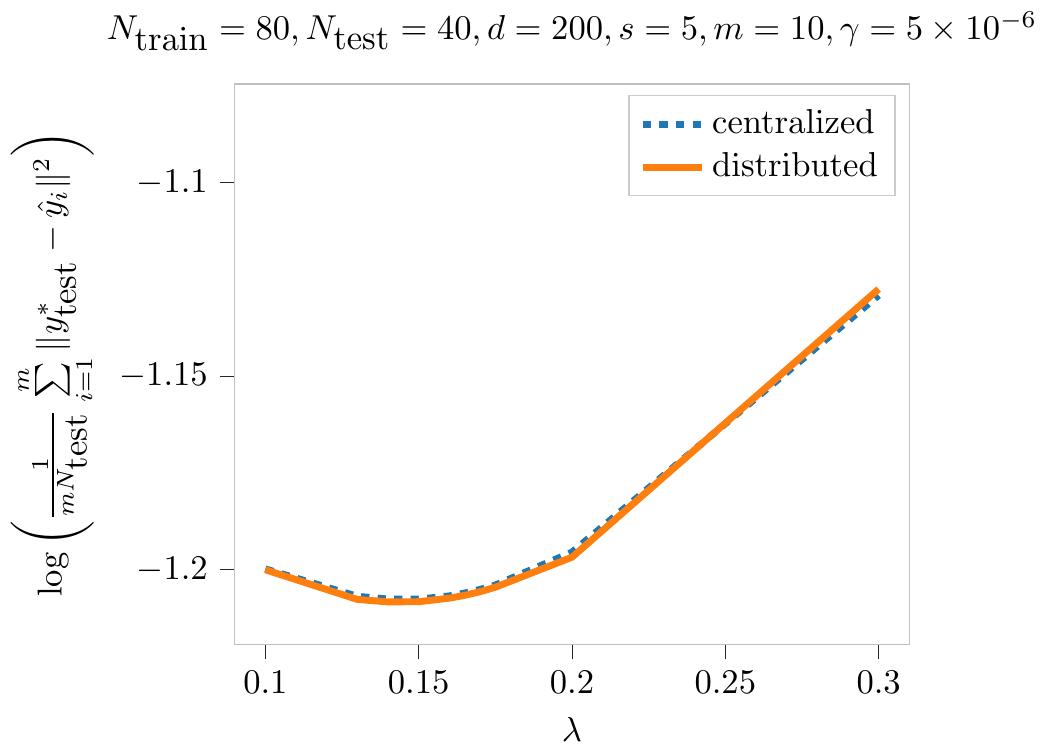}
\end{minipage}
 \begin{minipage}[t]{0.5\textwidth}
\centering
\includegraphics[width=6.5cm, height = 5cm]{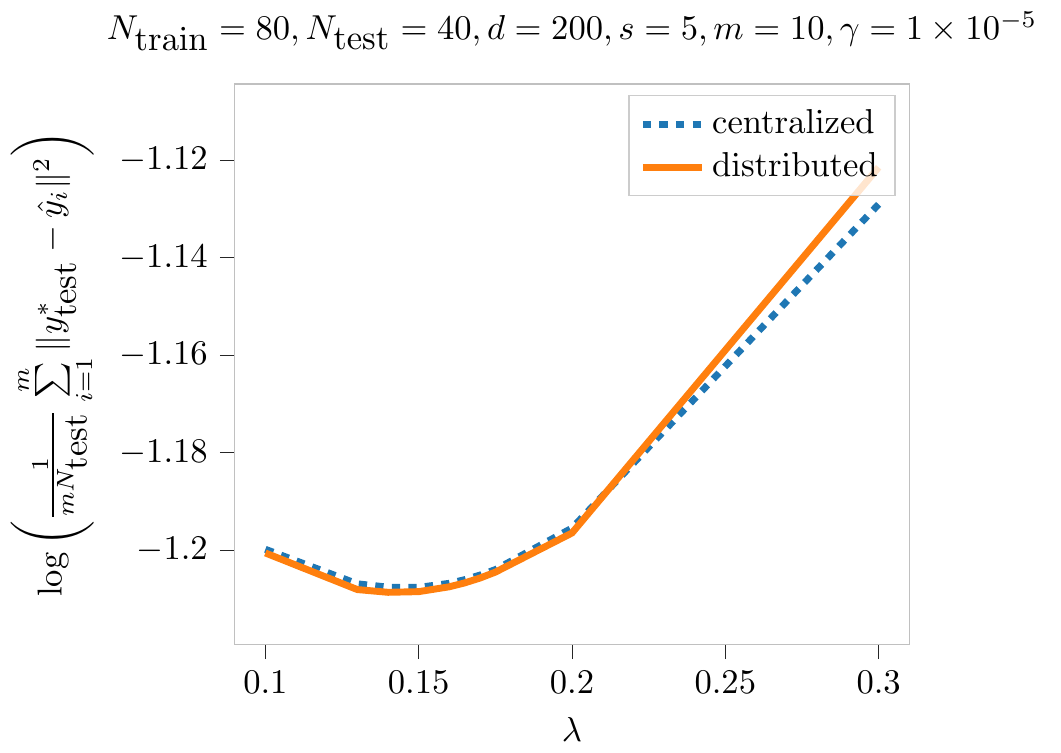}
\end{minipage}
\caption{$\texttt{MSE}^\infty$   defined in \eqref{MSE} associated with  the estimator $\widehat{\bT}$  [see \eqref{Change problem}] and the centralized LASSO estimator $\hat{\T}$ [see \eqref{original problem}] versus $\lambda$ using  the dataset \texttt{eyedata} in the \texttt{NormalBeta}-\texttt{Prime} package. \textbf{First row:} fully connected graph (${\rho=0.4897}$);    \textbf{Second row:}  Erd\"os-R\'enyi graph with $p=0.1,$ ($\rho\approx0.971$).  Notice that our theory explains the behaviour of the curves only for  values of $\lambda\geq 0.15$ (as required by (\ref{stat gamma condition})).}
\label{fig:real_fully}
\end{figure}
\noindent {\bf Experiment on real data.}  We test our findings on  the dataset \texttt{eyedata} in the \texttt{NormalBeta}-\texttt{Prime} package \citep{eyedata}. This dataset contains gene expression data of $d=200$ genes, and  $N=120$ samples. Data originate from microarray experiments of mammalian eye tissue samples. We randomly divide the dataset into training sample set with size $N_{\text{train}}=80$ and test dataset with size $N_{\text{test}}=40.$ We partition the training data into $m=10$ subsets. Each agent $i$ owns the data set portion with size $8.$ We run Monte Carlo simulations, with $30$ repetitions.  Since we  do not have access of the ground truth $\T^*,$   we replace the $\ell_2$ statistical error and the  $\ell_2$ optimization error  with the MSE errors 
 \begin{equation}\label{MSE}
\texttt{MSE}^\infty\triangleq      \frac{1}{mN_{\text{test}}}\sum\limits_{i=1}^m\|y^*_{\text{test}}-\hat{y}_i\|^2\quad \text{and}\quad     \texttt{MSE}^t\triangleq\frac{1}{mN_{\textnormal{test}}}\sum\limits_{i=1}^m\|y^*_{\textnormal{test}}-y_i^t\|^2,
 \end{equation} respectively,   where  $y^*_{\text{test}}$ is the output of the test set, and $\hat{y}_i=X_i\hat{\theta}_i$,  $i\in [m],$ are the model forecasts;  $\hat{y}_i^t=X_i\theta^t_i$,  $i\in [m],$ are output at iteration $t.$  $m=1$ corresponds to the centralized case, with $\mathbf{\hat{y}}=\mathbf{X}\hat{\theta}.$ 

  Our {first} experiment is meant to check whether the solution of the penalized problem \eqref{Change problem} matches  the  solution of the centralized LASSO  via a proper choice of $\gamma$. Fig.~\ref{fig:real_fully} plots the MSE (log scale) vs. $\gamma$ achieved by   Algorithm~\eqref{composite dgd variant} over a fully connected graph (\textbf{left panel}) and a weakly Erd\"os-R\'enyi graph with $p=0.1,$ resulting in $\rho\approx0.71$ (\textbf{right panel}).  The results confirm what we have already observed on synthetic data. 
  
\begin{figure}[ht]
\begin{minipage}[t]{0.5\textwidth}
\centering
\includegraphics[height = 6cm]{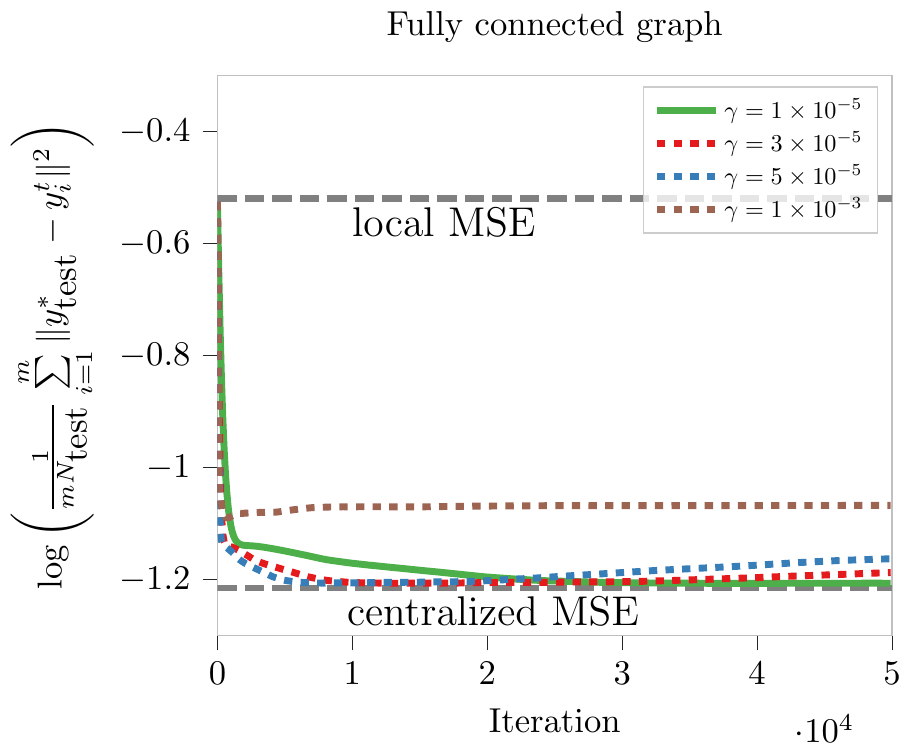} 
\end{minipage} 
\begin{minipage}[t]{0.5\textwidth}
\centering
\includegraphics[height = 6cm]{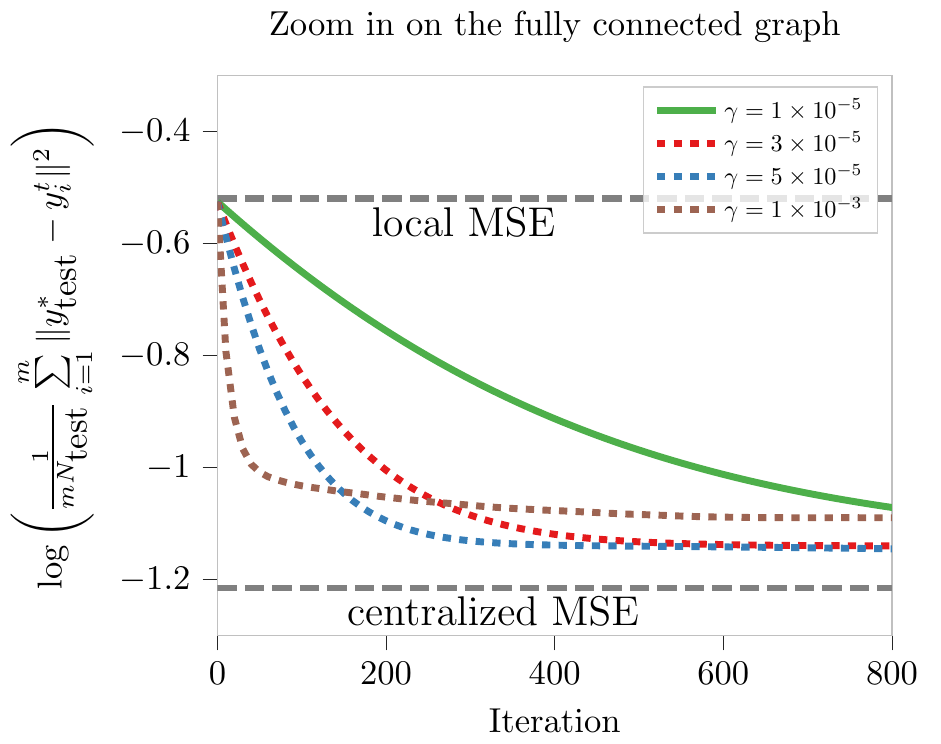} 
\end{minipage}\vspace{+0.7cm}
\begin{minipage}[t]{0.5\textwidth}
\centering
\includegraphics[height = 6cm]{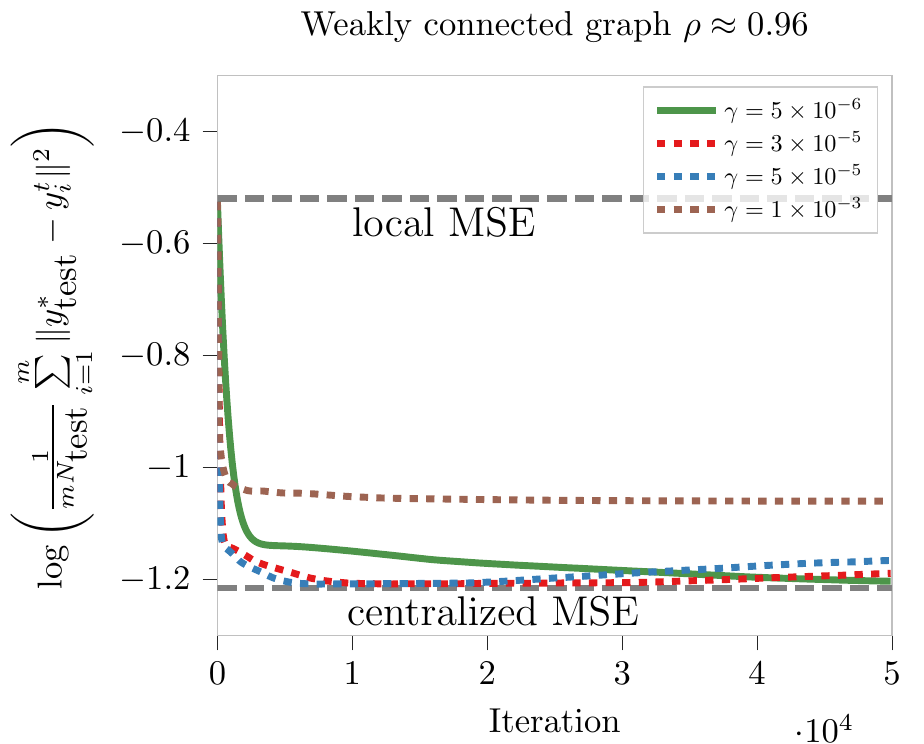} 
\end{minipage}\vspace{0.2cm}
\begin{minipage}[t]{0.5\textwidth}
\centering
\includegraphics[height = 6cm]{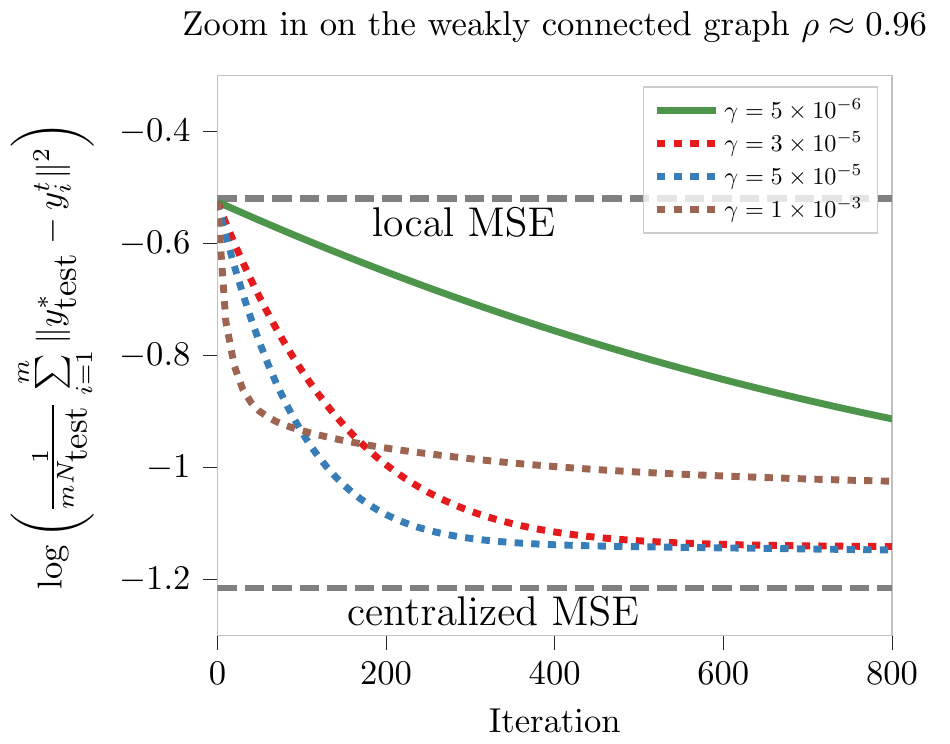} 
\end{minipage}\smallskip 
\caption{  Linear convergence of Algorithm~\eqref{composite dgd variant} up to the centralized statistical error, for different values of $\gamma$,  using    the dataset \texttt{eyedata} in the \texttt{NormalBeta}-\texttt{Prime} package: $\texttt{MSE}^t$ defined in  \eqref{MSE}  versus  the number of iterations (communications).   \textbf{First row:} fully connected graph (${\rho=0.4897}$);  \textbf{Second row:}  Erd\"os-R\'enyi graph with $p=0.1,$ ($\rho\approx0.96$).  \textbf{Left panel:} iterations up to $5\times 10^4.$ \textbf{Right panel:} zoom in on the iterations up to $8\times 10^2.$ 
    \label{fig:speed accuracy-real_data}}\label{fig:speed-accuracy-dilemma}\vspace{-.3cm}
\end{figure}
  Our second experiment on real data is to validate the speed-accuracy dilemma, postulated by our theory and already validated on synthetic data (cf.~Fig.\ref{fig:speed accuracy}).    Fig.~\ref{fig:speed accuracy-real_data} plots the log average optimization error versus  the number of iterations generated by   Algorithm~\eqref{composite dgd variant}, in the same network setting as for  Fig.~\ref{fig:real_fully}; different curves refer to different values of the penalty parameter $\gamma$.   Since $\theta^*$ is no longer available when using real data, we heuristically set $R$ in the projection \eqref{composite dgd variant}  as $R=\max_{1\leq i\leq m}\|\hat{\T}_i\|_1$.  This max-quantity can be obtained locally by each agent by running a min-consensus algorithms, requiring a number of communications of the order of the diameter of the network.  The figure still shows   linear convergence up to some tolerance, which is of the order of the MSE error in (\ref{MSE}). Even on real data the speed-accuracy dilemma is evident.

\section{Concluding Remarks}\label{sec_conclusions}
We studied sparse linear regression over mesh networks. 
We established statistical and computational guarantees in the high-dimensional regime of a penalty-based consensus formulation and associated distributed proximal gradient method. This is the first attempt of studying the behaviour of a distributed method in the high-dimensional regime; our interest in penalty-based   formulations  to decentralize the optimization/estimation was motivated by their popularity and early adoption  in the literature of distributed optimization (low-dimensional regime). We proved that optimal sample complexity $\mathcal{O}(s\log d/N)$  for the distributed estimator is achievable over  networks, {even when   {\it local sample size is not sufficient for statistical consistency}. This contrasts with D\&C methods which impose a condition on the local samples size (let alone they are readily    implementable over mesh networks). } On the computational side, such statistically optimal estimates can be achieved by the distributed proximal-gradient algorithm applied to the penalized problem, which converges at linear rate--such a rate however scales  as $\mathcal{O}(1/d)$, no matter how ``good''  the network connectivity is, {resulting in a total communication cost of $\mathcal{O}(d^2)$}.   

  { We claim that this unfavorable communication cost   is unavoidable for such penalty-based methods, because they  lack of any mechanism mixing directly local gradients    (they only  average  iterates).   This raises the question  whether communication costs of $\mathcal{O}(d)$ are achievable in high-dimension over mesh networks by other distributed, iterative   algorithms, yet with no conditions on the local sample size.  }
A first study towards this direction is  the companion work \citep{NetLASSO}, where the  projected  gradient algorithm \citep{sun2019distributed} based on gradient tracking  is studied in the high-dimensional setting. {The analysis  of other distributed methods employing other forms of gradient correction, such as primal-dual method as in  \citep{jakovetic2011cooperative,shi2014linear,Jakovetic:da,6926737,shi2015extra,shi2015proximal} remains an interesting topic for future investigation. }

\section{Acknowledge}
The work of Ji, Scutari, and Sun has been partially supported by the Office of Naval Research, under the Grant \# N00014-21-1-2673.

\section*{Appendix}
\appendix\markboth{Appendix}{Appendix}
\renewcommand{\thesection}{\Alph{section}}

  In this appendix we present the proofs of the results in the paper.
  We will use the same notation as in the paper along with the following additional definitions. 
  
Recall the statistical error $\hat{\bnu} \triangleq \hat{\bT} - 1_m \otimes \T^*.$ For any $\bT\in \mathbb{R}^{md}$  partitioned as $\bT=[\T_1,\ldots ,\T_m]$, with each $\T_i\in \mathbb{R}^d$,  we define  
\begin{equation}
    \bnu \triangleq \bT - 1_m \otimes \T^*.
\end{equation}  When needed, we decompose $\bT$, and accordingly ${\bnu}$, in its average and orthogonal component 
\begin{align}\label{def:nu-prop1}
     {\bnu}=1_m\otimes {\nu}_{\text{av}}+ {\bnu}_{\perp}, \quad \text{with} \quad  {\nu}_{\text{av}}=\frac{1}{m}\sum_{i=1}^m{\nu}_{i}.
\end{align}
In particular, when $\bT=\hat{\bT}$, we will write for the   augmented LASSO error \begin{equation}\label{eq:nu_hat_def}
    \hat{\bnu}\triangleq\hat{\bT}-1_m\otimes\T^*\quad\text{and}\quad \hat{\bnu}= 1_m\otimes \hat{\nu}_{\text{av}}+ \hat{\bnu}_{\perp},
\end{equation} 
 whereas when $\bT=\bT^t$, with $\bT^t$ being the iterates generated by Algorithm~\eqref{regularized algorith}  , we will write
 \begin{equation}\label{eq_nu_t_def}
     {\bnu}^t\triangleq{\bT}^t-1_m\otimes\T^*\quad\text{and}\quad  {\bnu}^t= 1_m\otimes  {\nu}^t_{\text{av}}+  {\bnu}^t_{\perp}.
 \end{equation}
 Finally, the optimization error (along with its decomposition in   average and orthogonal component) is denoted by 
  \begin{equation}\label{eq:Delta_def}
      {\bDelta}^t\triangleq \bT^t-\hat{\bT}\quad \text{and}\quad \boldsymbol{\Delta}^t=1_m \otimes \Delta_{\text{av}}^t+\boldsymbol{\Delta}_{\perp}^t.
 \end{equation}

Table~\ref{table:numerical constants} below summarizes all the universal constants used in the paper along with their range of values and associated constraints. 
   \begin{table}[h!]
 \begin{center}\resizebox{0.9\columnwidth}{!}{
 \begin{tabular}{ccccccccc}
                    universal constant  & $\tilde{c}_0$     & $\tilde{c}_1$ & $\tilde{c}_2$ & $\tilde{c}_3$ & $\;\tilde{c}_4\;$ & $\tilde{c}_5$  & $\tilde{c}_6$   &  $\tilde{c}_7$ \\[2ex]\hline\\[1ex]
value   &$>0$& $>0$    & $>0$    & $>0$ &$\;\;\;\max\{2, 4 \tilde{c}_2^2 , 4 \tilde{c}_2^2 \tilde{c}_3^{-1}\}\;\;\;$&$> 32$ &$\;\;\;\tilde{c}_5/32-1\;\;\;$ &free \\
\\[1ex]\hline                          \end{tabular}  }\end{center} 

\end{table}

\begin{table}[h!]
  \begin{center}\resizebox{0.9\columnwidth}{!}{
 \begin{tabular}{cccccccccc}
 universal constant  & $\tilde{c}_8$ &     $\tilde{c}_9$ & $\tilde{c}_{10}$ & $\tilde{c}_{11}$ & $\tilde{c}_{12}$ &  $\tilde{c}_{13}$  & $\tilde{c}_{14}$&   $\tilde{c}_{15}$  & $\tilde{c}_{16}$\\[2ex]\hline\\[1ex]
value  &$\geq\sqrt{6}$ & $\max\{128\tilde{c}_1,\tilde{c}_5\}$& $1824$ &$3648\tilde{c}_1$ & $\max\{3648\tilde{c}_1, \tilde{c}_5\}$ & $2731\tilde{c}_1^2/t_0$  &  $1152$  & $57\sqrt{6}\tilde{c}_8$  &         $\sqrt{6}\tilde{c}_8/64\tilde{c}_1$  \\
\\[1ex]\hline                          \end{tabular} }\end{center} 
\end{table}

\begin{table}[h!]
\begin{center}\resizebox{0.9\columnwidth}{!}{
 \begin{tabular}{ccccccccccc}
                  universal constant      & $\tilde{c}_{17}$   & $\tilde{c}_{18}$  & $\tilde{c}_{19}$& $\tilde{c}_{20}$ & $\tilde{c}_{21}$& $\tilde{c}_{22}$ &$\tilde{c}_{23}$   & $   {\tilde{c}_{24}}$\\[2ex]\hline\\[1ex]
value              &     $9$ &      $72\tilde{c}_1\tilde{c}_{17}$   & $\tilde{c}_1\tilde{c}_8^2\tilde{c}_{17}/988$ & $8\tilde{c}_1\tilde{c}_{17}/\tilde{c}_8^2$ &$\tilde{c}_8^2/1976$& $11339\tilde{c}_1\tilde{c}_8^2$& $21130\tilde{c}_4$  &         $  {>0}$ \\
\\[1ex]\hline                          \end{tabular}} \caption{Universal constants used in the Appendix.} \label{table:numerical constants}
\end{center} 
\end{table}      


 \newpage
\section{Proof of Proposition \ref{proximal nu average sparsity}}\label{Proof of Proposition 1}

For the sake of convenience, let us rewrite the objective function in \eqref{Change problem} as 
\begin{equation}\label{eq:def_G}
    G(\bT)=L_\gamma(\bT)+\frac{\lambda}{m}\lVert\bT \rVert_1,\quad \text{with}\quad L_\gamma(\bT)= \frac{1}{2N}\sum\limits_{i=1}^m  {\lVert y_i-X_i\theta_i\rVert^2} + \frac{1}{2 m \gamma} \| \bT\|_{V}^2.\end{equation} 
By the optimality of  $\hat{\bT}$, it follows 
\begin{align*}
   & G(\hat{\bT})\leq G(1_m\otimes \T^*)\notag\\
   \Leftrightarrow&\frac{1}{2N}\sum\limits_{i=1}^m  {\lVert y_i-X_i\hat{\theta}_i\rVert^2} + \frac{1}{2 m \gamma} \| \hat{\bT}\|_{V}^2+\frac{\lambda}{m}\|\hat{\bT}\|_1 \notag\\
  & \leq \frac{1}{2N}\sum\limits_{i=1}^m  {\lVert y_i-X_i\theta^*\rVert^2} + \underbrace{\frac{1}{2 m \gamma} \| 1_m\otimes \theta^*\|_{V}^2}_{=0\, \text{ (Assumption~\ref{W})}}+\frac{\lambda}{m}\| 1_m\otimes \theta^*\|_1.
\end{align*}
Using  $y_i=X_i\theta^*+w_i$ and the fact that $\T^\ast$ is $\mathcal{S}$-sparse,  we can write
\begin{align*}
     &\frac{1}{N}\sum\limits_{i=1}^m\lVert X_i\hat{\theta}_i-X_i\T^*\rVert^2
     {\leq} \frac{2}{N}\sum\limits_{i=1}^m w_i^{\top}X_i\hat{\nu}_i+\frac{2\lambda}{m}\sum\limits_{i=1}^m(\lVert(\hat{\nu}_i)_{\mathcal{S}}\rVert_1-\lVert(\hat{\nu}_i)_{ \mathcal{S}^c}\rVert_1)-\frac{1}{m\gamma} \lVert\hat{\bT}\rVert_{V}^2. 
\end{align*}
Using   the factorization  $\hat{\bnu}=1_m\otimes\hat{\nu}_{\text{av}}+\hat{\bnu}_{\perp},$  the above bounds reads
\begin{align*}
   &\frac{1}{N}\sum\limits_{i=1}^m\lVert X_i\hat{\theta}_i-X_i\T^*\rVert^2\notag\notag\\
    \leq\ &\frac{2}{N}\sum\limits_{i=1}^m w_i^{\top}X_i(\hat{\nu}_{\text{av}}+\hat{\nu}_{\perp i}) +\frac{2\lambda}{m}\sum\limits_{i=1}^m(\lVert(\hat{\nu}_{\text{av}})_{\mathcal{S}}+(\hat{\nu}_{\perp i })_\mathcal{S}\rVert_1 -\lVert(\hat{\nu}_{\text{av}})_{\mathcal{S}^c}+(\hat{\nu}_{\perp i })_{\mathcal{S}^c}\rVert_1)-\frac{1}{m\gamma} \lVert\hat{\bT}\rVert_{V}^2\notag\\
     \overset{{(a)}}{\leq}
    &\frac{2}{N}w^{\top}\mathbf{X}\hat{\nu}_{\text{av}} +\frac{2}{N}\sum\limits_{i=1}^m w_i^{\top}X_i\hat{\nu}_{\perp i}+\frac{2\lambda}{m}\sum\limits_{i=1}^m(\lVert(\hat{\nu}_{\text{av}})_{\mathcal{S}}\rVert_1-\lVert(\hat{\nu}_{\text{av}})_{\mathcal{S}^c}\rVert_1) +\frac{2\lambda}{m}\sum\limits_{i=1}^m\lVert\hat{\nu}_{\perp i}\rVert_1-\frac{1}{m\gamma}\lVert\hat{\bnu}_{\perp}\rVert_{V}^2\notag,
\end{align*}
where in (a) we used    $\sum\limits_{i=1}^m w_i^{\top}X_i\hat{\nu}_{\text{av}}=\mathbf{w}^{\top}\mathbf{X}\,\hat{\nu}_{\text{av}}$ and $\lVert\hat{\bT}\rVert_{V}^2 =\lVert\hat{\bT}-1_m\otimes\T^*\rVert_{V}^2=\lVert\hat{\bnu}_{\perp}\rVert_{V}^2$.
 

We bound now 
the two  terms $\mathbf{w}^{\top}\mathbf{X}\hat{\nu}_{\text{av}}$ and $\sum\limits_{i=1}^m w_i^{\top}X_i\hat{\nu}_{\perp i}$. We have  
    \begin{align}\label{intermediate_step}
     &\frac{1}{N}\sum\limits_{i=1}^m\lVert X_i\hat{\theta}_i-X_i\T^*\rVert^2\notag \\
  \overset{\text{H\"older's}}{\leq}\ & {\frac{2}{N}\lVert \mathbf{w}^{\top}\mathbf{X}\rVert_{\infty}\lVert\hat{\nu}_{\text{av}}\rVert_1+\max\limits_{i\in [m]}\lVert w_i^{\top} X_i\rVert_{\infty}\,\frac{2}{N}\sum\limits_{i=1}^m\lVert\hat{\nu}_{\perp i}\rVert_1}\notag\\
    &+\frac{2\lambda}{m}\sum\limits_{i=1}^m(\lVert(\hat{\nu}_{\text{av}})_{\mathcal{S}}\rVert_1-\lVert(\hat{\nu}_{\text{av}})_{\mathcal{S}^c}\rVert_1) +\frac{2\lambda}{m}\sum\limits_{i=1}^m\lVert\hat{\nu}_{\perp i}\rVert_1-\frac{1}{m\gamma}\lVert\hat{\bnu}_{\perp}\rVert_{V}^2\notag\\
    \overset{\eqref{connectivity number},\eqref{distributed lambda}}{\leq} & {3\lambda\lVert(\hat{\nu}_{\text{av}})_{\mathcal{S}}\rVert_1-\lambda\lVert(\hat{\nu}_{\text{av}})_{\mathcal{S}^c}\rVert_1}
    \underbrace{{-\frac{1-\rho}{m\gamma} \lVert\hat{\bnu}_{\perp}\rVert^2
    +\bigg(\max\limits_{i\in[m]}\lVert w_i^{\top} X_i\rVert_{\infty}\frac{2}{N}+\frac{2\lambda}{m}\bigg)\sum\limits_{i=1}^m\lVert\hat{\nu}_{\perp i}\rVert_1}}_{\text{Term II}}. 
\end{align}
Finally, we bound Term II. Since we have no sparsity information on $\hat{\bnu}_{\perp}$, we can only assert that
$\lVert\hat{\nu}_{\perp i}\rVert_1\leq \sqrt{d}\lVert\hat{\nu}_{\perp i}\rVert$, for all  $i\in[m].$ Hence,
\begin{align}\label{median step for next lemma}
    \text{Term II}
    \leq-\frac{1-\rho}{m\gamma} \lVert\hat{\bnu}_{\perp}\rVert^2+\bigg(\max\limits_{i\in[m]}\lVert w_i^{\top} X_i\rVert_{\infty}\frac{2}{N}+\frac{2\lambda}{m}\bigg)\sum\limits_{i=1}^m\sqrt{d}\lVert\hat{\nu}_{\perp i}\rVert 
    \overset{\eqref{h}}{=}&\lambda\, h(\gamma,\lVert\hat{\bnu}_{\perp}\rVert). 
\end{align}
 Using (\ref{median step for next lemma}) in (\ref{intermediate_step}), we finally obtain
  \begin{align*} 
     &\frac{1}{N\lambda}\sum\limits_{i=1}^m\lVert X_i\hat{\theta}_i-X_i\T^*\rVert^2\leq  {3 \lVert(\hat{\nu}_{\text{av}})_{\mathcal{S}}\rVert_1- \lVert(\hat{\nu}_{\text{av}})_{\mathcal{S}^c}\rVert_1} + h(\gamma,\lVert\hat{\bnu}_{\perp}\rVert),
\end{align*}
implying ${3 \lVert(\hat{\nu}_{\text{av}})_{\mathcal{S}}\rVert_1- \lVert(\hat{\nu}_{\text{av}})_{\mathcal{S}^c}\rVert_1} + h(\gamma,\lVert\hat{\bnu}_{\perp}\rVert)\geq 0$, which concludes the proof. $\hfill \square$

\section{Proof of Lemma \ref{global-ARE-determin} and Lemma~\ref{Global ARE}}\label{Proof of Lemma 1}
\subsection{Proof of Lemma~\ref{global-ARE-determin}} \label{Proof of Lemma 4}
 
Fix  $\gamma\in(0,(1-\rho)/L_{\max}]$  and let $\bDelta\in\mathbb{C}_{\gamma}(\mathcal{S})$. Then,  we have    
\begin{equation*}
\lVert(\Delta_{\text{av}})_{\mathcal{S}^c}\rVert_1\leq 3\lVert(\Delta_{\text{av}})_{\mathcal{S}}\rVert_1+h(\gamma,\lVert\bDelta_{\perp}\rVert),
\end{equation*}
where  $h(\gamma,\lVert\bDelta_{\perp}\rVert)$  is  defined in (\ref{h}). Substituting the above inequality into the RSC   condition~\eqref{Arsc},    yields
\begin{equation}
    \frac{1}{N}\lVert \mathbf{X}\Delta_{\text{av}}\rVert^2\geq \bigg(\frac{\mu}{2}-16s\tau\bigg)\lVert\Delta_{\text{av}}\rVert^2-\tau h^2(\gamma,\lVert\bDelta_{\perp}\rVert).
\end{equation}
Therefore,   for all $\bT$ and  $\bT'$, with  $\bT-\bT'\in\mathbb{C}_{\gamma}(\mathcal{S}),$ it holds
  {
\begin{align}
    \mathcal{T}_{L}({\bT};\bT')
    &\geq\bigg(\frac{\mu}{2}-16s\tau\bigg)\lVert(\bT-\bT')_{\text{av}}\rVert^2-\tau \,h^2(\gamma,\lVert(\bT-\bT')_{\perp}\rVert)-{\bigg(\frac{L_{\max}}{2m}-\frac{1-\rho}{2m\gamma}\bigg)\|(\bT-\bT')_{\perp}\|^2}\notag\\
    & {\geq} \delta\lVert(\bT-\bT')_{\text{av}}\rVert^2-\xi \,h^2(\gamma,\lVert(\bT-\bT')_{\perp}\rVert),
\end{align}
where   we used   $\gamma\in(0,(1-\rho)/L_{\max}],$
}
 and set   {$\delta=\mu/2-16s\tau,$ $\xi=\tau$}. This proves   (\ref{ARE eq}).
 $\hfill \square$
\subsection{Proof of Lemma~\ref{Global ARE}}\label{Proof of Lemma 5}
Let  $\mathbf{X}\in \mathbb{R}^{N\times d}$  be a design matrix satisfying Assumption~\ref{ass_X_random}. The RSC condition {\citep[Theorem 1]{Ras}} implies that there exist $\tilde{c}_0,\tilde{c}_{1}>0,$ such that for all $\Delta_{\text{av}}\in \mathbb{R}^{d},$
\begin{equation}\label{Ras inequality}
    \frac{1}{N}\lVert \mathbf{X}\Delta_{\text{av}}\rVert^2\geq\frac{1}{2}\lVert\Sigma^{\frac{1}{2}}\Delta_{\text{av}}\rVert^2-\frac{\tilde{c}_{1}\zeta_{\Sigma}\log d}{N}\lVert\Delta_{\text{av}}\rVert_1^2
\end{equation}
holds with probability at least $1-\exp(-\tilde{c}_0N).$
Furthermore, by condition (c) of $\mathbf{X},$ we have
\begin{equation}\label{eq:sigma_min}
    \lVert\Sigma^{\frac{1}{2}}\Delta_{\text{av}}\rVert^2\geq \lambda_{\min}(\Sigma)  \lVert\Delta_{\text{av}}\rVert^2.
\end{equation}
Let $\bDelta\in\mathbb{C}_{\gamma}(\mathcal{S}),$ that is,
\begin{equation}\label{eq:sparsity_cone}
    \lVert(\Delta_{\text{av}})_{\mathcal{S}^c}\rVert_1\leq 3\lVert(\Delta_{\text{av}})_{\mathcal{S}}\rVert_1+h(\gamma,\lVert\Delta_{\perp}\rVert).
\end{equation}
Substituting (\ref{eq:sigma_min}) and (\ref{eq:sparsity_cone})   into   (\ref{Ras inequality}), yields
\begin{align*}
     \frac{\lVert \mathbf{X}\Delta_{\text{av}}\rVert^2}{N}&\geq 
    \bigg(\frac{\lambda_{\min}(\Sigma)}{2}-\frac{32sc_{1}\zeta_{\Sigma}\log d}{N}\bigg)\lVert\Delta_{\text{av}}\rVert^2-\frac{2c_{1}\zeta_{\Sigma}\log d}{N}h^2(\gamma,\lVert\bDelta_{\perp}\rVert)\notag\\
     &\geq \frac{\lambda_{\min}(\Sigma)}{4}\lVert\Delta_{\text{av}}\rVert^2-\frac{\lambda_{\min}(\Sigma)}{64s}h^2(\gamma,\lVert\bDelta_{\perp}\rVert),\quad \text{for}\quad N\geq\frac{128s \tilde{c}_1\zeta_{\Sigma}\log d}{\lambda_{\min}(\Sigma)}.
\end{align*}
The proof follows using the above  bound in  (\ref{def:L-linearization}) along with    $\gamma\in(0,(1-\rho)/L_{\max}]. \hfill \square$
 
\section{Proof of Theorem \ref{solution err bound}}\label{Proof of Theorem 2}

Our starting point toward the upper bound on the average LASSO error $({1}/{m})\sum_{i=1}^m\lVert\hat{\nu}_i\rVert^2$ is    lower- and upper-bounding the average of local  errors $({1}/{N})\sum_{i=1}^m\lVert X_i\hat{\nu}_i\rVert^2$ while decomposing $\hat{\bT}$  in its average component and orthogonal one. This decomposition is instrumental to separate in the desired final bound a term of the same order of the centralized LASSO error from the (additive) perturbation due to the lack of exact consensus.  

\noindent {\bf $\bullet$ Step 1: Establishing the upper bound of $({1}/{N})\sum\limits_{i=1}^m\lVert X_i\hat{\nu}_i\rVert^2.$}

We start with the optimality condition of Problem (\ref{Change problem}). By optimality of   $\hat{\mathbf{\bT}}$, it follows that 
 
 \begin{equation}\label{eq:opt-cond-p4}
     \frac{1}{N}\sum_{i=1}^m \left(X_i \hat{\T}_i-y_i\right)^\top X_i\hat{\nu}_i \leq \frac{\lambda}{m}(\lVert1_m\otimes\theta^*\rVert_1-\lVert\hat{\bT}\rVert_1)+\frac{1}{2m\gamma}(\underbrace{\lVert1_m\otimes\T^*\rVert_{V}^2}_{=0\, \text{ (Assumption~\ref{W})}}-\lVert\hat{\bT}\rVert_{V}^2).
 \end{equation}
 We can then write
\begin{equation}\label{noparametric oracle}
\begin{aligned}
     \frac{1}{N}\sum\limits_{i=1}^m\lVert X_i\hat{\nu}_i\rVert^2 
   \overset{\eqref{eq:opt-cond-p4}}{\leq}  \frac{2}{N}\sum\limits_{i=1}^m  (y_i-X_i\T^*)^\top X_i \hat{\nu}_i+\frac{2\lambda}{m}(\lVert1_m\otimes\theta^*\rVert_1-\lVert\hat{\bT}\rVert_1)-\frac{1}{m\gamma}\lVert\hat{\bT}\rVert_{V}^2-\frac{1}{N}\sum\limits_{i=1}^m\lVert X_i\hat{\nu}_i\rVert^2\notag. 
\end{aligned}
\end{equation}
Using  $y_i=X_i\theta^*+w_i$ and the fact that $\T^\ast$ is $\mathcal{S}$-sparse,  we can write
\begin{align*}
     \frac{1}{N}\sum\limits_{i=1}^m\lVert X_i\hat{\nu}_i\rVert^2 
    \leq \frac{2}{N}\sum\limits_{i=1}^m w_i^{\top}X_i\hat{\nu}_i+\frac{2\lambda}{m}\sum\limits_{i=1}^m(\lVert(\hat{\nu}_i)_{\mathcal{S}}\rVert_1-\lVert(\hat{\nu}_i)_{\mathcal{S}^c}\rVert_1)-\frac{1}{m\gamma}\lVert\hat{\bT}\rVert_{V}^2-\frac{1}{N}\sum\limits_{i=1}^m\lVert X_i\hat{\nu}_i\rVert^2. 
\end{align*}
Introducing  the decomposition $\hat{\bnu}=1_m\otimes\hat{\nu}_{\text{av}}+\hat{\bnu}_{\perp},$  the above bound reads
\begin{align*}
   &\frac{2}{N}\sum\limits_{i=1}^m\lVert X_i\hat{\nu}_i\rVert^2\notag\notag\\
    \leq&\frac{2}{N}\sum\limits_{i=1}^m w_i^{\top}X_i(\hat{\nu}_{\text{av}}+\hat{\nu}_{\perp i}) +\frac{2\lambda}{m}\sum\limits_{i=1}^m(\lVert(\hat{\nu}_{\text{av}})_{\mathcal{S}}+(\hat{\nu}_{\perp i })_\mathcal{S}\rVert_1 -\lVert(\hat{\nu}_{\text{av}})_{\mathcal{S}^c}+(\hat{\nu}_{\perp i })_{\mathcal{S}^c}\rVert_1)-\frac{1}{m\gamma} \lVert\hat{\bT}\rVert_{V}^2\notag\\
     \overset{{(a)}}{\leq}
    &\frac{2}{N}\mathbf{w}^{\top}\mathbf{X}\hat{\nu}_{\text{av}} +\frac{2}{N}\sum\limits_{i=1}^m w_i^{\top}X_i\hat{\nu}_{\perp i}+\frac{2\lambda}{m}\sum\limits_{i=1}^m(\lVert(\hat{\nu}_{\text{av}})_{\mathcal{S}}\rVert_1-\lVert(\hat{\nu}_{\text{av}})_{\mathcal{S}^c}\rVert_1) +\frac{2\lambda}{m}\sum\limits_{i=1}^m\lVert\hat{\nu}_{\perp i}\rVert_1-\frac{1}{m\gamma}\lVert\hat{\bnu}_{\perp}\rVert_{V}^2\notag,
\end{align*}
where in (a) we used  $\sum\limits_{i=1}^m w_i^{\top}X_i\hat{\nu}_{\text{av}}=\mathbf{w}^{\top}\mathbf{X}\,\hat{\nu}_{\text{av}}$ and $\lVert\hat{\bT}\rVert_{V}^2 =\lVert\hat{\bT}-1_m\otimes\T^*\rVert_{V}^2=\lVert\hat{\bnu}_{\perp}\rVert_{V}^2.$

We bound now 
the two  terms $\mathbf{w}^{\top}\mathbf{X}\hat{\nu}_{\text{av}}$ and $\sum\limits_{i=1}^m w_i^{\top}X_i\hat{\nu}_{\perp i}$. We have  
    \begin{align*}
     &\frac{2}{N}\sum\limits_{i=1}^m\lVert X_i\hat{\theta}_i-X_i\T^*\rVert^2\notag \\
  \overset{ \text{H\"older's}}{\leq}\ & \frac{2}{N}\lVert \mathbf{w}^{\top}\mathbf{X}\rVert_{\infty}\lVert\hat{\nu}_{\text{av}}\rVert_1+\max\limits_{i\in [m]}\lVert w_i^{\top} X_i\rVert_{\infty}\,\frac{2}{N}\sum\limits_{i=1}^m\lVert\hat{\nu}_{\perp i}\rVert_1\notag\\
    &+\frac{2\lambda}{m}\sum\limits_{i=1}^m(\lVert(\hat{\nu}_{\text{av}})_{\mathcal{S}}\rVert_1-\lVert(\hat{\nu}_{\text{av}})_{\mathcal{S}^c}\rVert_1) +\frac{2\lambda}{m}\sum\limits_{i=1}^m\lVert\hat{\nu}_{\perp i}\rVert_1-\frac{1}{m\gamma}\lVert\hat{\bnu}_{\perp}\rVert_{V}^2\notag\\
    \overset{\eqref{connectivity number},(\ref{distributed lambda})}{\leq} & {3\lambda\lVert(\hat{\nu}_{\text{av}})_{\mathcal{S}}\rVert_1-\lambda\lVert(\hat{\nu}_{\text{av}})_{\mathcal{S}^c}\rVert_1}-\frac{1-\rho}{m\gamma} \lVert\hat{\bnu}_{\perp}\rVert^2
    +\bigg(\max\limits_{i\in[m]}\lVert w_i^{\top} X_i\rVert_{\infty}\frac{2}{N}+\frac{2\lambda}{m}\bigg)\sum\limits_{i=1}^m\lVert\hat{\nu}_{\perp i}\rVert_1.
\end{align*}

We further relax the bound by dropping  $-\lambda\lVert(\hat{\nu}_{\text{av}})_{\mathcal{S}^c}\rVert_1$ and enlarging $\lVert(\hat{\nu}_{\text{av}})_{\mathcal{S}}\rVert_1\leq \lVert \hat{\nu}_{\text{av}}\rVert_1$ while revealing the term $\frac{9\lambda^2s}{2\delta}$ which is of the order of the centralized LASSO error: 
\begin{align}\label{prediction error result}
    &\frac{2}{N}\sum\limits_{i=1}^m\lVert X_i\hat{\nu}_i\rVert^2 \notag\\
     \leq\ & 
      \, 2\cdot\frac{3\lambda\sqrt{s}}{\sqrt{2\delta}}\cdot\sqrt{\frac{\delta}{2}}\lVert\hat{\nu}_{\text{av}}\rVert-\frac{1-\rho}{m\gamma}\lVert\hat{\bnu}_{\perp}\rVert^2+\bigg(\max\limits_{i\in[m]}\lVert w_i^\top X_i\rVert_{\infty}\frac{2}{N}+\frac{2\lambda}{m}\bigg)\lVert\hat{\bnu}_{\perp}\rVert_1\notag\\
     \overset{\eqref{ARE eq}}{\leq}& \frac{9\lambda^2s}{2\delta}+\frac{1}{2}\bigg(\frac{\lVert \mathbf{X}\hat{\nu}_{\text{av}}\rVert^2}{N}+\xi h^2(\gamma,\lVert\hat{\bnu}_{\perp}\rVert)\bigg)-\frac{1-\rho}{m\gamma}\lVert\hat{\bnu}_{\perp }\rVert^2
   +\bigg(\max\limits_{i\in[m]}\lVert w_i^{\top} X_i\rVert_{\infty}\frac{2}{N}+\frac{2\lambda}{m}\bigg)\lVert\hat{\bnu}_{\perp }\rVert_1.
\end{align}
 {\bf $\bullet$ Step 2: Establishing the lower bound of $({1}/{N})\sum\limits_{i=1}^m\lVert X_i\hat{\nu}_i\rVert^2.$} 

Invoking the decomposition $\hat{\nu}_i=\hat{\nu}_{\text{av}}+\hat{\nu}_{\perp i}$, $i\in [m]$, along with the Young's inequality, we can write 
\begin{align}\label{transfer to parametric error}
    \frac{2}{N}\sum\limits_{i=1}^m\lVert X_i(\hat{\nu}_{\text{av}}+\hat{\nu}_{\perp\,i})\rVert^2
    &\ \geq 
   \frac{1}{N}\lVert \mathbf{X}\hat{\nu}_{\text{av}}\rVert^2-\frac{2}{N}\sum\limits_{i=1}^m\lVert X_i\hat{\nu}_{\perp i}\rVert^2\notag\\
    &\overset{\eqref{ARE eq}}{\geq}\frac{1}{2}[\delta\lVert\hat{\nu}_{\text{av}}\rVert^2-\xi h^2(\gamma,\lVert\hat{\bnu}_{\perp}\rVert)]+\frac{1}{2N}\lVert \mathbf{X}\hat{\nu}_{\text{av}}\rVert^2-\frac{2}{N}\sum\limits_{i=1}^m\lVert X_i\hat{\nu}_{\perp i}\rVert^2. 
\end{align}
{\bf $\bullet$ Step 3: Lower bound $\leq$ Upper bound.}  
 
 Chaining  (\ref{prediction error result}) and  (\ref{transfer to parametric error}) while adding $\frac{\delta}{2m}\lVert\hat{\bnu}_{\perp}\rVert^2$ on both sides, yield
 \begin{align}\label{average ell 2 error}
    &\frac{1}{2}\delta\lVert\hat{\nu}_{\text{av}}\rVert^2+\frac{\delta}{2m}\lVert\hat{\bnu}_{\perp}\rVert^2\notag\\
    \leq&\frac{9\lambda^2s}{2\delta}+\xi h^2(\gamma,\lVert\hat{\bnu}_{\perp}\rVert)+\bigg(\frac{2L_{\max}}{m}+\frac{\delta}{2m}-\frac{1-\rho}{m\gamma}\bigg)\lVert\hat{\bnu}_{\perp }\rVert^2+\bigg(\max\limits_{i\in[m]}\lVert w_i^{\top} X_i\rVert_{\infty}\frac{2}{N}+\frac{2\lambda}{m}\bigg)\lVert\hat{\bnu}_{\perp }\rVert_1\notag\\
   {\leq}&\frac{9\lambda^2s}{2\delta}+\xi h_{\max}^2+ \underbrace{\bigg(\frac{2L_{\max}}{m}+\frac{\delta}{2m}-\frac{1-\rho}{m\gamma}\bigg)\lVert\hat{\bnu}_{\perp }\rVert^2+\bigg(\max\limits_{i\in[m]}\lVert w_i^{\top} X_i\rVert_{\infty}\frac{2}{N}+\frac{2\lambda}{m}\bigg) \sqrt{md}\lVert\hat{\bnu}_{\perp }\rVert}_{\triangleq h_1(\gamma,\lVert\hat{\bnu}_{\perp}\rVert)},
\end{align}
where in the last inequality we used  $L_{\max}=\max\limits_{i\in[m]}\lambda_{\max}(X_i^{\top} X_i/n)$ [cf.~\eqref{L max}], and the following upper bound  for $h(\gamma,\lVert\hat{\bnu}_{\perp}\rVert)$ 
\begin{align}\label{def:h-max}
  {h(\gamma,\lVert\hat{\bnu}_{\perp}\rVert)\leq 
  h_{\max}\triangleq\frac{ d\gamma}{\lambda(1-\rho)}\bigg(\frac{\max_{i\in [m]}\lVert w_i^{\top} X_i\rVert_{\infty}}{n}+\lambda\bigg)^2.}
\end{align}
Under  the condition on $\gamma$ as in \eqref{eq:gamma_cond},
 $h_1(\gamma,\lVert\hat{\bnu}_{\perp}\rVert)$ is a quadratic function of $\|\hat{\bnu}_{\perp}\|$ opening downward, and it can be upper bounded over  $\mathbb{R}^+$ as
 \begin{equation}\label{eq:h_1_max}
     h_{1\max}\triangleq\frac{2d\gamma(\max_{i\in [m]}\lVert w_i^{\top} X_i\rVert_{\infty}/n+\lambda)^2}{2(1-\rho)-4L_{\max}\gamma-\delta\gamma}.
 \end{equation}
Using \eqref{eq:h_1_max} in \eqref{average ell 2 error},
  we finally obtain \eqref{oracle}.
 $\hfill \square$

\section{Proof of Theorem \ref{statistical optimization error result}}\label{Proof of Theorem 3}

The proof builds on the following four steps: {\bf 1)} We first consider as   source of  randomness  only the design matrix $\mathbf{X}$ (cf.~Assumption~\ref{ass_X_random}) while keeping $\mathbf{w}$ fixed, deriving a high-probability  bound for  ${L_{\max}}$   in \eqref{L max};  
 {\bf 2)} We then fix  $\mathbf{X}$ and  consider the randomness coming from  the noise $\mathbf{w},$ providing high-probability bounds for  the noise-dependent   terms $\lVert \mathbf{X}^{\top} \mathbf{w}\rVert_{\infty}/N$ and $\max_{1\leq i\leq m}\lVert X_i^{\top} w_i\rVert_{\infty}/n$; {\bf 3)} We then combine the previous two results via the  union bound and establish  a lower bound on $\lambda$ for \eqref{distributed lambda} to hold with high probability; {\bf 4)} Finally,  we use the bound in {\bf 3)}  to 
obtain   the final error bound on the $\ell_2$-LASSO error. 

Let $\mathbb{P}$ be a probability measure on the product sample space $\mathbb R^{N \times d} \otimes \mathbb R^N$. For brevity, we use the same notation for the marginal distributions on $\mathbb R^{N\times d}$ and $\mathbb{R}^N$.
 
\noindent {\bf $\bullet$ Step 1: Randomness from $\mathbf{X}.$ }
 
 We define three ``good'' events so that the largest eigenvalue of $(1/n) X_i^\top X_i$, smallest eigenvalue of $(1/N) \mathbf{X}^\top \mathbf{X}$ and the norm of the columns of $\mathbf{X}$ are well-controlled. We prove next that  these events jointly occur with high probability. Specifically, let
 \begin{equation}
    A_1\triangleq
    \bigg\{\mathbf{X}\in\mathbb{R}^{N\times d}~\bigg|\text{ } L_{\max}\leq \tilde{c}_4 \lambda_{\max}(\Sigma) \left(1+  \frac{d + \log m}{n}  \right)\bigg\},\label{eq:def_A1}
\end{equation} 
\begin{equation}\label{eq:def_A2_A3}
  A_2\triangleq
    \bigg\{\mathbf{X}\in\mathbb{R}^{N\times d}~\bigg|\text{ } \mathbf{X}\text{ } \text{satisfies}\text{ } \eqref{network slack} \bigg\},\ \ \text{and}\ \    A_3\triangleq\bigg\{\mathbf{X}\in\mathbb{R}^{N\times d}~\bigg|\text{ } \max_{j=1,\dots,d}\frac{1}{\sqrt{N}}\lVert \mathbf{X}e_j\rVert\leq \sqrt{\frac{3\zeta_{\Sigma}}{2}}\bigg\},
\end{equation}
where $\tilde{c}_4 > 0$ is a universal constant (see~\eqref{new bound l-max}), and we recall from \eqref{L max} and \eqref{zeta} that $ L_{\max}\triangleq\max_{  i\in [m]}\lambda_{\max}(X_i^{\top} X_i/n),$ and $\zeta_{\Sigma}\triangleq\max\limits_{i\in [d]}\Sigma_{ii}$, respectively.
We proceed to bounding $\mathbb{P}(A_1)$, $\mathbb{P}(A_2)$, and $\mathbb{P}(A_3)$. \smallskip 

  \textbf{(i)  Bounding $\mathbb{P}(A_1)$:}
Recall that  $\mathbf{X}=[X_1^\top,\ldots, X^\top_m]^\top,$ and $\mathbf{X}$ satisfies Assumption \ref{Random Gaussian model}. Thus, $\{X_i\}_{i\in[m]}$ are i.i.d random matrices, with i.i.d. rows drawn from   $\mathcal{N}(0,\Sigma).$  By  {\citep[Remark 5.40]{Vershynin2012IntroductionTT}} it follows that the following holds with probability at least 
\begin{equation}\label{c_3}
    1-2\exp\{-\tilde{c}_3t^2\},
\end{equation}
for  all $t\geq 0$
\begin{equation}\label{X random event 1}
    \bigg\|\frac{1}{n}X_i^{\top}X_i-\Sigma\bigg\|\leq\max\{a,a^2\}\|\Sigma\|,\ \text{where}\ a\triangleq \tilde{c}_2\bigg(\sqrt{\frac{d}{n}}+\frac{t}{\sqrt{n}}\bigg),
\end{equation}
with constants $\tilde{c}_3$ and  $\tilde{c}_2>0$.
 Given~(\ref{X random event 1}) and using the triangle inequality, we have
\begin{align}\label{cov-bound-intermediate}
   \Big\| \frac{1}{n} X_i^\top X_i \Big\|  
    \leq \Big\| \frac{1}{n} X_i^\top X_i -\Sigma\Big\| +\Big\|\Sigma\Big\| 
    \leq 
   \lambda_{\max} (\Sigma)  \max\{a,a^2\}+\lambda_{\max} (\Sigma).
\end{align}
Applying the union bound we  obtain the following  bound for  $L_{\max}$
 \begin{align}
     &\mathbb{P}\left(L_{\max}\leq \lambda_{\max}(\Sigma) (1+ \max\{a,a^2\}) \right) \geq 1 - m \cdot 2\exp\{-\tilde{c}_3t^2\}.
 \end{align}
Setting $t = \sqrt{d + \tilde{c}_3^{-1} \log m}$, yields
 \begin{equation*}
     a =  \tilde{c}_2 \left( \sqrt{\frac{d}{n}}+ \sqrt{\frac{d + \tilde{c}_3^{-1} \log m}{n}}\right).
 \end{equation*}
Therefore,  we conclude
 \begin{align}\label{new bound l-max}
 \begin{split}
         L_{\max} &\leq \lambda_{\max}(\Sigma) \left(1+ a + a^2\right) 
         \leq 2 \lambda_{\max}(\Sigma)  \left(1+ a^2\right) \leq 
    \tilde{c}_4 \lambda_{\max}(\Sigma) \left(1+  \frac{d + \log m}{n}  \right),
 \end{split}
 \end{align}
 with probability at least $1 - 2 \exp(-\tilde{c}_3 d)$ and 
 \begin{equation}\label{c_4}
     \tilde{c}_4 = \max\{2, 4 \tilde{c}_2^2 , 4 \tilde{c}_2^2 \tilde{c}_3^{-1}\}\geq2.
 \end{equation}

\smallskip

 \textbf{(ii)  Bounding $\mathbb{P}(A_2)$:}
It follows readily  from Lemma~\ref{Global ARE}: if $ N\geq\frac{128s \tilde{c}_1\zeta_{\Sigma}\log d}{\lambda_{\min}(\Sigma)}$  and $\gamma>0$,
\begin{equation}\label{eq-A2 prop}
    \mathbb{P}(A_2^c)\leq \exp(-\tilde{c}_0N).
\end{equation}

\smallskip

  \textbf{(iii)  Bounding $\mathbb{P}(A_3)$:} 
Recall Assumption \ref{Random Gaussian model}. 
It follows that $\mathbf{X}e_j$ is an isotropic Gaussian random vector in $\mathbb{R}^N$ with $\mathcal{N}(0,\Sigma_{jj})$ entries.
Hence, $\|\mathbf{X}e_j\|^2/\Sigma_{jj}$ is a chi-squared random variable with degree $N.$ Then, applying the standard bound for chi-squared random variables {\citep[Example 2.11]{Wainwright-book}} we have
\begin{align}\label{chi square 1}
    \mathbb{P}\bigg(\bigg|\frac{1}{N} \bigg\|\frac{\mathbf{X}e_j}{\sqrt{\Sigma_{jj}}}\bigg\|^2-1\bigg|\geq t\bigg)\leq2\exp(-Nt^2/8),\qquad \text{for all} \ t\in(0,1).
\end{align}
Taking $t=\frac{1}{2}$ in \eqref{chi square 1} and applying the union bound, we obtain
\begin{align}\label{chi square 3}
\begin{split}
    \mathbb{P}\bigg(\max\limits_{j\in[d]}\frac{1}{N} \bigg\|\frac{\mathbf{X}e_j}{\sqrt{\Sigma_{jj}}}\bigg\|^2\geq \frac{3}{2}\bigg)\leq  d\ \mathbb{P}\bigg(\bigg|\frac{1}{N} \bigg\|\frac{\mathbf{X}e_j}{\sqrt{\Sigma_{jj}}}\bigg\|^2-1\bigg|
    \geq \frac{1}{2}\bigg) 
    \leq
    2\exp(-N/32+\log d).
\end{split}
\end{align}
Therefore, for all  $N\geq \tilde{c}_5\log d$,  with    $\tilde{c}_5> 32,$ we have
\begin{align}\label{column norm normalize}
\begin{split}
     \mathbb{P}\bigg(\max\limits_{j\in[d]}\frac{\lVert \mathbf{X}e_j\rVert^2}{N}\leq \frac{3}{2}\zeta_{\Sigma}\bigg)\geq&1-2\exp[-(\tilde{c}_5/32)\log d+\log d]\\
     =&1-2\exp(-\tilde{c}_6\log d),\  \text{where}\ \tilde{c}_6=\tilde{c}_5/32-1>0.
\end{split}
\end{align}
 Combining the conditions on $N,$ we have
\begin{equation}\label{sample_stat_final}
    N\geq \frac{ \tilde{c}_9s\zeta_{\Sigma}\log d}{\lambda_{\min}(\Sigma)}\overset{\text{(a)}}{\geq} \max\bigg\{\frac{128s \tilde{c}_1\zeta_{\Sigma}\log d}{\lambda_{\min}(\Sigma)},\tilde{c}_5\log d\bigg\},
\end{equation}
where $\tilde{c}_9=\max\{128\tilde{c}_1,\tilde{c}_5\},$ and in (a) we used $s\geq1,\zeta_{\Sigma}\geq\lambda_{\min}(\Sigma).$

Finally, we combine \eqref{new bound l-max}, \eqref{eq-A2 prop}, and \eqref{column norm normalize}; using the union bound again, we have
\begin{equation*}
    \mathbb{P}(A_1^c\cup A_2^c\cup A_3^c)\leq \mathbb{P}(A_1^c)+\mathbb{P}(A_2^c)+\mathbb{P}(A_3^c)\leq 2\exp(-\tilde{c}_3d)+\exp(-\tilde{c}_0N)+2\exp(-\tilde{c}_6\log d).
\end{equation*}
Define $A\triangleq A_1\cap A_2\cap A_3,$
\begin{equation}\label{eq:p:X}
\mathbb{P}(A)\geq 1- 2\exp(-\tilde{c}_3d)-\exp(-\tilde{c}_0N)-2\exp(-\tilde{c}_6\log d).
\end{equation}

\noindent {\bf $\bullet$ Step 2: Randomness from $\mathbf{w}.$}
We start with bounding $ \|\mathbf{X}^{\top} \mathbf{w}\|_{\infty}.$
For fixed  $\mathbf X\in A,$ and $\mathbf{w}\sim \mathcal{N}(0,\sigma^2I_{N}),$ recall   $\mathbf{X}=[X_1^\top,\ldots, X^\top_m]^\top,$ and $\mathbf{w}=[w_1^\top,\ldots, w^\top_m]^\top,$  where for each agent $i\in[m],$ $X_i \in \mathbb{R}^{n\times d}$ is the design matrix, $w_i \in \real^n$ is observation noise. Then, for any  $i\in[m]$ and $j\in[d],$  
\begin{equation}\label{normal distribution regularized}
    \frac{\mathbf{w}^{\top} \mathbf{X}e_j}{N}\ \bigg|_{\mathbf{X}\in A}\sim\mathcal{N}\bigg(0,\frac{\sigma^2}{N}\cdot\frac{\|\mathbf{X}e_j\|^2}{N}\bigg)\quad\text{and}\quad\frac{w_i^{\top} X_i e_j}{N}\ \bigg|_{\mathbf{X}\in A}\sim\mathcal{N}\bigg(0,\frac{\sigma^2}{N}\cdot\frac{\|X_ie_j\|^2}{N}\bigg).
\end{equation}
Note that 
\begin{equation}\label{local matrix column normalize}
   \max\limits_{i\in[m]}\max\limits_{j\in[d]}\frac{\|X_ie_j\|^2}{N}\leq \max\limits_{j\in[d]}\frac{1}{m}\sum\limits_{i=1}^m \frac{\|X_ie_j\|^2}{n}=\max\limits_{j\in[d]}\frac{\lVert \mathbf{X}e_j\rVert^2}{N},
\end{equation}
due to  $\frac{1}{m}\sum_{i=1}^m \frac{\|X_ie_j\|^2}{n}=\frac{\lVert \mathbf{X}e_j\rVert^2}{N}$. 

By definition, for all $\mathbf{X}\in A\subseteq A_3,$ $2\|\mathbf{X}e_j\|^2/(3\zeta_{\Sigma}N)\leq 1$ and, by \eqref{local matrix column normalize}, $2\|X_ie_j\|^2/(3\zeta_{\Sigma}N)\leq 1.$
Therefore, combining it   with \eqref{normal distribution regularized}, we obtain 
\begin{align}\label{normalized vector}
\begin{split}
    &\sqrt{\frac{2}{3\zeta_{\Sigma}}}\frac{\mathbf{w}^{\top} \mathbf{X}e_j}{N}\ \bigg|_{\mathbf{X}\in A} \sim\mathcal{N}\bigg(0,\frac{\sigma^2}{N}\cdot\frac{2\|\mathbf{X}e_j\|^2}{3\zeta_{\Sigma}N}\bigg),\qquad\,\,\, \text{where}\qquad \frac{2\|\mathbf{X}e_j\|^2}{3\zeta_{\Sigma}N}\leq 1;\ \ \text{and}\\ &\sqrt{\frac{2}{3\zeta_{\Sigma}}}\frac{w_i^{\top} X_i e_j}{N}\ \bigg|_{\mathbf{X}\in A}\sim\mathcal{N}\bigg(0,\frac{\sigma^2}{N}\cdot\frac{2\|X_ie_j\|^2}{3\zeta_{\Sigma}N}\bigg),\qquad  \text{where}\qquad  \frac{2\|X_ie_j\|^2}{3\zeta_{\Sigma}N}\leq 1.
\end{split}
\end{align}
Denote $p_{\mathbf{X}}(x)$ and $p_{X_i}(x_i)$ as the density of $\sqrt{2/(3\zeta_{\Sigma})}\mathbf{w}^{\top} \mathbf{X}e_j/N$ and   $\sqrt{2/3\zeta_{\Sigma}}w_i^{\top} X_i e_j/N,$ respectively.
Let $Z\sim\mathcal{N}(0,\sigma^2/N),$ with density function  $p_{Z}(z)$. Since $2\|\mathbf{X}e_j\|^2/(3\zeta_{\Sigma}N)\leq 1$ and $2\|X_ie_j\|^2/(3\zeta_{\Sigma}N)\leq 1,$  we conclude $p_{\mathbf{X}}(0)\geq p_{Z}(0)$ and $p_{X_i}(0)\geq p_{Z}(0).$ \citep[Theorem 1]{10.1214/aos/1176350964} implies 
\begin{equation*}
    \bigg|\sqrt{\frac{2}{3\zeta_{\Sigma}}}\frac{\mathbf{w}^{\top} \mathbf{X}e_j}{N}\ 
    \bigg|_{\mathbf{X}\in A}\bigg|\preceq^{\text{st}}|Z|, \ \ \text{as well as}\ \ \bigg|\sqrt{\frac{2}{3\zeta_{\Sigma}}}\frac{w_i^{\top} X_i e_j}{N}\ \bigg\lvert_{\mathbf{X}\in A}\bigg|\preceq^{\text{st}}|Z|.
\end{equation*}
Therefore, 
\begin{align}\label{normalized gaussian bound}
     \mathbb{P}\bigg(\frac{\lvert \mathbf{w}^{\top} \mathbf{X}e_j\rvert}{N}\geq x \sqrt{\frac{3\zeta_{\Sigma}}{2}}\ \bigg|\  \mathbf{X}\in A \bigg) 
     \leq \mathbb{P}(|Z|\geq x) 
    \leq  2\exp\bigg(-\frac{Nx^2}{2\sigma^2}\bigg). 
\end{align} 
Notice that  $\lVert \mathbf{X}^{\top} \mathbf{w}\rVert_{\infty}/N=\max_{j\in [d]} \lvert \mathbf{w}^{\top} \mathbf{X}e_j\rvert/N$. Hence, setting $x=\sigma\sqrt{t_0\log d/N},$ with $t_0>2,$ the union bound implies
\begin{align}\label{global L}
    &\mathbb{P}\bigg(\frac{\lVert \mathbf{X}^{\top} \mathbf{w}\rVert_{\infty}}{N}\geq \sigma\sqrt{\frac{t_0\log d}{N}}\sqrt{\frac{3\zeta_{\Sigma}}{2}}\ \bigg|\ \mathbf{X}\in A\bigg) 
    \leq  
    2\exp\left(-\frac{1}{2}(t_0-2)\log d\right).
\end{align}
Define
\begin{align}
    D_1\triangleq \bigg\{\mathbf{w}\in\mathbb{R}^N\text{ }\bigg|\text{ }\frac{\lVert \mathbf{X}^{\top} \mathbf{w}\rVert_{\infty}}{N}\leq \sigma\sqrt{\frac{t_0\log d}{N}}\sqrt{\frac{3\zeta_{\Sigma}}{2}}\bigg\}.
\end{align}
We have   $\mathbb{P}(D_1\ |\ \mathbf{X}\in A)\geq 1-2\exp\left(-\frac{1}{2}(t_0-2)\log d\right).$
Combining it with \eqref{eq:p:X}, yields
  {\begin{align}
\begin{split}
    \mathbb{P}(A\cap D_1)=&\mathbb{P}(D_1\ |\ A)\mathbb{P}(A)\\
    \geq&[1-2\exp\{-[(t_0-2)\log d]/2\}][1- 2\exp(-\tilde{c}_3d)-\exp(-\tilde{c}_0N)-2\exp(-\tilde{c}_6\log d)]\notag\\
    \geq&1- 2\exp(-\tilde{c}_3d)-\exp(-\tilde{c}_0N)-2\exp(-\tilde{c}_6\log d)-2\exp\{-[(t_0-2)\log d]/2.
\end{split}
\end{align}}
  {It remains to bound $\max_{i\in [m]}\lVert X_i^{\top} w_i\rVert_{\infty}.$ Since  $X_i,$ $i\in[m]$, are independent, the  columns of $X_i$ are $n$ dimensional i.i.d Gaussian random vectors, each element has variance at most $\zeta_{\Sigma}$,  and the elements of $w_i\sim \mathcal{N}(0,\sigma^2I_{n}).$  Then each element of $X_i^{\top}w_i$ is the sum of
$n$ independent sub-exponential random variables with  sub-exponential norm  at most
$\sigma\sqrt{\zeta_\Sigma}.$ 
Applying   \cite[Proposition 5.16]{Vershynin2012IntroductionTT} and the union bound, we obtain 
\begin{align*}
    \mathbb{P}\left(\frac{1}{n}\max_{i\in [m]}\lVert X_i^{\top} w_i\rVert_{\infty}\leq t\right)\geq 1- 2\exp\left(-\tilde{c}_{24}\min\left\{\frac{t^2}{\sigma^2\zeta_\Sigma},\frac{t}{\sigma\sqrt{\zeta_\Sigma}}\right\}n+\log md\right),\quad t\geq 0,
\end{align*}
for some $\tilde{c}_{24}>0.$
Thus, under $2{{\log md}}\leq \tilde{c}_{24}n$ and $t= {\sigma  }\sqrt{\frac{{2\zeta_\Sigma\log md}}{n\tilde{c}_{24}}},$  
\begin{align}\label{eq:sub_exp_noise_1}
    &\,\,\,\,\,\mathbb{P}\left(\frac{1}{n}\max_{i\in [m]}\lVert X_i^{\top} w_i\rVert_{\infty}\leq {\sigma  }\sqrt{\frac{{2\zeta_\Sigma\log md}}{n\tilde{c}_{24}}}\right)\notag\\
    &\,\,\,\,\,\geq 1-2\exp\left(-\tilde{c}_{24}\min\left\{\frac{2\sigma^2  \zeta_\Sigma{\log md}}{\tilde{c}_{24}n\sigma^2\zeta_\Sigma},\frac{\sigma  \sqrt{2\zeta_\Sigma\log md}}{ \sqrt{\tilde{c}_{24}n\zeta_\Sigma}\sigma}\right\}n+\log md\right)\notag\\
      &\,\,\,\,\,\geq 1-2\exp\left(-\log d\right),
\end{align}
while, under  $2{{\log md}}>\tilde{c}_{24}n$ and $t={\frac{2\sigma  \sqrt{\zeta_\Sigma}\log md}{n\tilde{c}_{24}}},$   it holds
\begin{align}\label{eq:sub_exp_noise_2}
    &\,\,\,\,\,\mathbb{P}\left(\frac{1}{n}\max_{i\in [m]}\lVert X_i^{\top} w_i\rVert_{\infty}\leq{\frac{2\sigma \sqrt{\zeta_\Sigma}\log md}{n\tilde{c}_{24}}}\right)\notag\\
    &\,\,\,\,\,\geq 1-2\exp\left(-\tilde{c}_{24}\min\left\{\frac{4\sigma^2 {\zeta_\Sigma}{\log^2 md}}{\tilde{c}_{24}^2n^2\sigma^2{\zeta_\Sigma}},\frac{2\sigma  \sqrt{\zeta_\Sigma}{\log md}}{ {\tilde{c}_{24}}n\sigma\sqrt{\zeta_\Sigma}}\right\}n+\log md\right)\notag\\
      &\,\,\,\,\,\geq 1-2\exp\left(-\log d\right).
\end{align}
Combining \eqref{eq:sub_exp_noise_1} and \eqref{eq:sub_exp_noise_2}, we have
{
\begin{equation}\label{eq:sub_e_p2}
    \mathbb{P}\left(\frac{1}{n}\max_{i\in [m]}\lVert X_i^{\top} w_i\rVert_{\infty}\leq \sigma\sqrt{\zeta_\Sigma}\min\left\{{\frac{2\log md}{n\tilde{c}_{24}}},\sqrt{\frac{{2\log md}}{n\tilde{c}_{24}}}\right\}\right)\geq1-4\exp\left(-\log d\right).
\end{equation}
}
Define 
\begin{align*}
         D_2\triangleq \bigg\{\mathbf{w}\in\mathbb{R}^{N}\text{ }\bigg|\text{ }\frac{1}{n}\max_{i\in [m]}\lVert X_i^{\top} w_i\rVert_{\infty}\leq \sigma \sqrt{\zeta_\Sigma}\min\left\{{\frac{2\log md}{n\tilde{c}_{24}}},\sqrt{\frac{{2\log md}}{n\tilde{c}_{24}}}\right\}\bigg\}, \\
\end{align*} 
and $D \triangleq D_1 \cap D_2.$
Then, chaining \eqref{eq:p:X}, \eqref{eq:sub_e_p2}, and \eqref{global L}, we finally get
\begin{align}\label{eq:p-bound-23}
    &\mathbb{P}(A\cap D)\notag\\
   &\geq 1- 2\exp(-\tilde{c}_3d)-\exp\left(-\tilde{c}_0\frac{\tilde{c}_9s\zeta_{\Sigma}\log d}{\lambda_{\min}(\Sigma)}\right)-2\exp(-\tilde{c}_6\log d)-2\exp\{-[(t_0-2)\log d]/2\nonumber \\&\quad -4\exp(-\log d)\notag\\
    &\geq
   1- 11\exp(-\tilde{c}_{8}\log d),
\end{align}
where $\tilde{c}_{8}=\min\{1,\tilde{c}_3,\tilde{c}_6,(t_0-2)/2,\tilde{c}_0\tilde{c}_9\}.$}  
\smallskip  

\noindent {\bf $\bullet$ Step 3: Sufficient condition on $\lambda$ for \eqref{distributed lambda} to hold with high probability. }
 
We first recall \eqref{distributed lambda} for convenience.
\begin{equation*}
    \frac{2}{N}\lVert \mathbf{X}^{\top} \mathbf{w}\rVert_{\infty}\leq \lambda.
    \end{equation*}
Combining it with the high probability upper bound for $\|\mathbf{X}^{\top}\mathbf{w}\|_{\infty}/N$ in \eqref{global L} ({\bf Step 2}), we conclude that, 
if $\lambda$ satisfies
\begin{equation}\label{stat lambda}
    \lambda= \tilde{c}_8\sigma\sqrt{\frac{\zeta_{\Sigma}t_0\log d}{N}},
\end{equation}
with $\tilde{c}_8\geq \sqrt{6},$ then \eqref{distributed lambda} holds with probability at least \eqref{eq:p-bound-23}.
\smallskip 

\noindent {\bf $\bullet$ Step 4: Bounding the statistical error under \eqref{stat gamma condition}.}

Recall the deterministic error bounds in Theorem \ref{solution err bound}: for any fixed   $\lambda$ and $\gamma$ satisfying (\ref{eq:gamma_cond}),  
\begin{align*}
   &\frac{1}{m}\sum\limits_{i=1}^m\lVert\hat{\nu}_i\rVert^2\notag\\
   \overset{\text{\eqref{oracle}}}{\leq}&\frac{9\lambda^2s}{\delta^2}+\frac{2\xi d^2\gamma^2}{\delta\lambda^2(1-\rho)^2}\bigg(\frac{\max_{i\in [m]}\lVert w_i^{\top} X_i\rVert_{\infty}}{n}+\lambda\bigg)^4+\frac{4d\gamma(\max_{i\in [m]}\lVert w_i^{\top} X_i\rVert_{\infty}/n+\lambda)^2}{\delta[2(1-\rho)-4L_{\max}\gamma-\delta\gamma]}.\notag
\end{align*}
In {\bf Step 3}, we provided a sufficient condition on $\lambda$ to guarantee \eqref{distributed lambda} holds with probability at least as \eqref{eq:p-bound-23}. Now we proceed to provide a sufficient condition on $\gamma,$ not only to guarantee $\gamma\leq 2(1-\rho)/(4L_{\max}+\delta),$ but also contribute to restricting the  error term above within the centralized statistical error, which is of the order $\mathcal{O}(\lambda^2s).$

By Lemma \ref{Global ARE},   if  $ N\geq128s \tilde{c}_1\zeta_{\Sigma}\log d/\lambda_{\min}(\Sigma),$   with probability at least $1-\exp(-\tilde{c}_0N),$ the in-network RE condition holds with $\delta=\lambda_{\min}(\Sigma)/4$ and $\xi=\lambda_{\min}(\Sigma)/(64s).$ Combining this with the high probability upper bound derived on $L_{\max}$ in \eqref{new bound l-max} and the high probability upper bound derived for $\max_{i\in[m]}\lVert w_i^{\top} X_i\rVert_{\infty}/n$ in \eqref{eq:sub_e_p2},  we have
\begin{align*}
     \frac{1}{m}\sum\limits_{i=1}^m\lVert\hat{\nu}_i\rVert^2\leq&\frac{144\lambda^2s}{\lambda_{\min}(\Sigma)^2}+\underbrace{\frac{ d^2\gamma^2\bigg(  {\sigma\sqrt{\zeta_\Sigma}\min\left\{{\frac{2\log md}{n\tilde{c}_{24}}},\sqrt{\frac{{2\log md}}{n\tilde{c}_{24}}}\right\}}+\lambda\bigg)^4}{8\lambda^2s(1-\rho)^2}}_{\texttt{Term I}}\notag\\
    &+\underbrace{\frac{8d\gamma \bigg(  {\sigma\sqrt{\zeta_\Sigma}\min\left\{{\frac{2\log md}{n\tilde{c}_{24}}},\sqrt{\frac{{2\log md}}{n\tilde{c}_{24}}}\right\}}+\lambda\bigg)^2}{\lambda_{\min}(\Sigma)\left(1-\rho-4\gamma \tilde{c}_4 \lambda_{\max}(\Sigma) \left(1+  \frac{d + \log m}{n}  \right)-\frac{\gamma\lambda_{\min}(\Sigma)}{8}\right)}}_{\texttt{Term II}}
\end{align*}
 with probability larger than \eqref{eq:p-bound-23}. 
 
 It remains to prove that  condition \eqref{stat gamma condition} on  $\gamma$ is sufficient    for \texttt{Term I} and \texttt{Term II} to be within $\mathcal{O}(\lambda^2 s)$. Notice that 
\begin{equation}
    \texttt{Term I}\leq\texttt{Term II}^2\cdot\frac{\lambda_{\min}(\Sigma)^2}{512\lambda^2s}.
\end{equation} Thus it is sufficient to bound only \texttt{Term II}. Enforcing $\texttt{Term II}\leq \tilde{c}_7\lambda^2s/\lambda_{\min}(\Sigma)^2,$ where $\tilde{c}_7$ is a numerical constant, we derive the following sufficient condition on $\gamma$ to ensure $\texttt{Term II}\leq \tilde{c}_7\lambda^2s/\lambda_{\min}(\Sigma)^2$:  
\begin{equation}\label{stat gamma 1}
    \gamma\leq \frac{1-\rho}{8\lambda_{\min}(\Sigma)\frac{d}{\tilde{c}_7s}\left[  {\frac{\sigma\sqrt{\zeta_\Sigma}}{\lambda}\min\left\{{\frac{2\log md}{n\tilde{c}_{24}}},\sqrt{\frac{{2\log md}}{n\tilde{c}_{24}}}\right\}+1}\right]^2+ 4\tilde{c}_4 \lambda_{\max}(\Sigma)[1+  \frac{d + \log m}{n}  ]+\frac{\lambda_{\min}\Sigma}{8}}.
\end{equation}
Thus, under \eqref{stat gamma 1}, we have
\begin{equation*}
    \texttt{Term II}\leq \tilde{c}_7\frac{\lambda^2s}{\lambda_{\min}(\Sigma)^2}\quad \Rightarrow\quad \texttt{Term I}\leq\frac{\tilde{c}_7^2\lambda^4s^2}{\lambda_{\min}(\Sigma)^4}\frac{\lambda_{\min}(\Sigma)^2}{512\lambda^2s}=\frac{\tilde{c}_7^2\lambda^2s}{512\lambda_{\min}(\Sigma)^2}.
\end{equation*}
Therefore, the final statistical error satisfies  
\begin{align*}
     \frac{1}{m}\sum\limits_{i=1}^m\lVert\hat{\nu}_i\rVert^2&\leq
     \bigg(144+\tilde{c}_7+\frac{\tilde{c}_7^2}{512}\bigg)\frac{\tilde{c}_8^2\sigma^2\zeta_{\Sigma}t_0}{\lambda_{\min}(\Sigma)^2}\frac{s\log d}{N}= c_6\frac{\sigma^2\zeta_{\Sigma}t_0}{\lambda_{\min}(\Sigma)^2}\frac{s\log d}{N}, 
\end{align*}
where $c_6\triangleq\bigg(144+\tilde{c}_7+\frac{\tilde{c}_7^2}{512}\bigg)\tilde{c}_8^2.$
Since $\lambda$ satisfies \eqref{stat lambda}, we have
  {\begin{align}\label{eq:bound_I}
   \frac{\sigma\sqrt{\zeta_\Sigma}}{\lambda}\min\left\{{\frac{2\log md}{n\tilde{c}_{24}}},\sqrt{\frac{{2\log md}}{n\tilde{c}_{24}}}\right\} 
   \leq \frac{1}{\tilde{c}_8}\sqrt{\frac{2m\log md}{\tilde{c}_{24}t_0\log d}}.
\end{align}}
Substituting \eqref{eq:bound_I}  into \eqref{stat gamma 1}, we have the following sufficient condition for \eqref{stat gamma 1}
  {
\begin{equation*}
    \gamma\leq \frac{8n(1-\rho)\tilde{c}_7s}{32\tilde{c}_4\tilde{c}_7s \lambda_{\max}(\Sigma)[n+  (d + \log m)  ]+\lambda_{\min}(\Sigma )n\left\{64d\left[  {\frac{1}{\tilde{c}_8}\sqrt{\frac{2m\log md}{\tilde{c}_{24}t_0\log d}}+1}\right]^2+\tilde{c}_7s\right\}}.
\end{equation*}
Notice that
\begin{align*}
    &64d\left[  {\frac{1}{\tilde{c}_8}\sqrt{\frac{2m\log md}{\tilde{c}_{24}t_0\log d}}+1}\right]^2 
     \leq128d\left[{{\frac{m(\log m+1)}{3\tilde{c}_{24}}}+1}\right].\notag
\end{align*} 
In addition, 
\begin{align}
   & \frac{s\cdot n (1-\rho)}{4s \tilde{c}_4 \lambda_{\max}(\Sigma)[n+  (d + \log m)  ]+\lambda_{\min}(\Sigma )n\left[16d\left[{{{m(\log m+1)}/{(3\tilde{c}_{24}\tilde{c}_7)}}+1}\right]+s/8\right]}\notag\\
      &\geq\frac{ c_5 (1-\rho)}{\lambda_{\max}(\Sigma)(d + \log m) +\lambda_{\min}(\Sigma )d{m(\log m+1)}},
\end{align}
where $c_5\triangleq1/\max\{8\tilde{c}_4, 32/(\tilde{c}_{24}\tilde{c}_7)\}$
Hence,  \eqref{stat gamma condition} is sufficient for \eqref{stat gamma 1}.
}
$\hfill\square$

\section{Proof of Lemma~\ref{solution eq}}\label{sec:R-bound}
 Using (\ref{def:nu-prop1}), each $ \|\hat{\T}_i\|_1$ can be bounded as
\begin{align}\label{deterministic solution equivalence_I}
    \|\hat{\T}_i\|_1\leq  \|\hat{\T}_i-\T^*\|_1+\|\T^*\|_1 
    = \|\hat{\nu}_{\text{av}}+\hat{\nu}_{\perp i}\|_1+\|\T^*\|_1 
    \leq \|\hat{\nu}_{\text{av}}\|_1+\sqrt{d}\,\|\hat{\nu}_{\perp i}\|+\|\T^*\|_1.
    \end{align}
    We bound next $\|\hat{\nu}_{\text{av}}\|_1$.
By Proposition~\ref{proximal nu average sparsity}, 
any solution $\hat{\bT}$ of   (\ref{Change problem}) satisfies
\begin{equation*}
    \lVert (\hat{\nu}_{\text{av}})_{\mathcal{S}^c}\rVert_1\leq3\lVert (\hat{\nu}_{\text{av}})_{\mathcal{S}}\rVert_1+h(\gamma,\lVert\hat{\bnu}_{\perp}\rVert).
\end{equation*}
Therefore, 
\begin{align}\label{ell 1 norm bound}
     \lVert\hat{\nu}_{\text{av}}\rVert_1\leq&4\lVert (\hat{\nu}_{\text{av}})_{\mathcal{S}}\rVert_1+h(\gamma,\lVert\hat{\bnu}_{\perp}\rVert)\notag\\
     \leq &4\sqrt{s}\frac{\lVert \hat{\bnu}\rVert}{\sqrt{m}}+h(\gamma,\lVert\hat{\bnu}_{\perp}\rVert)\notag\\
     \overset{(a)}{\leq}& 4\sqrt{s}\sqrt{\frac{9\lambda^2s}{\delta^2}+\frac{2\tau d^2\gamma^2(\max_{  i\in [m]}\lVert w_i^{\top} X_i\rVert_{\infty}+\lambda n)^4}{\delta\lambda^2n^4(1-\rho)^2}+\frac{4d\gamma(\max_{  i\in [m]}\lVert w_i^{\top} X_i\rVert_{\infty}+\lambda n)^2}{\delta n^2[2(1-\rho)-4L_{\max}\gamma-\delta\gamma]}}\notag\\
     &+h(\gamma,\lVert\hat{\bnu}_{\perp}\rVert)\notag\\
      \leq & \frac{12\lambda s}{\delta}+\sqrt{\frac{32s\tau}{\delta}}\frac{ d\gamma(\max_{i\in [m]}\lVert w_i^{\top}X_i\rVert_{\infty}+\lambda n)^2}{\lambda n^2(1-\rho)}+\sqrt{\frac{64sd\gamma(\max_{  i\in [m]}\lVert w_i^{\top} X_i\rVert_{\infty}+\lambda n)^2}{\delta n^2[2(1-\rho)-4L_{\max}\gamma-\delta\gamma]}}\notag\\
     &+h(\gamma,\lVert\hat{\bnu}_{\perp}\rVert),
\end{align}
 where in (a) we used Theorem~\ref{solution err bound} and the fact that the RSC (\ref{Arsc}) implies the  in-network RE (\ref{ARE eq})   [cf.~Lemma~\ref{global-ARE-determin}], with $\xi=\tau$ and $\delta=\mu/2-16s\tau>0$.

Substituting (\ref{ell 1 norm bound}) in (\ref{deterministic solution equivalence_I}) yields
\begin{align}\label{deterministic solution equivalence}
  \|\hat{\T}_i\|_1
     \leq&\frac{12\lambda s}{\delta}+\sqrt{\frac{32s\tau}{\delta}}\frac{ d\gamma(\max_{  i\in [m]}\lVert w_i^{\top} X_i\rVert_{\infty}+\lambda n)^2}{\lambda n^2(1-\rho)}+\sqrt{\frac{64sd\gamma(\max_{  i\in [m]}\lVert w_i^{\top} X_i\rVert_{\infty}+\lambda n)^2}{\delta n^2[2(1-\rho)-4L_{\max}\gamma-\delta\gamma]}}\notag\\
     &+\underbrace{h(\gamma,\lVert\hat{\bnu}_{\perp}\rVert)+\sqrt{d}\|\hat{\bnu}_{\perp }\|}_{\triangleq h_2(\gamma,\|\hat{\bnu}_{\perp }\|)}+\|\T^*\|_1\notag\\
     \overset{\text{(a)}}{\leq}&\frac{12\lambda s}{\delta}+\sqrt{\frac{32s\tau}{\delta}}\frac{ d\gamma(\max_{  i\in [m]}\lVert w_i^{\top} X_i\rVert_{\infty}+\lambda n)^2}{\lambda n^2(1-\rho)}+\sqrt{\frac{64sd\gamma(\max_{  i\in [m]}\lVert w_i^{\top} X_i\rVert_{\infty}+\lambda n)^2}{\delta n^2[2(1-\rho)-4L_{\max}\gamma-\delta\gamma]}}\notag\\
     &+\frac{d\gamma  (2\max_{  i\in [m]}\lVert w_i^{\top} X_i\rVert_{\infty}+(2+\sqrt{m})\lambda n)^2}{4\lambda n^2(1-\rho)}+\|\T^*\|_1\notag\\
      \overset{\text{(b)}}{\leq}&\frac{12\lambda s}{\delta}+\sqrt{\frac{32s\tau}{\delta}}\underset{\texttt{Term I=\,$h_{\max}$}}{\underbrace{\frac{ d\gamma(\max_{  i\in [m]}\lVert w_i^{\top} X_i\rVert_{\infty}+\lambda n)^2}{\lambda n^2(1-\rho)}}}+\sqrt{\frac{64s}{\delta}}\underset{\texttt{Term II}}{\underbrace{\sqrt{\frac{d\gamma(\max_{  i\in [m]}\lVert w_i^{\top} X_i\rVert_{\infty}+\lambda n)^2}{  n^2[2(1-\rho)-4L_{\max}\gamma-\delta\gamma]}}}}\notag\\
     &+\underset{\texttt{Term III}}{\underbrace{\frac{d\gamma(\max_{  i\in [m]}\lVert w_i^{\top} X_i\rVert_{\infty}+2\sqrt{m}\lambda n)^2}{\lambda n^2(1-\rho)}}}+\|\T^*\|_1,
\end{align}
 where in (a) we bounded   $h_2(\gamma,\sbullet)$ on $\mathbb{R}$ as 
\begin{align*}
    h_2 (\gamma , \|\hat{\bnu}_\bot \|)  & \overset{(\ref{h})}{=}   
        -\frac{1-\rho}{\lambda m\gamma}\|\hat{\bnu}_{\perp}\|^2+\left(2\max_{  i\in [m]}\lVert w_i^{\top} X_i\rVert_{\infty}/(\lambda n)+2\right) \sqrt{d/m}\|\hat{\bnu}_{\perp}\| + \sqrt{d} \| \hat{\bnu}_\bot\|\\
    & \,\,\leq \frac{ d \gamma\left( 2\max_{  i\in [m]}\lVert w_i^{\top} X_i\rVert_{\infty}+ (2 + \sqrt{m}) \lambda n \right)^2}{4\lambda n^2(1-\rho)};
\end{align*}
and in (b) we enlarged $2+\sqrt{m}\leq 4\sqrt{m}.$

We bound now \texttt{Term I}--\texttt{Term III} using condition  \eqref{equ-sol-gamma} on $\gamma$. We have the following:
\begin{align}\label{refined bound}
   &  h_{\max}= \texttt{Term I}\leq
     \frac{\lambda s}{ {128}\,\delta},\medskip \\\notag     
 &\texttt{Term II}\leq \sqrt{\frac{d\gamma(\max_{  i\in [m]}\lVert w_i^{\top} X_i\rVert_{\infty}+2\sqrt{m}\lambda n)^2}{ n^2[2(1-\rho)-4L_{\max}\gamma-\delta\gamma]}}\leq \sqrt{\frac{\lambda^2s}{{256}\,\delta} },\medskip\\\notag
& \texttt{Term III}\leq \frac{\lambda\, s}{{128}\,\delta}.
 \end{align}
 Substituting the above bounds in  (\ref{deterministic solution equivalence}) we obtain
\begin{align*}
    \|\hat{\T}_i\|_1\leq&\frac{12\lambda s}{\delta}+\sqrt{\frac{32\tau s}{\delta}}\frac{\lambda s}{ {128}\delta}+\frac{\lambda s}{ {2}\delta}+\frac{\lambda s}{ {128}\delta}+\|\T^*\|_1\notag\\
    \leq& {\frac{\lambda s}{\delta}\bigg(13+\frac{1}{32}\sqrt{\frac{2\tau s}{\delta}}\bigg)}+\|\T^*\|_1\notag\\
    \overset{\eqref{c}}{\leq}& (1-r)\cdot  R + r\cdot R=R.
   \notag
\end{align*}
This completes the proof.\hfill $\square$.

\section{Proof of Lemma~\ref{lm-closed _form}}\label{proof_lm_closed_form}
Since $L_{\gamma}(\bT) \triangleq \frac{1}{m} \sum_{i = 1}^m f_i(\T_i)+ \frac{1}{2 m \gamma} \|\bT\|_V^2$, we have
\begin{align*}
    \nabla L_{\gamma} (\bT^t) = \frac{1}{m} 
    \begin{bmatrix}
\nabla f_1 (\T_1^t)\\
\vdots\\
\nabla f_m (\T_m^t)
\end{bmatrix}  +\frac{1}{m\gamma}\left((I-W)\otimes I_d\right) \bT^{t}.
\end{align*}
Substituting the expression of $\nabla L_{\gamma}(\bT^t)$ into Problem~\eqref{regularized algorith}, it is not hard to see it is separable in the $\T_i$'s, and the update of $\T_i$ given as
\begin{align*}
    \T_i^{t+1} & = 
    \argmin_{\|\T_i\|_1 \leq R}~ \frac{1}{2 } \left\| \T_i - \T_i^t +   {\beta}\nabla f_i(\T_i^t) +  {\frac{\beta}{\gamma}} \left( \T_i^t - \sum_{j = 1}^m w_{ij} \T_j^t \right)\right\|^2 +   {\beta \lambda}\|\T_i\|_1.
\end{align*}
The problem boils down to solving 
\begin{equation}~\label{p:prox-constr-l1}
  \begin{aligned}
    &\min_{\T_i} && \|\T_i - \psi_i^t\|^2 + \lambda' \| \T_i\|_1\\
    & {\rm \ s.t.} && \|\T_i\|_1 \leq R,
\end{aligned}  
\end{equation}
with  $\lambda' \triangleq   {{2 \beta \lambda}}$ and $\psi_i^t$   defined in \eqref{composite dgd variant}.

To solve~\eqref{p:prox-constr-l1} we first drop the constraint $\|\T_i\|_1 \leq R$. The minimizer of the objective function is given by
\begin{align}\label{prox-psi}
   \tilde{\T}_i = {\rm prox}_{\frac{\lambda'}{2} \|\cdot\|_1} (\psi_i^t). 
\end{align}
Note that $\tilde{\T}_i$ can be computed in closed form  by soft-thresholding $\psi_i^t$. 

{\bf Case 1:} 
$\tilde{\T}_i$ satisfies the constraint in \eqref{p:prox-constr-l1}, \ie, $\|\tilde{\T}_i\|_1 \leq R$. We conclude that  $\tilde{\T}_i$ is a solution of~\eqref{p:prox-constr-l1}. 

{\bf Case 2:} $\tilde{\T}_i$ violates the constraint in \eqref{p:prox-constr-l1}, \ie, $\|\tilde{\T}_i\|_1 > R$.  Then, the constraint must be active at the optimal point of \eqref{p:prox-constr-l1}. Hence, 
Problem~\eqref{p:prox-constr-l1} is   equivalent to 
\begin{equation}\label{p:proj-ncvx}
  \begin{aligned}
    &\min_{\T_i} && \|\T_i - \psi_i^t\|^2 \\
    & {\rm \ s.t.} && \|\T_i\|_1 = R,
\end{aligned}  
\end{equation}
where we dropped the term $\lambda' \|\T_i\|_1$ in the objective function, since it is constant on the constraint set.
Since~\eqref{prox-psi} can be computed in closed form by soft-thresholding $\psi_i^t$, we conclude $\| \psi_i^t\|_1 \geq \|\tilde{\T}_i\|_1 > R$, and thus the convex problem with constraint~\eqref{p:proj-ncvx} is equivalent to 
\begin{equation}\label{p:proj-cvx}
  \begin{aligned}
    &\min_{\T_i} && \|\T_i - \psi_i^t\|^2 \\
    & {\rm \ s.t.} && \|\T_i\|_1 \leq R.
\end{aligned}  
\end{equation}
Combining the two cases completes the proof. $\hfill\square$

\section{Proof of Theorem \ref{deterministic optimization error result}}\label{Guarantees for distributed regularized Lasso} 
Recall the factorization of the objective function by $G$ and $L$ as introduced in \eqref{eq:def_G}$$
    G(\bT)=L_\gamma(\bT)+\frac{\lambda}{m}\lVert\bT \rVert_1,\quad \text{with}\quad L_\gamma(\bT)= \frac{1}{2N}\sum\limits_{i=1}^m  {\lVert y_i-X_i\theta_i\rVert^2} + \frac{1}{2 m \gamma} \| \bT\|_{V}^2.$$ 
  We begin ({\bf Step 1}) proving  a weaker result than  Theorem \ref{deterministic optimization error result}, that is,  linear convergence of the error $G(\bT^t)-G(\hat{\bT})$, up to the tolerance as on the RHS of \eqref{final bound lagrangian}--this is Theorem~\ref{Composite function decrease} below. Then  ({\bf Step 2}), leveraging the curvature property of $G$ along the trajectory of the algorithm (see Lemma~\ref{Sufficient decrease} in Appendix~\ref{sec:intermediate_lemmata}), we      transfer the rate decay of  $G(\bT^t)-G(\hat{\bT})$ on that of the  iterates error  $\|\bT^t-\hat{\bT}\|$, which completes the proof of Theorem \ref{deterministic optimization error result}. \smallskip 
  
\noindent  {\bf $\bullet$ Step 1: On linear convergence of  the optimality gap  $G(\bT^t)-G(\hat{\bT})$.}  
   
   Recall the definition of $\varepsilon_{\textnormal{stat}}^2$ and $\mu_{\textnormal{av}}$ as given in \eqref{eq-error-stat} and \eqref{uav deterministic}, respectively. 

\begin{theorem}\label{Composite function decrease}   Instate the setting of Theorem~\ref{deterministic optimization error result}.
There holds: 
\begin{equation}\label{objective decrease}
    G(\bT^t)-G(\hat{\bT})\leq \alpha^2, 
\end{equation}
for any tolerance parameter $\alpha^2$ such that 
  \begin{equation}\label{determinstic alpha_appendix}
   {\min\bigg\{\frac{R\lambda}{4},\eta_G^0\bigg\}}\geq\alpha^2\geq 4s\tau \cdot\varepsilon_{\textnormal{stat}}^2,
\end{equation}
  and for all  
\begin{align}\label{eq:rate_estimate_appendix}
   t\geq  \bigg\lceil\log_2\log_2\left(\frac{R\lambda}{\alpha^2}\right)\bigg\rceil\bigg(1+\frac{L_{\max}\log 2}{\mu_{\textnormal{av}}}+\frac{(1+\rho)\log 2}{\gamma\mu_{\textnormal{av}}} \bigg) +  \left(\frac{ L_{\max}}{\mu_{\textnormal{av}}} + \frac{1+\rho}{\gamma\mu_{\textnormal{av}}}\right) \log\left(\frac{\eta_G^0}{ \alpha^2}\right).
\end{align}
   Furthermore, the interval in (\ref{determinstic alpha_appendix}) is nonempty. 
\end{theorem}
\begin{proof}
See Appendix~\ref{sec:proof_L15}.
\end{proof}
 
\noindent {\bf $\bullet$ Step 2: On linear convergence of the optimality gap $\|\bT^t-\hat{\bT}\|$.} 
 
 We can now proceed  to prove Theorem~\ref{deterministic optimization error result}. Given   Theorem~\ref{Composite function decrease}, it is sufficient to show that \eqref{final bound lagrangian} holds.

Recall the shorthand for the optimization error, $\bDelta^t=\bT^t-\hat{\bT}$. At high-level the idea is to construct a lower bound of $G(\bT^t)-G(\hat{\bT})$ as a function of  $\|\bDelta^t\|^2$ by exploiting, under the  RSC condition~(\ref{Arsc}), the curvature property of $G$   along a restricted set of directions. Specifically, we use the following curvature property proved in Lemma~\ref{Sufficient decrease} (cf. Sec.~\ref{sec:intermediate_lemmata}),  which holds   under the more stringent setting of  Theorem~\ref{Composite function decrease}\footnote{Specifically,  in the proof of Theorem~\ref{Composite function decrease}, we showed that   condition on $R$ as in (\ref{R condidion}) is more stringent than  \eqref{c} in Lemma~\ref{Sufficient decrease}--see Fact~1 in Appendix~\ref{step 2: Recursive application of Lemma step 2: Recursive application of Lemma}.}: for all $t\geq T$, 
 \begin{equation}\label{eq:growing-property}
       \mu_{\textnormal{av}}\lVert\Delta_{\textnormal{av}}^t\rVert^2- f(\lVert\bDelta^t_{\perp}\rVert)\leq G(\bT^t)-G(\hat{\bT})+\frac{\tau}{4}(v^2+8 h^2_{\max}),
    \end{equation}
    where $f(\lVert\bDelta^t_{\perp}\rVert)$ is defined as [cf.~\eqref{def:L-linearization}],   
    \begin{align*} 
          f(\lVert\bDelta^t_{\perp}\rVert)= \left(\frac{L_{\max}}{2m}-\frac{1-\rho}{2m\gamma} \right)\lVert\bDelta^t_{\perp}\rVert^2, 
    \end{align*}   $v^2$ is given by [cf.~\eqref{eq:def_v}]
\begin{equation}\label{eq:v_alpha}
    v^2=144s\lVert\hat{\nu}_{\textnormal{av}}\rVert_2^2+4\min\bigg\{\frac{2\eta}{\lambda}, 2R\bigg\}^2,\quad \text{with}\quad \eta=\alpha^2,
\end{equation} 
and $  h_{\max}$ is defined as [cf.~(\ref{def:h-max})]
\begin{equation*}
    h_{\max}=\frac{ d\gamma}{\lambda(1-\rho)}\bigg(\frac{\max_{i\in [m]}\lVert w_i^{\top} X_i\rVert_{\infty}}{n}+\lambda\bigg)^2.  
\end{equation*} 

   We proceed now to bound the LHS and RHS of \eqref{eq:growing-property}. 
   The goal is to lower bound the LHS by a quantity proportional to $\|\bDelta^t\|^2$, so that \eqref{eq:growing-property} will provide the desired bound of $\|\bDelta^t\|^2$ in terms of the  optimization gap $G(\bT^t)-G(\hat{\bT})$ (up to a tolerance). The following bound of 
   $f(\lVert\bDelta^t_{\perp}\rVert)$, which is a consequence of (\ref{f bound})   serves the scope:
    $$f(\lVert\bDelta^t_{\perp}\rVert)\leq  -\frac{ \mu_{\textnormal{av}}}{m}\lVert\bDelta^t_{\perp}\rVert^2.$$ 
    We also upper bound the RHS of \eqref{eq:growing-property} to further simplify the final expression; specifically, we use 
    \begin{equation}\label{eq:h_max_bound}
   h_{\max} \overset{\eqref{refined bound}}{\leq}  \frac{\lambda s}{64(\mu-32s\tau)}
\end{equation}
   and    
\begin{equation*}
     144s\lVert\hat{\nu}_{\text{av}}\rVert^2+4\min\bigg\{\frac{2\alpha^2}{\lambda}, 2R\bigg\}^2\leq \frac{144s\lVert\hat{\bnu}\rVert^2}{m}+\frac{16\alpha^4}{\lambda^2},
\end{equation*}
  where the inequality follows from $\lVert\hat{\nu}_{
    \text{av}}\rVert^2\leq     {\lVert\hat{\bnu}\rVert^2}/{m}$ and the fact that   $\alpha^2\leq R\lambda/4$  [cf.~\eqref{determinstic alpha}].
    
    Using the above bounds   along with \eqref{objective decrease} in (\ref{eq:growing-property})   yield: for all $t\geq T$, 
\begin{align}\label{arsc nu average t}
   \mu_{\textnormal{av}} \frac{\lVert\bDelta^t\rVert^2}{m}&
    \leq \alpha^2 +  \frac{\tau}{4}\bigg(\frac{144s\lVert\hat{\bnu}\rVert^2}{m}+8\,\left(\frac{\lambda s}{{64}(\mu-32s\tau)}\right)^2+\frac{16 \alpha^4}{\lambda^2}\bigg)\notag\\
    &\leq{\alpha^2 +  \frac{36s\tau\lVert\hat{\bnu}\rVert^2}{m}+\frac{\tau s\lambda^2 s}{{1976\mu^2}}+\frac{4\tau\alpha^4}{\lambda^2}},
\end{align} 
where the last inequality follows from   $\mu\geq \tilde{c}_{10}s\tau$ with $\tilde{c}_{10}=1824$. This proves   \eqref{final bound lagrangian}. 
 $\hfill\square$



\section{Proof of Theorem \ref{Th_stat_convergence}}\label{Statistical results for DGD-CTA first order methods}
 Given the postulated random data model and  the conclusions of  Theorem~\ref{deterministic optimization error result},   it is sufficient to prove the following:    \textbf{Step 1:} Under (\ref{global lambda condition}) on $\lambda$,  condition \eqref{lambda condition} holds with high-probability; 
 \textbf{Step 2:} Condition \eqref{gamma-final-algo} on $\gamma$ is sufficient for \eqref{constrained gamma condition2} to hold with high probability; 
\textbf{ Step 3:} Under  condition (\ref{eq:cond_on_N}) on $N$,  any  $R$ in (\ref{eq:alg-conditions-random})  satisfies also~\eqref{R condidion}; furthermore, the interval of values of $R$ in (\ref{eq:alg-conditions-random}) is nonempty;  \textbf{ Step 4:}    Any $\alpha^2$ in \eqref{alpha stat range} satisfies  \eqref{determinstic alpha}  with high-probability; and the range for $\alpha^2$ in \eqref{alpha stat range} is nonempty with high-probability; 
 \textbf{ Step 5:} Given the bound on the statistical error as in \eqref{final bound lagrangian} and for all $t$ satisfying  \eqref{T1 expression},  we conclude that    \eqref{stat_algo_bound_general}   holds, for all $t$ satisfying \eqref{T2 expression}, with high-probability.

\smallskip  

\noindent {\bf $\bullet$ Step 1:  Sufficient condition on $\lambda$ for \eqref{lambda condition} to hold with high probability.}
To prove this result, we follow a similar path as introduced in the proof of Theorem \ref{statistical optimization error result}.\smallskip 

 {\noindent\bf (i) Randomness from $\mathbf{X}.$ }
 This step is the  same as  {\bf Step 1} in the proof of Theorem~\ref{statistical optimization error result} (cf.~Appendix~\ref{Proof of Theorem 3}), except for the definition of the event $A_2$ replaced here with $A_2'$, defined as 
 \begin{equation}\label{RSC_param}
    A_2'\triangleq
    \bigg\{\mathbf{X}\in\mathbb{R}^{N\times d}\bigg|\text{ } \mathbf{X}\text{ } \text{satisfies RSC}\text{ } \eqref{Arsc}\text{ } \text{with parameters} \text{ } (\mu,\tau)=\bigg(\lambda_{\min}(\Sigma), 2c_{1} \zeta_{\Sigma}\frac{\log d}{N}\bigg)\bigg\},
\end{equation}
where $\zeta_{\Sigma}=\max_{i\in [d]}\Sigma_{ii}.$ Lemma \ref{Global RE} implies 
\begin{equation}\label{mu tau high probability}
    \mathbb{P}(A_2')\geq1-\exp(-\tilde{c}_0N). 
\end{equation}
Define $A'=A_1\cap A_2'\cap A_3,$ where $A_1$ and $A_3$ are defined in \eqref{eq:def_A1} and \eqref{eq:def_A2_A3} (cf.~Appendix~\ref{Proof of Theorem 3}), respectively, and recall here for convenience
\begin{align*}
    A_1&\triangleq
    \bigg\{\mathbf{X}\in\mathbb{R}^{N\times d}~\bigg|\text{ } L_{\max}\leq \tilde{c}_4 \lambda_{\max}(\Sigma) \left(1+  \frac{d + \log m}{n}  \right)\bigg\} \ \ \text{and}\ \    \\
A_3&\triangleq\bigg\{\mathbf{X}\in\mathbb{R}^{N\times d}~\bigg|\text{ } \max_{j=1,\dots,d}\frac{1}{\sqrt{N}}\lVert \mathbf{X}e_j\rVert\leq \sqrt{\frac{3\zeta_{\Sigma}}{2}}\bigg\}.
\end{align*}
Then, similar to under condition \eqref{sample_stat_final}, there holds \eqref{eq:p:X}, since $N$ satisfies \eqref{eq:cond_on_N}
with $\tilde{c}_{12}=\max\{\tilde{c}_9,\tilde{c}_{11}\}$ $\tilde{c}_9=\max\{128\tilde{c}_1,\tilde{c}_5\}$, and $\tilde{c}_{11}=3648\tilde{c}_1,$ 
we have 
\begin{equation}\label{eq:p:X:alter}
\mathbb{P}(A')\geq 1- 2\exp(-\tilde{c}_3d)-\exp(-\tilde{c}_0N)-2\exp(-\tilde{c}_6\log d).
\end{equation}

\noindent {\bf (ii) Randomness from $\mathbf{w}.$}  This step follows {\bf Step 2} as in the proof of Theorem \ref{statistical optimization error result} (cf.~Appendix~ \ref{Proof of Theorem 3}) and thus is not duplicated. In particular,  recalling the definitions of $D_1, D_2,$ and $D$ therein  for convenience
 \begin{align*}
    \begin{split}
         &D_1\triangleq \left\{\mathbf{w}\in\mathbb{R}^N\text{ }\bigg|\text{ }\frac{\lVert \mathbf{X}^{\top} \mathbf{w}\rVert_{\infty}}{N}\leq \sigma\sqrt{\frac{t_0\log d}{N}}\sqrt{\frac{3\zeta_{\Sigma}}{2}}\right\}, \\
         &D_2\triangleq \left\{\mathbf{w}\in\mathbb{R}^N\text{ }\bigg|\text{ }\frac{\max_{i\in [m]}\lVert X_i^{\top} w_i\rVert_{\infty}}{n}\leq   {\sigma\sqrt{\zeta_\Sigma}\min\left\{{\frac{2\log md}{n\tilde{c}_{24}}},\sqrt{\frac{{2\log md}}{n\tilde{c}_{24}}}\right\}}\right\},
    \end{split}
\end{align*}
$\text{and}\ \  D \triangleq D_1 \cap D_2.$
The following the same reasoning as in {\bf Step 2} of the proof of Theorem~\ref{statistical optimization error result}, we have, for all $t_0\geq 2,$
\begin{align}
\begin{split}
    \mathbb{P}(A'\cap D)\geq&{1- 11\exp(-\tilde{c}_{8}\log d)}.\label{eq:p-bound-23:alter}
\end{split}
\end{align}
 \smallskip  
{\bf (iii) Sufficient condition on $\lambda$ for \eqref{lambda condition} to hold with high probability.} Recall \eqref{lambda condition} for convenience,
\begin{equation*}
    \lambda  \geq \max\left\{\frac{2\| \bX^\top \mathbf{w}\|_\infty}{N},{64} \tau \| \T^*\|_1 \right\}.
\end{equation*}
Combining it with the high probability upper bound for $\|\mathbf{X}^{\top}\mathbf{w}\|_{\infty}/N$  derived in \eqref{eq:p-bound-23:alter}, we conclude the following: suppose $ \lambda\geq \sigma\sqrt{\frac{6\zeta_{\Sigma}t_0\log d}{N}}$, then for any tuple $(\mathbf{X},\mathbf{w})\in A'\cap D,$ and any $t_0>2,$, since $2\|\mathbf X^\top \mathbf w\|/N \leq \sigma \sqrt{\frac{6\zeta_{\Sigma}t_0\log d}{N}}$, it follows that $\lambda \geq 2\|\mathbf{X}^{\top} \mathbf{w}\|_\infty/N.$ 
~That is,
\begin{align*}
   \mathbb{P}\bigg(\lambda \geq \frac{2\|\mathbf{X}^{\top} \mathbf{w}\|_\infty}{N} \ \bigg)
   \geq\ & \mathbb{P}(A'\cap D)
   \overset{\eqref{eq:p-bound-23:alter}}{\geq}{1- 11\exp(-\tilde{c}_{8}\log d)}.
\end{align*}
Furthermore, for any tuple $(\mathbf{X},\mathbf{w})\in A'\cap D,$ \eqref{RSC_param}  implies $\tau= 2c_{1} \zeta_{\Sigma}\log d/N$. Therefore it follows that if $\lambda\geq\frac{128s\tilde{c}_1\zeta_{\Sigma}\log d}{N}$, then $64\tau\| \T^*\|_1$.
Using \eqref{global lambda condition}, we conclude that for any $t_0>2,$     
  {
\begin{equation*}
    \lambda\geq{c}_{11}\max\bigg\{ \sigma\sqrt{\frac{\zeta_{\Sigma}t_0\log d}{N}},\frac{s\zeta_{\Sigma}\log d}{N}\bigg\},
\end{equation*}
with ${c}_{11}=\max\{\sqrt{6},128\tilde{c}_1\},$
}
is sufficient for \eqref{lambda condition}  to hold with probability at least \eqref{eq:p-bound-23:alter}. 

 \smallskip
 
\noindent{\bf $\bullet$ Step 2: \eqref{gamma-final-algo} is sufficient  for \eqref{constrained gamma condition2} to hold with high probability.}

Recall \eqref{constrained gamma condition2} for convenience,
\begin{equation*}
    \gamma \leq\frac{1-\rho}{2L_{\max}+(\mu/2-16s\tau)\left(1+128(d/s) (\max_{  i\in [m]}\lVert w_i^{\top} X_i\rVert_{\infty}/(\lambda n)+2\sqrt{m})^2\right)}.
\end{equation*}
In order to derive a sufficient  condition on $\gamma$ to ensure  \eqref{constrained gamma condition2}  holds with high probability, we leverage  \textbf{Step 1 (i)} above, where we derived high probability  bounds for $L_{\max}$,  $\max_{  i\in [m]}\lVert w_i^{\top} X_i\rVert_{\infty}/n$, and   $\lambda$. 
Specifically, substituting  into \eqref{constrained gamma condition2} the bounds on $L_{\max}$ [as in  \eqref{new bound l-max}],  $\max_{i\in [m]}\lVert X_i^{\top} w_i\rVert_{\infty}/n$ [as in \eqref{eq:p-bound-23:alter}], and the explicit expression of the RSC parameters $(\mu, \tau)$ [as in \eqref{RSC_param}],   we conclude that if
\begin{equation}\label{inter-gamma-algo}
    \gamma\leq\frac{1-\rho}{2\tilde{c}_4\lambda_{\max}(\Sigma)\left(1+\frac{d+\log m}{n}\right)+\left(\frac{\lambda_{\min}(\Sigma)}{2}-\frac{32s\tilde{c}_1\zeta_{\Sigma}\log d}{N}\right)\left[1+ \frac{128d}{s}\left(  {g(m,d)}+2\sqrt{m}\right)^2\right]},
\end{equation}
where $  {g(m,d)=\frac{\sigma}{\lambda}\sqrt{\zeta_\Sigma}\min\left\{{\frac{2\log md}{n\tilde{c}_{24}}},\sqrt{\frac{{2\log md}}{n\tilde{c}_{24}}}\right\}},$
  then \eqref{constrained gamma condition2}   holds with probability  at least  \eqref{probability}. 
We proceed by showing that \eqref{gamma-final-algo} is sufficient  for (\ref{inter-gamma-algo}).
Specifically, since 
  {\begin{align*}
   g(m,d)&=\frac{\sigma}{\lambda}\sqrt{\zeta_\Sigma}\min\left\{{\frac{2\log md}{n\tilde{c}_{24}}},\sqrt{\frac{{2\log md}}{n\tilde{c}_{24}}}\right\}
    \overset{\eqref{global lambda condition}}{\leq} \sqrt{\frac{m\log md}{3t_0\tilde{c}_{24}\log d}},
\end{align*}}
and 
\begin{equation*}
    \frac{\lambda_{\min}(\Sigma)}{2}\geq\frac{\lambda_{\min}(\Sigma)}{2}-\frac{32s\tilde{c}_1\zeta_{\Sigma}\log d}{N}\overset{\eqref{eq:cond_on_N}}{\geq}0,
\end{equation*}
we obtain the more conservative condition
\begin{align}\label{stat_gamma_original}
    \gamma&\leq\frac{1-\rho}{2\tilde{c}_4\lambda_{\max}(\Sigma)\left(1+\frac{d+\log m}{n}\right)+\frac{\lambda_{\min}(\Sigma)}{2}\left[1+ \frac{128d}{s}\left(   {\sqrt{\frac{m\log md}{3t_0\tilde{c}_{24}\log d}}}+2\sqrt{m}\right)^2\right]}.
\end{align}
We proceed to further simplify \eqref{stat_gamma_original}. Notice that 
  {
\begin{align}\label{gamma_sim_1}
     1 + \frac{128d}{s} \left( \sqrt{\frac{m\log md}{3t_0\tilde{c}_{24}\log d}}+2\sqrt{m}\right)^2 
   & \overset{\text{(a)}}{\leq}  1+ \frac{128d}{s} \cdot  \left[{\frac{1}{3\tilde{c}_{24}}}\left({{2\log m}}+1\right)+8\right]{m} \notag\\
   &\overset{\text{(b)}}{\leq}  {256dm}\cdot  \left[\left({{2\log m}}+1\right)/\tilde{c}_{24}+8\right],
\end{align}
where (a) is due to 
$1\leq 2\log d$, for $d\geq2$ and $t_0\geq 2;$
and in (b)  we upper bound both   terms   by ${128d} \cdot  \left[\left({{2\log m}}+1\right)/\tilde{c}_{24}+8\right].$  }
Using \eqref{gamma_sim_1} and  further simplification,  we have
   {\begin{align}
&\frac{1-\rho}{2\tilde{c}_4\lambda_{\max}(\Sigma)\left(1 + \frac{d+\log m}{n}\right)+128\lambda_{\min}(\Sigma){dm}\cdot  \left[\left({{2\log m}}+1\right)/\tilde{c}_{24}+8\right]}\notag\\
&\geq \frac{ c_{12}(1-\rho)}{\lambda_{\max}(\Sigma)\left({d+\log m}\right)+\lambda_{\min}(\Sigma){dm}\cdot \left({{\log m}}+1\right)},
 \end{align}
 where $c_{12}\triangleq\frac{1}{\max\{4\tilde{c}_4,512/\tilde{c}_{24},2048\}}$. Hence, 
 }
 Hence,  under  \eqref{gamma-final-algo}, 
   \eqref{constrained gamma condition2}  holds with probability  at least  \eqref{probability}. 




\smallskip

\noindent{\bf $\bullet$ Step 3: Ensuring there exists an $R$ fulfilling~\eqref{R condidion}.}
Substituting in \eqref{R condidion} the explicit expression of the RSC parameters $(\mu, \tau)$ [under the event in  \eqref{RSC_param}] as well as $\|\T^*\|_1=s$,  we conclude that  \eqref{R condidion} holds with probability at least $1-\exp(-\tilde{c}_0N)$, whenever  $R$ satisfies display \eqref{eq:alg-conditions-random},
\begin{equation*}
    \max\bigg\{\frac{56\lambda s}{\lambda_{\min}(\Sigma)-64s\tilde{c}_1\zeta_{\Sigma}\log d/N}, 2s\bigg\}\leq R\leq \frac{\lambda N}{64\tilde{c}_1\zeta_{\Sigma}\log d}.
\end{equation*}
We now show that the interval \eqref{eq:alg-conditions-random} is non-empty.
Since $N$ satisfies \eqref{eq:cond_on_N} 
with $c_{10}=\tilde{c}_{12}=\max\{\tilde{c}_5,3648\tilde{c}_1\}$  there holds
\begin{equation}\label{R_nonempty_1}
 \frac{56\lambda s}{\lambda_{\min}(\Sigma)-64s\tilde{c}_1\zeta_{\Sigma}\log d/N}\leq\frac{\lambda N}{64\tilde{c}_1\zeta_{\Sigma}\log d}.
\end{equation}
Furthermore,   \eqref{global lambda condition} [{\bf Step 1 (iii)}] implies  
\begin{equation}\label{R_nonempty_2}
    2s\leq\frac{\lambda N}{64\tilde{c}_1\zeta_{\Sigma}\log d}.
\end{equation}
By  \eqref{R_nonempty_1} and \eqref{R_nonempty_2}, we infer that \eqref{eq:alg-conditions-random} is non-empty.
\smallskip

\noindent{\bf $\bullet$ Step 4: \eqref{alpha stat range} is sufficient for \eqref{determinstic alpha} to hold with high-probability.} 
Substituting in \eqref{determinstic alpha} the explicit expression of the RSC parameters $(\mu, \tau)$ [under the event  \eqref{RSC_param}], we conclude that \eqref{alpha stat range} is in fact sufficient for \eqref{determinstic alpha} to hold with probability at least  \eqref{eq:p-bound-23:alter}. 

 It remains to prove that the (random) interval \eqref{alpha stat range} is non-empty with high probability, which we do next. 
 To this end, we upper bound the statistical error $\sum_{i=1}^m\lVert\hat{\T}_i-\T^*\rVert^2/m$, under \eqref{determinstic alpha}. Recall that, with probability  at least \eqref{eq:p-bound-23:alter},   \eqref{constrained gamma condition2} holds. Therefore, we can invoke Theorem~\ref{solution err bound} to bound the statistical error, and write: with probability at least \eqref{eq:p-bound-23:alter}, 
 \begin{align*}
      &\,\,\frac{1}{m}\sum\limits_{i=1}^m\lVert\hat{\T}_i-\T^\ast\rVert^2
  \notag\\
    &\overset{\text{Thm.~\ref{solution err bound}}}{\leq} 
   \,\,\frac{9\lambda^2s}{\delta^2}+\frac{2\xi}{\delta}\underbrace{\frac{ d^2\gamma^2({\max_{i\in[m]}\lVert w_i^{\top} X_i\rVert_{\infty}}+\lambda n)^4}{\lambda^2n^4(1-\rho)^2}}_{(\texttt{Term I})^2 \ \texttt{in  \eqref{deterministic solution equivalence}}}+\frac{4}{\delta}\underbrace{\frac{d\gamma({\max_{i\in[m]}\lVert w_i^{\top} X_i\rVert_{\infty}}+\lambda n)^2}{ n^2[2(1-\rho)-4L_{\max}\gamma-\delta\gamma]}}_{(\texttt{Term II})^2 \ \texttt{in  \eqref{deterministic solution equivalence}}}\notag\\
  &\,\overset{\eqref{refined bound}}{\leq}\frac{9\lambda^2s}{\delta^2}+\frac{2\xi}{\delta} \bigg(\frac{\lambda s}{ {128}\,\delta}\bigg)^2+\frac{4}{\delta}\frac{\lambda^2s}{{256}\,\delta}\notag\\
  &\,\,\overset{(a)}{=}  \frac{9\lambda^2s}{(\mu/2-16s\tau)^2}+\frac{2\tau}{\mu/2-16s\tau} \bigg(\frac{\lambda s}{ {128}\,(\mu/2-16s\tau)}\bigg)^2+\frac{4}{\mu/2-16s\tau}\frac{\lambda^2s}{{256}\,(\mu/2-16s\tau)}\\
  &\,\overset{\eqref{RSC_param}}{ \leq} \frac{1}{(\lambda_{\min}(\Sigma)-64s\tilde{c}_1 \zeta_{\Sigma}\log d/N)^2} \notag\\
  &\qquad \left(36\lambda^2s+\frac{16}{128^2}\frac{2s\tilde{c}_1 \zeta_{\Sigma}\log d}{N}\frac{\lambda^2s}{(\lambda_{\min}(\Sigma)-64s\tilde{c}_1 \zeta_{\Sigma}\log d/N)}+\frac{16}{256}\lambda^2s\right),
 \end{align*}
 where in (a) we used   $\xi=\tau$ and $\delta=\mu/2-16s\tau>0$ (due to Lemma~\ref{global-ARE-determin}). 

Thus, with probability  at least \eqref{eq:p-bound-23:alter}, we can upper bound the lower interval bound  in \eqref{alpha stat range} by
\begin{align*}
   &\frac{8\cdot36s\tilde{c}_1 \zeta_{\Sigma}\log d}{N(\lambda_{\min}(\Sigma)-64s\tilde{c}_1 \zeta_{\Sigma}\log d/N)^2} \left(36\lambda^2s+\frac{16}{128^2}\frac{2s\tilde{c}_1 \zeta_{\Sigma}\log d}{N}\frac{\lambda^2s}{(\lambda_{\min}(\Sigma)-64s\tilde{c}_1 \zeta_{\Sigma}\log d/N)}\right.\notag\\
    &\quad\left.+\frac{16}{256}\lambda^2s\right)
   +\frac{8s\tilde{c}_1 \zeta_{\Sigma}\log d}{N}{\frac{\lambda^2s}{1976(\lambda_{\min}(\Sigma)-64s\tilde{c}_1 \zeta_{\Sigma}\log d/N)^2}}\notag\\
    &\leq\frac{288s\tilde{c}_1 \zeta_{\Sigma}\log d}{N(\lambda_{\min}(\Sigma)-64s\tilde{c}_1 \zeta_{\Sigma}\log d/N)^2} \left(\frac{s\tilde{c}_1 \zeta_{\Sigma}\log d}{512N}\frac{\lambda^2s}{(\lambda_{\min}(\Sigma)-64s\tilde{c}_1 \zeta_{\Sigma}\log d/N)}+ 37\lambda^2s\right).
\end{align*} 
Using the bound on $N$ given by~\eqref{eq:cond_on_N}
\begin{align*}
  N\geq\frac{\tilde{c}_{12}s\zeta_\Sigma\log d}{\lambda_{\min}(\Sigma)}, \quad \text{with} \quad \tilde{c}_{12}=\max\{3648\tilde{c}_1, \tilde{c}_5\},
\end{align*}
we obtain
$   64s\tilde{c}_1 \zeta_{\Sigma}\log d/N \leq  \lambda_{\min}(\Sigma)/57.$ 
Substituting into the inequality above we have
\begin{align*}
   &\frac{8s\tilde{c}_1 \zeta_{\Sigma}\log d}{N} \bigg(\frac{36}{m}\sum\limits_{i=1}^m\lVert\hat{\T}_i-\T^*\rVert^2+{\frac{\lambda^2s}{1976\lambda_{\min}(\Sigma)^2}}\bigg)\notag\\  
   &\leq  \frac{288s\tilde{c}_1 \zeta_{\Sigma}\log d}{N(\lambda_{\min}(\Sigma)-64s\tilde{c}_1 \zeta_{\Sigma}\log d/N) \left( \frac{56}{57} \lambda_{\min}(\Sigma)\right)} \left( \lambda^2s + 37\lambda^2s\right).
\end{align*}
Applying again the lower bound on $N$, we  further get 
\begin{align*}
    &N(\lambda_{\min}(\Sigma)-64s\tilde{c}_1 \zeta_{\Sigma}\log d/N) \\
    &\geq \max \left\{ \frac{\tilde{c}_{12}s\zeta_\Sigma\log d}{\lambda_{\min}(\Sigma)} (\lambda_{\min}(\Sigma)-64s\tilde{c}_1 \zeta_{\Sigma}\log d/N), N \cdot  \frac{56}{57} \lambda_{\min}(\Sigma) \right\},
\end{align*}
and thus
\begin{align*}
   &\frac{8s\tilde{c}_1 \zeta_{\Sigma}\log d}{N} \bigg(\frac{36}{m}\sum\limits_{i=1}^m\lVert\hat{\T}_i-\T^*\rVert^2+{\frac{\lambda^2s}{1976\lambda_{\min}(\Sigma)^2}}\bigg)\notag\\  
   &\leq  \frac{288s\tilde{c}_1 \zeta_{\Sigma}\log d}{ \frac{56}{57} \lambda_{\min}(\Sigma)} \left(  38\lambda^2s\right) \cdot \min \left\{  \frac{\lambda_{\min}(\Sigma)}{\tilde{c}_{12}s\zeta_\Sigma\log d} (\lambda_{\min}(\Sigma)-64s\tilde{c}_1 \zeta_{\Sigma}\log d/N)^{-1},\frac{57}{56} \lambda_{\min}(\Sigma)^{-1} \right\}\\
  &\leq  \min \left\{ 4 (\lambda_{\min}(\Sigma)-64s\tilde{c}_1 \zeta_{\Sigma}\log d/N)^{-1}  \lambda^2s, \quad 11339 \cdot \frac{s\tilde{c}_1 \zeta_{\Sigma}\log d}{  \lambda_{\min}(\Sigma)^2} \left( \lambda^2s\right)  \right\}
  \leq  \min \left\{ \frac{\lambda R}{4}, \eta_G^0\right\}.
\end{align*}
The last inequality follows from the conditions on $R$ and $\eta_G^0$ given by~\eqref{eq:alg-conditions-random} and~\eqref{initial_gap_Thm12}, respectively.
\smallskip

\noindent{\bf $\bullet$ Step 5:   \eqref{stat_algo_bound_general} holds, for all $t$ satisfying \eqref{T2 expression}},  with high probability. 

Building on the conclusions of the previous steps and   
Theorem~\ref{deterministic optimization error result},  to prove the statement of this step, it is sufficient to show that the RHS of \eqref{stat_algo_bound_general}  [resp.  of \eqref{T2 expression}] is an  upper bound of the RHS of  \eqref{final bound lagrangian} [resp. (\ref{T1 expression})] that  holds with high probability. 

 We begin with the RHS of \eqref{final bound lagrangian}:  with probability at least \eqref{eq:p-bound-23:alter}, there holds,
\begin{align*}
   &{\frac{1}{\mu/8-8s\tau}\alpha^2 +  \frac{36s\tau\lVert\hat{\bnu}\rVert^2}{m(\mu/8-8s\tau)}+\frac{\tau s\lambda^2 s}{{1976\mu^2(\mu/8-8s\tau)}}+\frac{4\tau\alpha^4}{\lambda^2(\mu/8-8s\tau)}}\notag\\
   & \overset{(a)}{\leq}\frac{456}{55\lambda_{\min}(\Sigma)}\left(\alpha^2+\frac{72s\tilde{c}_1\zeta_{\Sigma}\log d}{N}\frac{1}{m}\sum\limits_{i=1}^m\|\hat{\T}_i-\T^*\|^2+\frac{s\tilde{c}_1\zeta_{\Sigma}\log d}{988N}\frac{\lambda^2s}{\lambda_{\min}(\Sigma)^2}+\frac{8s\tilde{c}_1\zeta_{\Sigma}\log d}{N}\frac{\alpha^4}{\lambda^2s}\right)\notag\\
   &\leq \frac{{c}_{16}}{\lambda_{\min}(\Sigma)}\left[\alpha^2+\frac{s\zeta_{\Sigma}\log d}{N}\left(\frac{1}{m}\sum\limits_{i=1}^m\|\hat{\T}_i-\T^*\|^2+\frac{\lambda^2s}{\lambda_{\min}(\Sigma)^2}+\frac{\alpha^4}{\lambda^2s}\right)\right],
\end{align*} 
where in (a)  we use the following  fact (which holds with  probability at least \eqref{eq:p-bound-23:alter})
\begin{align*}
    \frac{\mu}{8}-8s\tau=\ &\frac{\lambda_{\min}(\Sigma)}{8}-16s\tilde{c}_1 \zeta_{\Sigma}\frac{\log d}{N} 
    \overset{\eqref{eq:cond_on_N}}{\geq}\frac{\lambda_{\min}(\Sigma)}{8}-\frac{16s\tilde{c}_1 \zeta_{\Sigma}\log d\lambda_{\min}(\Sigma)}{3648\tilde{c}_1s\zeta_{\Sigma}\log d}=\frac{55\lambda_{\min}(\Sigma)}{456}.
\end{align*}
 Next,  we bound  the RHS of (\ref{T1 expression}), invoking the high probability   bound for ${L_{\max}}$ [as in  \eqref{new bound l-max}] and the explicit expression of the RSC parameters $(\mu, \tau)$ [under  \eqref{RSC_param}]. We have the following
 \begin{align}
 \begin{split}
   &\bigg\lceil\log_2\log_2\left(\frac{R\lambda}{\alpha^2}\right)\bigg\rceil\bigg(1+\frac{L_{\max}\log 2}{\mu_{\text{av}}}+\frac{(1+\rho)\log 2}{\gamma\mu_{\text{av}}} \bigg) +  \left(\frac{ L_{\max}}{\mu_{\text{av}}} + \frac{1+\rho}{\gamma\mu_{\text{av}}}\right) \log\left(\frac{\eta^0_G}{ \alpha^2}\right)\\
    &\leq 
   \bigg\lceil\log_2\log_2\left(\frac{R\lambda}{\alpha^2}\right)\bigg\rceil\bigg(1+\frac{\lambda_{\max}(\Sigma)}{\lambda_{\min}(\Sigma)}\frac{456\tilde{c}_4  [1+ (d + \log m)/n]\log 2}{55}+\frac{456(1+\rho)\log 2}{55\lambda_{\min}(\Sigma)\gamma} \bigg) \\
     &\quad+  \bigg(\frac{\lambda_{\max}(\Sigma)}{\lambda_{\min}(\Sigma)}\frac{456\tilde{c}_4 [1+ (d + \log m)/n]}{55}+\frac{456(1+\rho)}{55\gamma\lambda_{\min}(\Sigma)} \bigg) \log\left(\frac{\eta^0_G}{ \alpha^2}\right).
 \end{split}\label{T-2_}
 \end{align} 
{
 Define 
 \begin{equation*}
     c_{18}\triangleq4\max\left\{\frac{456\tilde{c}_4\log 2}{55},\frac{456\log 2}{55}\right\}=\frac{1824\tilde{c}_4\log 2}{55}. 
 \end{equation*}
 Then, the RHS of \eqref{T-2_} can be bounded as 
 \begin{align*}
    c_{18}\left[\left\lceil\log_2\log_2\left(\frac{R\lambda}{\alpha^2}\right)\right\rceil+ \log\left(\frac{\eta^0_G}{ \alpha^2}\right)\right]\bigg({\kappa_{\Sigma}} (d + \log m)+\frac{(1+\rho)}{\lambda_{\min}(\Sigma)\gamma} \bigg). 
 \end{align*}
 }
 This completes the proof. \hfill $\square$

\section{Proof of Corollary \ref{Th_stat_convergence-alter}}\label{sec:pf-cor-13}
The corollary is a customization of 
   Theorem~\ref{Th_stat_convergence}, under the   (feasible) choices of $N$,   $\lambda$ and   $\gamma$  as in the statement of the corollary. 
  
 
 \smallskip
 
\noindent{\bf $\bullet$  Step 1: On the  choices of $\lambda$ and $\gamma.$} We show that, under \eqref{global sample size condition-alter}, \eqref{lambda}  and \eqref{gamma-final-algo-alter} are special instances of \eqref{global lambda condition} and \eqref{gamma-final-algo}, respectively. 

 Since  $N$ satisfies \eqref{global sample size condition-alter}, with  $\tilde{c}_{12}=\max\{3648\,\tilde{c}_1,\tilde{c}_5\}$  and $\tilde{c}_{13}=2731\tilde{c}_1^2/t_0,$ there holds 
\begin{equation}\label{lambda compare}
    \sigma\sqrt{\frac{6\zeta_{\Sigma}t_0\log d}{N}}\geq 128\tilde{c}_1\zeta_{\Sigma} \frac{ s\log d}{N}.
\end{equation}
Therefore, \eqref{global lambda condition} reduces to
\begin{equation*}
    \lambda \geq \sigma\sqrt{\frac{6\zeta_{\Sigma}t_0\log d}{N}},
\end{equation*}
which is satisfied by the choice of $\lambda$ as in \eqref{lambda}, with $\tilde{c}_8$ being any constant such that $\tilde{c}_8\geq\sqrt{6}$.

Consider now the condition on $\gamma$ as in \eqref{gamma-final-algo}. Since we are interested in the high-dimensional regime where   $N \ll d$, we assume that $d + \log m \geq n$. Using this, we can lower bound the RHS of \eqref{gamma-final-algo} and readily obtain the more stringent condition on $\gamma$ as in \eqref{gamma-final-algo-alter}, with  $\tilde{c}_{14}=1152$.

\noindent{\bf $\bullet$ Step 2: Condition on $R$ in \eqref{eq:alg-conditions-random-alter} implies \eqref{eq:alg-conditions-random}.} 
Using again \eqref{global sample size condition-alter}, we can upper bound the lower interval of $R$ in \eqref{eq:alg-conditions-random} as 
\begin{align*}
       \frac{56\lambda s}{\lambda_{\min}(\Sigma)-64s\tilde{c}_1\zeta_{\Sigma}\log d/N} 
   \overset{\eqref{global sample size condition-alter}}{\leq}    \frac{56\lambda s}{\lambda_{\min}(\Sigma)-\lambda_{\min}(\Sigma)/57}  
    \overset{\eqref{lambda}}{=}  \frac{57\tilde{c}_8s}{\lambda_{\min}(\Sigma)}  \sigma\sqrt{\frac{6t_0 \zeta_{\Sigma}\log d}{N}}.
\end{align*}
Using  \eqref{lambda},   the upper interval in the same condition reads
\begin{align*}
\frac{\lambda N}{64\tilde{c}_1\zeta_{\Sigma}\log d} = \frac{ \tilde{c}_8}{64\tilde{c}_1} \sigma\sqrt{\frac{6 t_0N}{\zeta_{\Sigma}\log d}}.
\end{align*}
Therefore, \eqref{eq:alg-conditions-random-alter} is sufficient for~\eqref{eq:alg-conditions-random} to hold, with 
$\tilde{c}_{15}=57\sqrt{6}\tilde{c}_8$ and $\tilde{c}_{16}= {\sqrt{6}\tilde{c}_8}/({64\tilde{c}_1 
}).$

 It remains to show that the range of value of $R$ in \eqref{eq:alg-conditions-random-alter} is nonempty. By \eqref{global sample size condition-alter}, 
\begin{equation*}
      N\geq\frac{\tilde{c}_{12} s \zeta_\Sigma \log d}{\lambda_{\min} (\Sigma)} \ \Rightarrow\ \frac{\sqrt{6}\tilde{c}_8}{64\tilde{c}_1 
}\sigma\sqrt{\frac{t_0N }{\zeta_{\Sigma}\log d}}\geq \frac{57\sqrt{6}\tilde{c}_8}{\lambda_{\min}(\Sigma)} s \sigma\sqrt{\frac{t_0 \zeta_{\Sigma}\log d}{N}},
\end{equation*}
and 
\begin{equation*}
    N\geq\frac{\tilde{c}_{13}s^2\zeta_{\Sigma}\log d}{\sigma^2}\ \Rightarrow\  N\geq\frac{\tilde{c}_{13}s^2\zeta_{\Sigma}\log d}{\tilde{c}_8^2\sigma^2}\ \Rightarrow\ \frac{\sqrt{6}\tilde{c}_8}{64\tilde{c}_1 
}\sigma\sqrt{\frac{t_0N }{\zeta_{\Sigma}\log d}}\geq 2s.
\end{equation*}

\noindent {\bf $\bullet$ Step 3: \eqref{alpha stat range} reduces to \eqref{alpha stat range-alter},    under  \eqref{lambda}.} The statement follows by a direct substitution of \eqref{lambda}  in \eqref{alpha stat range}:  
\begin{equation*}
\frac{\lambda^2s}{1976\lambda_{\min}(\Sigma)^2}= \frac{\tilde{c}_8^2\sigma^2\zeta_{\Sigma}t_0s\log d}{1976N\lambda_{\min}(\Sigma)^2}=\frac{\tilde{c}_{21}\sigma^2\zeta_{\Sigma}t_0}{\lambda_{\min}(\Sigma)^2}\frac{s\log d}{N},
\end{equation*}
where $\tilde{c}_{21}=\tilde{c}_8^2/1976,$ and 
\begin{equation*}
    \frac{R\lambda}{4}=\frac{R\sigma \tilde{c}_8}{4}\sqrt{\frac{\zeta_{\Sigma}t_0\log d}{N}}, \ \ {\eta_G^0\geq \frac{11339s\tilde{c}_1 \zeta_{\Sigma}\log d}{N\lambda_{\min}(\Sigma)^2} \lambda^2s=\tilde{c}_{22}\sigma^2t_0\left(\frac{s \zeta_{\Sigma}\log d}{N\lambda_{\min}(\Sigma)}\right)^2,}
\end{equation*}
{with  $\tilde{c}_{22}=11339\tilde{c}_1\tilde{c}_8^2.$}
Notice that  the (random) interval \eqref{alpha stat range-alter} is non-empty, with probability at least \eqref{probability}. This follows   from the fact that   \eqref{alpha stat range} is nonempty with the same probability, for all $R$ satisfies \eqref{eq:alg-conditions-random} (Theorem~\ref{Th_stat_convergence}), and that    \eqref{eq:alg-conditions-random-alter} is sufficient for~\eqref{eq:alg-conditions-random} to hold (\textbf{Step 2} above).

\noindent {\bf $\bullet$ Step 4: \eqref{eq:sat_ball} holds with high probability, for all $t$ satisfying \eqref{T2 expression-alter}.}
Eq.~\eqref{eq:sat_ball} follows readily from (\ref{stat_algo_bound_general}) by substitution of the values of    $\lambda$ and $\gamma$ as in \eqref{lambda}  and \eqref{gamma-final-algo-alter}, respectively; and defining the following  constants 
$\tilde{c}_{17}=9,$ $\tilde{c}_{18}=72\tilde{c}_1\tilde{c}_{17}$,  $\tilde{c}_{19}=\tilde{c}_1\tilde{c}_8^2\tilde{c}_{17}/988$, and  $\tilde{c}_{20}=8\tilde{c}_1\tilde{c}_{17}/\tilde{c}_8^2.$   

We conclude the proof showing that \eqref{T2 expression-alter} is a stronger condition than \eqref{T-2_}.  Using \eqref{lambda}  and \eqref{gamma-final-algo-alter}, explicitly writen as the following
\begin{equation*}
         \lambda=\tilde{c}_8\sigma\sqrt{\frac{\zeta_{\Sigma}t_0\log d}{N}}
\end{equation*}
and 
 \begin{equation*}
   {
  \gamma\leq\frac{1-\rho}{2\tilde{c}_4\lambda_{\max}(\Sigma)\left(1+\frac{d+\log m}{n}\right)+128\lambda_{\min}(\Sigma){dm}\cdot  \left[\left({{2\log m}}+1\right)/\tilde{c}_{24}+8\right]},
}
 \end{equation*}
the RHS of \eqref{T-2_} reads 
\begin{align}
  &  \bigg\lceil\log_2\log_2\left(\frac{\tilde{c}_8\sigma R}{\alpha^2}\sqrt{\frac{\zeta_{\Sigma}t_0\log d}{N}}\right)\bigg\rceil\left(1+\frac{\lambda_{\max}(\Sigma)}{\lambda_{\min}(\Sigma)}\frac{456\tilde{c}_4  [1+ (d + \log m)/n]\log 2}{55}\right)\notag\\
    &+\left\{\bigg\lceil\log_2\log_2\left(\frac{\tilde{c}_8\sigma R}{\alpha^2}\sqrt{\frac{\zeta_{\Sigma}t_0\log d}{N}}\right)\bigg\rceil \log 2+\log\left(\frac{\eta^0_G}{ \alpha^2}\right)\right\}\notag\\
    &\times \frac{456(1+\rho)}{55\lambda_{\min}(\Sigma)} \cdot  {\frac{2\tilde{c}_4\lambda_{\max}(\Sigma)\left(1 + \frac{d+\log m}{n}\right)+128\lambda_{\min}(\Sigma){dm}\cdot  \left[\left({{2\log m}}+1\right)/\tilde{c}_{24}+8\right]}{1-\rho}}\notag\\
     &+  \frac{456\tilde{c}_4 [1+ (d + \log m)/n]}{55} \frac{\lambda_{\max}(\Sigma)}{\lambda_{\min}(\Sigma)}\log\left(\frac{\eta^0_G}{ \alpha^2}\right)\notag\\
   \stackrel{\rho \leq 1}{\leq} \  & 
    \bigg\lceil\log_2\log_2\left(\frac{\tilde{c}_8\sigma R}{\alpha^2}\sqrt{\frac{\zeta_{\Sigma}t_0\log d}{N}}\right)\bigg\rceil\left(1+\frac{\lambda_{\max}(\Sigma)}{\lambda_{\min}(\Sigma)}\frac{456\tilde{c}_4  [1+ (d + \log m)/n]\log 2}{55}\right)\notag\\
    &+\left\{\bigg\lceil\log_2\log_2\left(\frac{\tilde{c}_8\sigma R}{\alpha^2}\sqrt{\frac{\zeta_{\Sigma}t_0\log d}{N}}\right)\bigg\rceil \log 2+\log\left(\frac{\eta^0_G}{ \alpha^2}\right)\right\}\notag\\
    &\times \frac{912}{55} \cdot \frac{1}{1 - \rho}\cdot \left(  {{2\tilde{c}_4\frac{\lambda_{\max}(\Sigma)}{\lambda_{\min}(\Sigma)}\left(1 + \frac{d+\log m}{n}\right)+128{dm}\cdot  \left[\left({{2\log m}}+1\right)/\tilde{c}_{24}+8\right]}}\right)\notag\\
     &+  \frac{456\tilde{c}_4 [1+ (d + \log m)/n]}{55} \frac{\lambda_{\max}(\Sigma)}{\lambda_{\min}(\Sigma)}\log\left(\frac{\eta^0_G}{ \alpha^2}\right)\notag\\
     \overset{\eqref{c_4}}{\leq}\ & 
    \bigg\lceil\log_2\log_2\left(\frac{\tilde{c}_8\sigma R}{\alpha^2}\sqrt{\frac{\zeta_{\Sigma}t_0\log d}{N}}\right)\bigg\rceil\underbrace{\left(\tilde{c}_4\,\frac{\lambda_{\max}(\Sigma)}{\lambda_{\min}(\Sigma)}\frac{456  [1+ (d + \log m)/n]\log 2+55}{55}\right)}_{\triangleq \texttt{term I}}\notag\\
    &+\left\{\bigg\lceil\log_2\log_2\left(\frac{\tilde{c}_8\sigma R}{\alpha^2}\sqrt{\frac{\zeta_{\Sigma}t_0\log d}{N}}\right)\bigg\rceil \log 2+\log\left(\frac{\eta^0_G}{ \alpha^2}\right)\right\}\notag\\
    &\times \frac{912}{55} \cdot \frac{1}{1 - \rho}\cdot\left[  {{4\tilde{c}_4\frac{\lambda_{\max}(\Sigma)}{\lambda_{\min}(\Sigma)}\frac{d+\log m}{n}+128{dm}\cdot  \left(\frac{{2\log m}+1}{\tilde{c}_{24}}+8\right)}}\right]\notag\\
     &+  \underbrace{\frac{456\tilde{c}_4 [1+ (d + \log m)/n]}{55} \frac{\lambda_{\max}(\Sigma)}{\lambda_{\min}(\Sigma)}\log\left(\frac{\eta^0_G}{ \alpha^2}\right).}_{\texttt{term II}}\label{eq:upper_bound_rate_cor_13}
     \end{align}
   Using  $(d + \log m)/n\geq 1$ and $\rho\in(0,1),$  we can bound   \texttt{term I} and   \texttt{term II} as
     \begin{align*}
         \texttt{term I}&\leq   
      (\log 2)\,\tilde{c}_4\,\frac{\lambda_{\max}(\Sigma)}{\lambda_{\min}(\Sigma)}\,\frac{1}{1-\rho}\, \frac{{1001}}{55}\,\frac{d + \log m}{n}\notag 
     \end{align*}
     and 
      \begin{align*}
         \texttt{term II}\leq \tilde{c}_4\,\frac{\lambda_{\max}(\Sigma)}{\lambda_{\min} (\Sigma)}\,\frac{1}{1 - \rho}\,\frac{912}{55}  \frac{d+\log m}{n} \log\left(\frac{\eta^0_G}{ \alpha^2}\right).
     \end{align*}
     Using the above bounds  along with $(d+\log m)/n\leq d  {m}$,   we can further bound the RHS of \eqref{eq:upper_bound_rate_cor_13} as 
  {
\begin{align}
    \left\{\bigg\lceil\log_2\log_2\left(\frac{R\sigma \tilde{c}_8}{\alpha^2}\sqrt{\frac{\zeta_{\Sigma}t_0\log d}{N}}\right)\bigg\rceil \log 2+\log\left(\frac{\eta^0_G}{ \alpha^2}\right)\right\} \cdot\frac{c_{{23}}}{1 - \rho} \frac{\lambda_{\max}(\Sigma)}{\lambda_{\min} (\Sigma)} {dm}\cdot  \left(\frac{{2\log m}+1}{\tilde{c}_{24}}+8\right),
\end{align}
}
where ${c_{{26}}={22222\,\tilde{c}_4}}$. This proves \eqref{T2 expression-alter}, and completes the proof of the corollary.\hfill $\square$

\section{Proof of Theorem \ref{Composite function decrease}}\label{sec:proof_L15}
  At high level the proof is organized in the following two steps. {\bf(Step 1)} Under the following event,
  a tolerance $\eta>0$ and an iteration number $T$ are given  such that 
\begin{equation}\label{Sufficient descrease}
    G(\bT^t)-G(\hat{\bT})\leq \eta,\quad  \forall t\geq T,
\end{equation}
we establish a sufficient decrease of the optimization error  $G(\bT^t)-G(\hat{\bT})$ in the form
\begin{equation}\label{linear convergence upto a tolerence_informal} 
    G(\bT^t)-G(\hat{\bT})
    \leq\kappa^{t-T} (G(\bT^T)-G(\hat{\bT}))+ \text{tolerance},\quad \forall t\geq T,
\end{equation}
for suitable $\kappa\in (0,1)$ and $\text{tolerance}>0$--this is proved in  Lemma~\ref{Geometric convergence} (cf.~Appendix~\ref{sec:intermediate_lemmata}). Then, ({\bf Step 2})  we  divide the iterations $t=0,1,2,\dots,$ into a series of disjoint epochs $[T_k, T_{k+1})$,  with $0=T_0\leq T_1\leq \cdots$, each one with  associated $\eta_k$, with $\eta_0\geq \eta_1\geq \cdots$. The tuples $\{(\eta_k, T_k)\}$ are constructed so that  $G(\bT^t)-G(\hat{\bT})\leq \eta_k,$ for all $t\geq T_k$. This permits to apply recursively \eqref{linear convergence upto a tolerence_informal} with smaller and smaller values of $\eta_k$, till the error    $G(\bT^t)-G(\hat{\bT})$   is driven below a desired  threshold. This second step, formalized in Proposition~\ref{lemma_Th2_Agar} (cf. Appendix~\ref{step 2: Recursive application of Lemma step 2: Recursive application of Lemma}),  leverages \citep[Th. 2]{agarwal2012fast}. 

\subsection{Step 1: Sufficient decrease of the optimization error under \eqref{Sufficient descrease} }\label{sec:intermediate_lemmata}
 
The   error decrease in the form  \eqref{linear convergence upto a tolerence_informal} is formally stated in Lemma~\ref{Geometric convergence} below. It requires two intermediate technical results, namely: (i) Lemma~\ref{proximal iterate norm cone}, which  restricts the (average of the) optimization error $\bDelta^t=\bT^t-\hat{\bT}$ to a  set of ``almost'' sparse directions; 
and (ii) Lemma~\ref{Sufficient decrease}, which establishes 
a curvature property of  $G$ along such   trajectories.



 
\begin{lemma}[On the sparsity of $\Delta_{\textnormal{av}}^t$]\label{proximal iterate norm cone}
Consider Problem~(\ref{Change problem}) under   Assumption~\ref{W}. Further assume that (i) the design matrix $\bX$ satisfies the RSC condition~(\ref{Arsc}) with  $\delta=\mu/2-16\,s\tau> 0$; (ii)  $\lambda$ satisfies~(\ref{distributed lambda}); and (iii) $\gamma$ satisfies (\ref{constrained gamma condition2}).   Let  $\{\bT^t\}$ be the sequence generated  by  Algorithm~\eqref{regularized algorith} with $R$ chosen such that   
\begin{align}\label{c_again}
    R \geq \max\bigg\{\frac{\lambda s}{\delta(1-r)}{\bigg(13+\frac{1}{32}\sqrt{\frac{2\tau s}{\delta}}\bigg)}, \frac{1}{r}\|\T^*\|_1\bigg\},
\end{align}
  for some  $r\in (0,1)$.
Under condition (\ref{Sufficient descrease})   with   parameters $(T,\eta)$, the following holds:  for any $t\geq T$,   
\begin{equation}\label{nu_average sparsity info original}
    \lVert(\Delta^t_{\textnormal{av} })_{\mathcal{S}^c}\rVert_1\leq 3\lVert(\Delta^t_{\textnormal{av}})_{\mathcal{S}}\rVert_1 + 6\lVert(\hat{\nu}_{\textnormal{av}})_\mathcal{S}\rVert_1+2h_{\max}+\min\bigg\{\frac{2\eta}{\lambda},2R\bigg\},
\end{equation}
where [cf.  \eqref{def:h-max}] 
\begin{equation}
    h_{\max}=\frac{ d\gamma}{\lambda(1-\rho)}\bigg(\frac{\max_{i\in [m]}\lVert w_i^{\top} X_i\rVert_{\infty}}{n}+\lambda\bigg)^2.
\end{equation}
\end{lemma}
\begin{proof}
See Appendix~\ref{sec:proof_L}.
\end{proof}
Invoking the RSC condition~(\ref{Arsc}), the next lemma links the objective- and the iterate-errors along  \eqref{nu_average sparsity info original}. 
\begin{lemma}[Curvature along \eqref{nu_average sparsity info original}]\label{Sufficient decrease}
Instate the assumptions of Lemma~\ref{proximal iterate norm cone}. 
Under condition (\ref{Sufficient descrease})   with   parameters $(T,\eta)$, the following holds:  for any $t\geq T$,   
\begin{equation}
      \left\{
      \begin{array}{ll}
       { \bigg(\frac{\mu}{8}-8\tau s\bigg)\lVert\Delta_{\textnormal{av}}^t\rVert^2\leq G(\bT^t)-G(\hat{\bT})+ f(\lVert\bDelta^t_{\perp}\rVert)+\frac{\tau}{4}(v^2+8 h^2_{\max})}, \\
        \bigg(\frac{\mu}{8}-8\tau s\bigg)\lVert\Delta_{\textnormal{av}}^t\rVert^2
     \leq \mathcal{T}_{L_{\gamma}}(\hat{\bT};\bT^t)+f(\lVert\bDelta^t_{\perp}\rVert)+ \frac{\tau}{4}(v^2+8 h^2_{\max}),
      \end{array}
      \right.
      \label{eqn:18}
    \end{equation}
    where     $\mathcal{T}_{L_\gamma}(\hat{\bT};\bT^t)$ is the first order Taylor error of $L$ at $\bT^t$ along the direction $\hat{\bT}-\bT^t$ [cf.~\eqref{def:L-linearization}],   
    \begin{align}\label{f}
          f(\lVert\bDelta^t_{\perp}\rVert)\triangleq \left(\frac{L_{\max}}{2m}-\frac{1-\rho}{2m\gamma} \right)\lVert\bDelta^t_{\perp}\rVert^2, 
    \end{align}   
  and 
\begin{equation}\label{eq:def_v}
    v^2\triangleq144s\lVert\hat{\nu}_{\textnormal{av}}\rVert^2+4\min\bigg\{\frac{2\eta}{\lambda}, 2R\bigg\}^2. 
\end{equation} 
\end{lemma}
\begin{proof}
See Appendix~\ref{sec:proof_L13}.
\end{proof}
Using Lemma~\ref{Sufficient decrease},   we are now ready to formally prove (\ref{linear convergence upto a tolerence_informal}). 
\begin{lemma}[Descent of the objective function]\label{Geometric convergence}
Instate the assumptions of Lemma~\ref{proximal iterate norm cone}, under the stronger condition  $\frac{\mu}{8}-8\tau s>0$ and the additional assumption  that   $\beta$ is chosen so that   
\begin{align}\label{gamma beta condition}
  \beta \leq \frac{  {\gamma}}{\gamma L_{\max}+1-\lambda_{\min}(W)}.
\end{align}
  Under   (\ref{Sufficient descrease})   with   parameters $(T,\eta)$, the following holds: 
\begin{equation}\label{linear convergence upto a tolerence}
    G(\bT^t)-G(\hat{\bT})
    \leq\kappa^{t-T} (G(\bT^T)-G(\hat{\bT}))+\tau (36 s \|\hat{\nu}_{\textnormal{av}}\|^2 +  2 h_{\max}^2+\epsilon^2),\quad \forall t\geq T,
\end{equation}
where \begin{equation}\label{eq:def_kappa}
   \kappa\triangleq 1-   {\beta}\bigg(\frac{\mu}{8}-8\tau s\bigg)\in \left(0,\frac{1}{2}\right)\quad \text{and}\quad \epsilon=\min\bigg\{\frac{2\eta}{\lambda}, 2R\bigg\}.
\end{equation} 
\end{lemma}
\begin{proof}
See Appendix~\ref{sec:proof_L14}.
\end{proof}

Note the structure of the tolerance term in \eqref{linear convergence upto a tolerence}:  $s\|\hat{\nu}_{\textnormal{av}}\|^2$  is of the order of the statistical error;   $h^2_{\max}$ is due to the lack of consensus on the agents trajectories $\theta_i^t$'s, it can be controlled by carefully choosing $\gamma$; and $\epsilon^2$  is a function of the threshold  $\eta$.  In Step 2 below  we show that, since $\kappa<1$, one can  eventually  driven   the error $\epsilon^2$ below the  threshold $\mathcal{O}(s\|\hat{\nu}_{\textnormal{av}}\|^2+h^2_{\max})$. 

\subsection{Step 2: Recursive application of Lemma~\ref{Geometric convergence}}\label{step 2: Recursive application of Lemma step 2: Recursive application of Lemma}
As anticipated, the key idea is to  divide the iterations $t=0,1,2,\dots,$ into a series of disjoint epochs $[T_k, T_{k+1})$,  with $T_k\leq  T_{k+1}$, each one with  associated $\eta_k$, such that (i) $G(\bT^t)-G(\hat{\bT})\leq \eta_k,$ for all $t\geq T_k$; and (ii) $\eta_0\geq \eta_1\geq \cdots$.  This permits to apply recursively Lemma~\ref{Geometric convergence} with smaller and smaller values of $\eta_k$, till the error    $G(\bT^t)-G(\hat{\bT})$   is driven below the threshold $ 4 \tau (36 s \|\hat{\nu}_{\text{av}}\|^2 +  2h_{\max}^2)$. This construction follows the same argument as in the proof of  \citep[Th. 2]{agarwal2012fast}   with minor adjustments (Lemma 4 therein is replaced with our Lemma~\ref{Geometric convergence}) and thus is omitted. 

\begin{proposition}\!\textnormal{\bf (\citep[Th. 2]{agarwal2012fast})}\label{lemma_Th2_Agar}  Instate the setting of Lemma~\ref{Geometric convergence}. Further assume, \begin{equation}\label{eq:upper_bound_R}
    R\leq \frac{\lambda}{32\tau}.
\end{equation}  
Then, there holds 
 \begin{equation*}
     G(\bT^t)-G(\hat{\bT})\leq \alpha^2,
 \end{equation*}
  for any tolerance  $\alpha^2$ such that 
 \begin{equation}\label{alpha_stronger}
    {\min\bigg\{\frac{R\lambda}{4},\eta_G^0\bigg\}}\geq\alpha^2\geq 4\tau(36 s \|\hat{\nu}_{\textnormal{av}}\|^2 +  2 h_{\max}^2),
  \end{equation}
and for all 
  \begin{equation}\label{original rate with general beta}
   t   \geq  \bigg\lceil\log_2\log_2\left(\frac{R\lambda}{\alpha^2}\right)\bigg\rceil\left(1+\frac{\log 2}{\log 1/\kappa}\right)+\frac{\log(\eta_G^0/\alpha^2)}{\log 1/\kappa},
  \end{equation}where $\eta_G^0=G(\bT^0)-G(\hat{\bT})$.
\end{proposition}





Equipped with  Proposition~\ref{lemma_Th2_Agar}, we can now complete the proof of Theorem~\ref{Composite function decrease}. It  remains to show  the following facts: 

$\bullet$  \textbf{Fact 1:} The lower bound condition on $R$ as in \eqref{R condidion} is more stringent than that in \eqref{c_again}, under a proper choice of $r\in (0,1)$; and the interval in \eqref{R condidion} is nonempty;  

$\bullet$  \textbf{Fact 2:} The range of $\alpha$ in (\ref{determinstic alpha_appendix}) is contained in  that of \eqref{alpha_stronger}; and the interval (\ref{determinstic alpha_appendix}) is nonempty;

$\bullet$  \textbf{Fact 3:} \eqref{eq:rate_estimate_appendix} is sufficient for \eqref{original rate with general beta}.\smallskip 

\noindent We prove these facts next.

 $\bullet$ \textbf{Fact 1:} Choosing $r=1/2$, the lower bound condition  on $R$   in   \eqref{c_again}  reads
\begin{equation}\label{eq:R-upper-lower-bound}
 R\geq \max\bigg\{\frac{2\lambda s}{\mu/2-16s\tau}\bigg(13+\frac{1}{32}\sqrt{\frac{2\tau s}{\mu/2-16s\tau}}\bigg), 2\|\T^*\|_1\bigg\},
\end{equation}
Recalling $\mu \geq \tilde{c}_{10}s\tau= 1824s \tau,$
the following holds for the lower bound in \eqref{eq:R-upper-lower-bound}:
\begin{equation*}
    \frac{2\lambda s}{\mu/2-16s\tau}\bigg(13+\frac{1}{32}\sqrt{\frac{2\tau s}{\mu/2-16s\tau}}\bigg)\leq\frac{56\lambda s}{\mu-32s\tau}.
\end{equation*}
Therefore, 
\begin{equation*}
 \max\bigg\{\frac{2\lambda s}{\mu/2-16s\tau}\bigg(13+\frac{1}{32}\sqrt{\frac{2\tau s}{\mu/2-16s\tau}}\bigg), 2\|\T^*\|_1\bigg\}\leq  \max\bigg\{\frac{56\lambda s}{\mu-32s\tau}, 2\|\T^*\|_1\bigg\}, 
\end{equation*}
which proves the desired implication. 

Finally, notice that the interval in \eqref{R condidion} is non-empty. This is a consequence of (i) the fact    $$\frac{56\lambda s}{\mu-32s\tau}\leq\frac{\lambda}{32\tau},$$
due to $\mu \geq \tilde{c}_{10}s\tau= 1824s \tau$; and (ii)    
the condition  $\lambda   \geq64 \tau \| \T^*\|_1$, due to  \eqref{lambda condition}.

\smallskip

$\bullet$  \textbf{Fact 2:}  
Using   the condition on $\gamma$ as in \eqref{constrained gamma condition2}, we have
\begin{align*}
    4\tau(36 s \|\hat{\nu}_{\textnormal{av}}\|^2 +  2 h_{\max}^2)
    \overset{\eqref{refined bound}}{\leq}\ & 4\tau\bigg(36 s \frac{\sum_{i=1}^m\lVert\hat{\nu}_i\rVert^2}{m} +    \frac{2\lambda^2 s^2}{ {128}^2\,(\mu/2-16\,s\tau)^2}\bigg)\notag\\
      \overset{\mu\geq \tilde{c}_{10} s\tau}{\leq}& 
   4s\tau\bigg(\frac{36}{m}\sum\limits_{i=1}^m\lVert\hat{\nu}_i\rVert^2+{\frac{\lambda^2s}{1976\mu^2}}\bigg).
\end{align*}
Therefore, the range of $\alpha$ in (\ref{determinstic alpha_appendix}) is included in that of \eqref{alpha_stronger}.

It remains to show that the range of $\alpha^2$ in \eqref{determinstic alpha_appendix} is nonempty, which is  a consequence of the following chain of inequalities.
\begin{align}\label{nonempty 1}
    &\quad\quad4\tau\bigg(36 s \frac{\lVert\hat{\bnu}\rVert^2}{m} +   \frac{\lambda^2 s^2}{1976\mu^2}\bigg)\notag\\
   & \overset{\overset{\text{Th.6}}{\xi=\tau \text{ (Lm. \ref{global-ARE-determin})}}}{\leq}144\tau s\left( \frac{9\lambda^2s}{(\mu/2-16s\tau)^2}+\frac{2\tau}{\mu/2-16s\tau}\,\underset{=\texttt{Term I}^2\text{ [see~\eqref{deterministic solution equivalence}]}}{\underbrace{\frac{ d^2\gamma^2(\max_{1\leq i\leq m}\lVert w_i^{\top} X_i\rVert_{\infty}+\lambda n)^4}{\lambda^2n^4(1-\rho)^2}}}\right. \notag\\&\qquad\quad \left.+\frac{4}{\mu/2-16s\tau}\,  \underset{=\texttt{Term II}^2\,\text{ [see~\eqref{deterministic solution equivalence}]}}{\underbrace{\frac{d\gamma(\max_{1\leq i\leq m}\lVert w_i^{\top} X_i\rVert_{\infty}+\lambda n)^2}{n^2[2(1-\rho)-4L_{\max}\gamma-(\mu/2-16s\tau)\gamma]}}} 
    +   \frac{\lambda^2 s}{36\cdot 1976(\mu-32s\tau)^2}\right)\notag\\
   &\quad \overset{\eqref{refined bound}}{\leq}\ \ \  144\tau s\bigg( \frac{9\lambda^2s}{(\mu/2-16s\tau)^2}+\frac{\tau s }{\mu/2-16s\tau}\,\frac{\lambda^2 s}{8192(\mu/2-16s\tau)^2}+\frac{\lambda^2s}{64(\mu/2-16s\tau)^2}\notag\\
  &\quad\quad+   \frac{\lambda^2 s}{36\cdot 7904(\mu/2-16s\tau)^2}\bigg)\notag\\
   & \quad\overset{(a)}{\leq}\ \ \ \  144\tau s\bigg( \frac{9\lambda^2s}{(\mu/2-16s\tau)^2}+\frac{1}{896}\frac{\lambda^2 s}{8192(\mu/2-16s\tau)^2}+\frac{\lambda^2s}{64(\mu/2-16s\tau)^2}\notag\\
    &\quad\quad+   \frac{\lambda^2 s}{36\cdot 7904(\mu/2-16s\tau)^2}\bigg)\notag\\
     &\quad<\frac{144\tau s}{(\mu/2-16s\tau)^2}\cdot 10\lambda\cdot \lambda \,s\notag\\
    &\quad\overset{\eqref{R condidion}}{\leq} 
    \frac{1440 \tau s}{28(\mu/2-16s\tau)}\,\lambda R \overset{(b)}{<}\frac{\lambda R}{17}<\frac{\lambda R}{4},
\end{align}
 where (a) and (b) follow from  $\mu/2-16s\tau\geq896s\tau$, due to      $\mu \geq \tilde{c}_{10}s\tau$, with $\tilde{c}_{10}=1824$. This together with \eqref{eq:cond_eta_G}  shows that  the range of $\alpha^2$ in \eqref{determinstic alpha_appendix} is non-empty.
 \smallskip
 
$\bullet$ \textbf{Fact 3:}  We obtain \eqref{eq:rate_estimate_appendix} from \eqref{original rate with general beta} by upper bounding the right hand side of \eqref{original rate with general beta}. To this end, we first lower bound   $\log ({1}/{\kappa})$ as: 
\begin{align}\label{upper_kappa}
    \log \bigg(\frac{1}{\kappa}\bigg) \overset{\eqref{beta condidion},\eqref{eq:def_kappa}}{=}\log\left(\frac{1}{1-\frac{\gamma(\mu/8-8\tau s)}{\gamma L_{\max}+1-\lambda_{\min}(W)}}\right)
   \overset{\eqref{connectivity number}}{\geq} \log\left(\frac{1}{1-\frac{\gamma(\mu/8-8\tau s)}{\gamma L_{\max}+1+\rho}}\right)\geq  \frac{\gamma(\mu/8-8\tau s)}{\gamma L_{\max}+1+\rho}.
\end{align}
Using \eqref{upper_kappa} in \eqref{original rate with general beta} and using $\eta_G^0\geq \alpha^2$ [due to (\ref{alpha_stronger})], we can upper bound the right hand side of \eqref{original rate with general beta}  as
 \begin{align*}\label{T_1}
     & \bigg\lceil\log_2\log_2\left(\frac{R\lambda}{\alpha^2}\right)\bigg\rceil\bigg(1+\frac{(\gamma L_{\max}+1+\rho)\log 2}{\gamma(\mu/8-8\tau s)} \bigg) + \frac{(\gamma L_{\max}+1+\rho)}{\gamma (\mu/8-8\tau s)} \log\left(\frac{\eta^0_G}{ \alpha^2}\right),\notag
 \end{align*}
which  proves \eqref{eq:rate_estimate_appendix}.

 $\hfill\square$


\section{Proofs of auxiliary Lemmata in  Sec.~\ref{sec:proof_L15}}\label{Proofs of auxiliary Lemmas in  Sec.7.}

\subsection{Proof of Lemma \ref{proximal iterate norm cone}}\label{sec:proof_L}
Recalling the definitions of $\bDelta^t$, $\bnu^t$, and $\widehat{\bnu}$ as given in (\ref{eq:Delta_def}), (\ref{eq_nu_t_def}) and (\ref{eq:nu_hat_def}), respectively,   
we have $\bDelta^t=\bT^t-\hat{\bT}=\bnu^t - \hat{\bnu}.$ Therefore,   $\Delta_{\text{av}}^t=\nu_{\text{av}}^t - \hat{\nu}_{\text{av}}.$  We can then bound the desired quantity $ \lVert(\Delta_{\text{av}}^t)_{\mathcal{S}^c}\rVert_1$ as  
\begin{equation}\label{eq:bound_opt_error_1}\lVert(\Delta_{\text{av}}^t)_{\mathcal{S}^c}\rVert_1\leq  \lVert(\nu_{\text{av}}^t)_{\mathcal{S}^c}\rVert_1+\lVert(\hat{\nu}_{\text{av}})_{\mathcal{S}^c}\rVert_1.
\end{equation}
   We prove below the following upper bounds for $\lVert(\nu_{\text{av}}^t)_{\mathcal{S}^c}\rVert_1$ and $\lVert(\hat{\nu}_{\text{av}})_{\mathcal{S}^c}\rVert_1$
   \begin{equation}\label{eq:bound_two_terms_opt_error}
       \left\{\begin{array}{l}
\lVert(\nu_{\text{av}}^{t})_{\mathcal{S}^{c}}\rVert_{1}\leq 3\lVert(\nu_{\text{av}}^{t})_{\mathcal{S}}\rVert_{1}+h(\gamma,\lVert\bnu_{\perp}^{t}\rVert)+\min\bigg\{\frac{2\eta}{\lambda},2R\bigg\},\medskip\\
\lVert(\hat{\nu}_{\text{av}})_{\mathcal{S}^c}\rVert_1\leq 3 \lVert(\hat{\nu}_{\text{av}})_{\mathcal{S}}\rVert_1+ h(\gamma,\lVert\hat{\bnu}_{\perp}\rVert).
\end{array}\right.
   \end{equation}
Using \eqref{eq:bound_two_terms_opt_error} in (\ref{eq:bound_opt_error_1}) and the triangle inequality  yields the desired result
\begin{align*}
   \lVert(\Delta_{\text{av}}^t)_{\mathcal{S}^c}\rVert_1& \ \leq  
  3(\lVert(\Delta_{\text{av}}^t)_{\mathcal{S}}\rVert_1+2\lVert(\hat{\nu}_{\text{av}})_{\mathcal{S}}\rVert_1)+ h(\gamma,\lVert\bnu_{\perp}^t\rVert)+ h(\gamma,\lVert\hat{\bnu}_{\perp}\rVert)+\min\bigg\{\frac{2\eta}{\lambda},2R\bigg\}\notag\\
   &  \overset{\eqref{def:h-max}}{\leq}  3\lVert(\Delta_{\text{av}}^t)_\mathcal{S}\rVert_1 + 6\lVert(\hat{\nu}_{\text{av}})_\mathcal{S}\rVert_1 + 2 h_{\max} +\min\bigg\{\frac{2\eta}{\lambda},2R\bigg\}.
\end{align*}
We prove next \eqref{eq:bound_two_terms_opt_error}.    From the optimality of $\hat{\bT}$
along with \eqref{Sufficient descrease}, we deduce  \begin{align}
     G(\bT^t)-G(1_m\otimes\T^*)
     &\leq \eta,\quad  \forall t\geq T\label{S1}.
\end{align}
Hence, for any $t\geq T$, there holds
\begin{align}\label{F10 condition}
   & \frac{1}{2N}\sum\limits_{i=1}^m\lVert X_i\theta_i^t-y_i\rVert^2+\frac{1}{2m\gamma}\lVert1_m\otimes\T^{*}+\bnu^t\rVert_{V}^2+\frac{\lambda}{m}\lVert 1_m \otimes\T^{*}+\bnu^t\rVert_1\notag\\
   \leq& \frac{1}{2N}\lVert \bX\T^{*}-y\rVert^2+\frac{1}{2m\gamma}\lVert1_m\otimes\T^{*}\rVert_{V}^2+\frac{\lambda}{m}\lVert1_m\otimes\T^{*}\rVert_1+\eta.
\end{align}
Subtracting   $\sum\limits_{i=1}^m\langle\frac{1}{N}X_i^{\top}(X_i\T^{*}-y_i), \nu_i^t\rangle$ from both sides and rearranging terms, we obtain 
\begin{align}\label{238}
 &\underset{\texttt{Term I}}{\underbrace{-\sum\limits_{i=1}^m\bigg\langle\frac{1}{N}X_i^{\top}(X_i\T^{*}-y_i), \nu_i^t\bigg\rangle+\frac{1}{2m\gamma}\lVert1_m\otimes\T^{*}\rVert_{V}^2-\frac{1}{2m\gamma}\lVert1_m\otimes\T^{*}+\bnu^t\rVert_{V}^2}}+\eta\notag\\ \geq&   \frac{1}{2N}\sum\limits_{i=1}^m\lVert X_i(\T^{*}+\nu_i^t)-y_i\rVert^2 -\frac{1}{2N}\lVert \bX\T^{*}-y\rVert^2-\sum\limits_{i=1}^m\bigg\langle\frac{1}{N}X_i^{\top}(X_i\T^{*}-y_i), \nu_i^t\bigg\rangle \notag\\
  & +\frac{\lambda}{m}(\lVert1_m\otimes\T^{*}+\bnu^t\rVert_1-\lVert1_m\otimes\T^{*}\rVert_1)\notag\\
  \geq& \frac{\lambda}{m} \underset{\texttt{Term II}}{\underbrace{ (\lVert1_m\otimes\T^{*}+\bnu^t\rVert_1-\lVert1_m\otimes\T^{*}\rVert_1)}},
\end{align}
where the last inequality follows from convexity of 
  $\sum\limits_{i=1}^m\lVert X_i\theta_i-y_i\rVert^2/(2N)$. 

We proceed upper (resp. lower) bounding \texttt{Term I} (resp. \texttt{Term II}). We have 
\begin{align}\label{decompose}
\begin{split}
       \texttt{Term I}&\ =
   \frac{1}{N} \mathbf{w}^{\top}\bX \nu^t_{\text{av}}+\frac{1}{N}\sum\limits_{i=1}^m w_i^{\top}X_i \nu^t_{\perp i}-\frac{1}{2m\gamma}\lVert\bnu^t_{\perp}\rVert_{V}^2\\
    &\overset{(\ref{distributed lambda})}{\leq} \frac{\lambda}{2}\lVert\nu^t_{\text{av}}\rVert_1+\frac{1}{N}\max_{i\in [m]}\lVert X_i^{\top} w_i\rVert_{\infty}\lVert\bnu^t_{\perp }\rVert_1-\frac{1}{2m\gamma}\lVert\bnu_{\perp}^t\rVert_{V}^2.
\end{split}
\end{align}
To lower bound \texttt{Term II}
we decompose $\T^{*}+\nu_i^t$ as
$   \T^{*}+\nu_i^t = \T^{*}_{\mathcal{S}}+ \T^{*}_{\mathcal{S}^c}+(\nu_i^t)_{\mathcal{S}}+(\nu_i^t
    )_{\mathcal{S}^c}.$
Then, invoking  the decomposibility of the regularizer, we can write: for all $i,$   
\begin{align}\label{triangle inequality and l1 norm decomposibility}
 \lVert\T^*+\nu_i^t\rVert_1-\lVert\T^*\rVert_1 \geq 
    (\lVert(\nu^t_{\text{av}})_{\mathcal{S}^c}\rVert_1-\lVert(\nu^t_{\text{av}})_{\mathcal{S}}\rVert_1)-\lVert\nu_{\perp i}^t\rVert_1.
\end{align}
Using \eqref{decompose} and  \eqref{triangle inequality and l1 norm decomposibility}      in \eqref{238}  yields
\begin{equation}\label{intermedia norm result}
\lVert(\nu_{\text{av}}^{t})_{\mathcal{S}^{c}}\rVert_{1}\leq 3\lVert(\nu_{\text{av}}^{t})_{\mathcal{S}}\rVert_{1}+h(\gamma,\lVert\bnu_{\perp}^{t}\rVert)+\frac{2\eta}{\lambda}.
\end{equation}
On the other hand,  since $\|\T^t_i\|_1\leq R$   and $\|\theta^*\|_1<R$,  we have
 \begin{equation}\label{intermedia norm result 2}
    \lVert(\nu^t_{\text{av}})_{\mathcal{S}^c}\rVert_1 \leq \lVert\nu^t_{\text{av}}\rVert_1\leq \lVert\T^*\rVert_1+ \lVert\T_{\text{av}}^t\rVert_1< R + R=2R.
\end{equation}
Using \eqref{intermedia norm result 2}, we can then  strengthen  (\ref{intermedia norm result}) as the first inequality in 
 \eqref{eq:bound_two_terms_opt_error}.

 The proof of the second inequality in \eqref{eq:bound_two_terms_opt_error} follows the same steps   and uses the  fact that  $\|\hat{\theta}_i\|_1\leq R$, for all $i\in [m]$ (Lemma \ref{solution eq}).   $\hfill \square$

\subsection{Proof of Lemma \ref{Sufficient decrease}}\label{sec:proof_L13}
To bound the (average component of the) optimization error in terms of the function optimality gap we leverage the curvature property of $L_\gamma$ [under the RSC condition (\ref{Arsc})] along the trajectory of the algorithm. We explicitly use the fact that the  trajectory lies in the set  described by \eqref{nu_average sparsity info original} [cf.~Lemma \ref{proximal iterate norm cone}].

 Recalling that $G(\bT)=L_\gamma(\bT)+\frac{\lambda}{m}\lVert\bT\rVert_1,$ by the optimality of $\hat{\bT}$, it follows
\begin{equation}\label{convexity of l1}
  \langle \bDelta^t, \nabla L_\gamma(\hat{\bT})\rangle+ \frac{\lambda}{m}\lVert\bT^t\rVert_1-\frac{\lambda}{m}\lVert\hat{\bT}\rVert_1\geq 0.
\end{equation} 
We can then write 
\begin{align}\label{90a}
     G(\bT^t)-G(\hat{\bT})
  &\ \overset{\eqref{convexity of l1}}{\geq} \mathcal{T}_{L_\gamma}({\bT^t};\hat{\bT})\notag\\
  &\ \overset{{\eqref{def:L-linearization}}}{\geq} \frac{1}{4}\frac{\lVert \bX\Delta^t_{\text{av}}\rVert^2}{N}-\bigg(\frac{L_{\max}}{2m}-\frac{1-\rho}{2m\gamma}\bigg)\|\bDelta^t_{\perp}\|^2\notag\\
   &\overset{{\text{RSC}~ \eqref{Arsc}}}{\geq}\frac{1}{4}\bigg(\frac{\mu}{2}\lVert\Delta^t_{\text{av}}\rVert^2-\frac{\tau}{2}\lVert\Delta^t_{\text{av}}\rVert_1^2\bigg)-\bigg(\frac{L_{\max}}{2m}-\frac{1-\rho}{2m\gamma}\bigg)\|\bDelta^t_{\perp}\|^2\notag\\
   &\ \ \ \overset{\text{(a)}}{\geq}\frac{1}{4}\bigg(\frac{\mu}{2}\lVert\Delta^t_{\text{av}}\rVert^2-\frac{\tau}{2}(64s \lVert\Delta^t_{\text{av}}\rVert^2+2v^2+16 h_{\max}^2)\bigg)-\bigg(\frac{L_{\max}}{2m}-\frac{1-\rho}{2m\gamma}\bigg)\|\bDelta^t_{\perp}\|^2,
\end{align}
where in  (a) we used 
\begin{align*}
     \lVert\Delta_{\textnormal{av}}^t\rVert_1^2& \overset{\eqref{nu_average sparsity info original}} {\leq}\left(4\lVert(\Delta^t_{\textnormal{av}})_{\mathcal{S}}\rVert_1 + 6\lVert(\hat{\nu}_{\textnormal{av}})_\mathcal{S}\rVert_1+2h_{\max}+\min\bigg\{\frac{2\eta}{\lambda},2R\bigg\}\right)^2\notag\\
     &\ \leq4(4\lVert(\Delta^t_{\textnormal{av}})_{\mathcal{S}}\rVert_1)^2 + 4(6\lVert(\hat{\nu}_{\textnormal{av}})_\mathcal{S}\rVert_1)^2+4(2h_{\max})^2+4\bigg(\min\bigg\{\frac{2\eta}{\lambda},2R\bigg\}\bigg)^2\notag\\
     &\ \leq 64s \lVert\Delta_{\textnormal{av}}^t\rVert^2+2v^2+16 h_{\max}^2,
\end{align*}
with $v^2=144s\lVert\hat{\nu}_{\textnormal{av}}\rVert^2+4\min\bigg\{\frac{2\eta}{\lambda}, 2R\bigg\}^2.$ 

Reorganizing the terms in \eqref{90a} yields the first inequality in \eqref{eqn:18}. 

Similar arguments apply to derive the second inequality in \eqref{eqn:18} by noticing that, for quadratic $L_\gamma,$ we have
$\mathcal{T}_{L_\gamma}(\hat{\bT};\bT^t)=\mathcal{T}_{L_\gamma}(\bT^t;\hat{\bT}).$ 
This concludes the proof. $\hfill \square$

\subsection{Proof of Lemma \ref{Geometric convergence}}\label{sec:proof_L14}

The proof follows  descent arguments (see,  \eg, \citep{Nesterov2007GradientMF}), suitably  coupled with the curvature property established in Lemma~\ref{Sufficient decrease} to achieve contraction up to a controllable tolerance. 

By definition of  $\bT^{t+1}$ in \eqref{regularized algorith},   we have $G_t(\bT^{t+1})\leq G_t(\bT)$, for all feasible $\bT$. Recalling that   $\hat{\bT}$ is feasible (Lemma \ref{solution eq}),   $\bT_{\omega}\triangleq \omega\hat{\bT}+(1-\omega)\bT^t$ is feasible as well,  for any $\omega \in (0,1)$.
Therefore,
\begin{align}\label{F}
        &G_{t}(\bT^{t+1})\notag\\
         \leq\ &  G_{t}(\bT_{\omega}) 
   = (1-\omega)L_\gamma(\bT^t)+\omega L_\gamma(\hat{\bT})-\omega \mathcal{T}_{L_\gamma}(\hat{\bT}; \bT^t) +\frac{\omega^2}{2  {\beta m}}\lVert\bDelta^t\rVert^2+\frac{\lambda}{m}\lVert\bT_{\omega}\rVert_1\notag\\
   \stackrel{\rm  \eqref{eqn:18}}{\leq}& 
   (1-\omega) G(\bT^t)+\omega G(\hat{\bT})+\omega f(\lVert\bDelta^t_{\perp}\rVert)+\omega\frac{\tau}{4}(v^2+8 h^2_{\max})+\frac{\omega^2}{2  {\beta m}}\lVert\bDelta^t\rVert^2-\omega \bigg(\frac{\mu}{8}-8\tau s\bigg)\lVert\Delta_{\textnormal{av}}^t\rVert^2.
\end{align}
We proceed to relate $G(\bT^{t+1})$ with $G_{t}(\bT^{t+1}).$ 
\begin{align}\label{Taylor}
          &G(\bT^{t+1})\notag\\
   = \  \ & G_t(\bT^{t+1}) -\frac{1}{2  {\beta m}}\lVert\bT^{t+1}-\bT^{t}\rVert^2   + \underbrace{L_\gamma(\bT^{t+1})-L_\gamma(\bT^{t})-\langle\nabla L_\gamma(\bT^t),\bT^{t+1}-\bT^{t}\rangle}_{=\frac{1}{2N} \sum_{i = 1}^m \| X_i (\T_i^{t+1} - \T_i^t)\|^2+ \frac{1}{2 m \gamma} \| \bT^{t+1} - \bT^t\|_V^2}\notag\\
     \stackrel{\eqref{F}}{\leq} &  G(\bT^t)-\omega (G(\bT^t)-G(\hat{\bT}))+\omega f(\lVert\bDelta^t_{\perp}\rVert)+\omega\frac{\tau}{4}(v^2+8 h^2_{\max})+\frac{\omega^2}{2  {\beta m}}\lVert\bDelta^t\rVert^2-\omega \bigg(\frac{\mu}{8}-8\tau s\bigg)\lVert\Delta_{\textnormal{av}}^t\rVert^2\notag\\
    &+ \frac{1}{2N} \sum_{i = 1}^m \| X_i (\T_i^{t+1} - \T_i^t)\|^2  + \frac{1}{2 m \gamma} \| \bT^{t+1} - \bT^t\|_V^2-\frac{1}{2  {\beta m}}\lVert\bT^{t+1}-\bT^{t}\rVert^2.
\end{align}
Subtracting $G(\hat{\bT})$ from both sides of the above inequality and denoting the function gap as $\eta_G^{t}=G(\bT^t)-G(\hat{\bT}),$ we have 
\begin{align}\label{eta F}
        \eta_G^{t+1}\leq & \, (1-\omega )\,\eta_G^t+\omega  f(\lVert\bDelta^t_{\perp}\rVert)+\omega\frac{\tau}{4}(v^2+8 h^2_{\max})+\frac{\omega^2}{2  {\beta m}}\lVert\bDelta^t\rVert^2-\omega \bigg(\frac{\mu}{8}-8\tau s\bigg)\lVert\Delta_{\textnormal{av}}^t\rVert^2\notag\\
        & + \frac{1}{2N} \sum_{i = 1}^m \| X_i (\T_i^{t+1} - \T_i^t)\|^2  + \frac{1}{2 m \gamma} \| \bT^{t+1} - \bT^t\|_V^2-\frac{1}{2  {\beta m}}\lVert\bT^{t+1}-\bT^{t}\rVert^2\notag\\
        \leq & \, (1-\omega)\eta_G^t+\omega  f(\lVert\bDelta^t_{\perp}\rVert)+\omega\frac{\tau}{4}(v^2+8 h^2_{\max})+\frac{\omega^2}{  {\beta m}}\lVert\bDelta^t_\perp\rVert^2\notag\\
    &+\underset{\leq 0, \text{ for  } 0\leq\, \omega\,\leq   {\beta} \big(\frac{\mu}{8}-8\tau s\big)}{\underbrace{\omega\,\bigg(  {\frac{\omega\,}{\beta}}- \bigg(\frac{\mu}{8}-8\tau s\bigg)\bigg)\lVert\Delta_{\textnormal{av}}^t\rVert^2}}+ \underbrace{\frac{1}{2}\bigg(\frac{L_{\max}}{m}+\frac{1-\lambda_{\min} (W)}{m\gamma}-  {\frac{1}{\beta m}}\bigg)}_{\text{$\leq 0$ [condition on $\beta$ in \eqref{gamma beta condition}]}}\lVert\bT^{t+1}-\bT^{t}\rVert^2\notag\\
    \stackrel{(a)}{\leq}   &\,\underset{\texttt{Term I}}{\underbrace{(1-\omega)\eta_G^t}}+\underset{\texttt{Term II}}{\underbrace{\omega f(\lVert\bDelta^t_{\perp}\rVert)+  {\frac{\omega^2}{\beta m}}\lVert\bDelta^t_{\perp}\rVert^2}}+\underset{\texttt{Term III}}{\underbrace{\omega\frac{\tau}{4}(v^2+8 h^2_{\max})}},
\end{align}
where (a) holds under the condition $0\leq\, \omega\,\leq  {\beta} \big(\frac{\mu}{8}-8\tau s\big)$.  Note that $\mu/8-8\tau s>0$ by assumption; hence the interval for $\omega$   is non-empty. 
{Furthermore, $\omega\in (0,1/2)$, due to   } 
\begin{align}\label{rate range}
      {\beta} \left(\frac{\mu}{8}-8\tau s\right){\overset{\eqref{gamma beta condition}}{\leq} }\frac{\gamma}{\gamma L_{\max}+1-\lambda_{\min}(W)} \cdot \left(\frac{\mu}{8}-8\tau s\right) {\overset{\text{(b)}}{<} \frac{1}{2}},
\end{align}
{where in (b)  we used the following lower bound for $(1-\lambda_{\min})/\gamma$: 
\begin{align*}
    &\frac{1-\lambda_{\min}(W)}{\gamma}\notag\\
    \overset{\eqref{constrained gamma condition2}}{\geq}& \frac{1-\lambda_{\min}(W)}{1-\rho}\left(2L_{\max}+\frac{\mu}{2}-16s\tau+\frac{128d}{s}\left(\frac{\mu}{2}-16s\tau\right)\left(\frac{\max_{  i\in [m]}\lVert w_i^{\top} X_i\rVert_{\infty}}{\lambda n}+2\sqrt{m}\right)^2\right)\notag\\
    \overset{\eqref{connectivity number}}{\geq}&\left(2L_{\max}+\frac{\mu}{2}-16s\tau+\frac{128d}{s}\left(\frac{\mu}{2}-16s\tau\right)\left(\frac{\max_{  i\in [m]}\lVert w_i^{\top} X_i\rVert_{\infty}}{\lambda n}+2\sqrt{m}\right)^2\right) 
    \geq\  2\left(\frac{\mu}{8}-8\tau s\right).
\end{align*}}

In  (\ref{eta F}), \texttt{Term I} captures the  geometric decrease of the objective error, for any $\omega<1$;   \texttt{Term II} is due to consensus errors and it is controllable  by choosing a sufficiently small network regularizer $\gamma;$ finally, \texttt{Term III} is due to the lack of strong convexity, determining  a nonzero tolerance on the achievable  objective error. 

We choose $\omega$ to minimize the contraction factor in \texttt{Term I}, resulting in   \begin{equation}\label{eq:alpha_choice}
    \omega=  {\beta}\bigg(\frac{\mu}{8}-8\tau s\bigg).
\end{equation} 
Under this choice we can bound \texttt{Term I}--\texttt{Term III} as follows.

\noindent $\bullet$ \texttt{Term I:} 
\begin{equation}\label{objective rate}
       \texttt{Term I}=\underbrace{\bigg(1-   {\beta}\bigg(\frac{\mu}{8}-8\tau s\bigg)\bigg)}_{\kappa}\eta_G^t.
\end{equation}
{Note that $\kappa\in(0,1/2),$ due to \eqref{rate range}.}

\noindent $\bullet$ \texttt{Term II:}
Using the upper bound of $\gamma$ in \eqref{equ-sol-gamma}, we can bound   $f(\lVert\bDelta^t_{\perp}\rVert)$ [cf.~\eqref{f}] as
\begin{align}\label{f bound}
    f(\lVert\bDelta^t_{\perp}\rVert) &\leq
    -\bigg(\frac{L_{\max}}{2m}+\frac{\mu-32\,s\tau}{4m}+\frac{32d(\mu-32\,s\tau)}{sm}(\max_{  i\in [m]}\lVert w_i^{\top} X_i\rVert_{\infty}/(\lambda n)+2\sqrt{m})^2\bigg)\lVert\bDelta^t_{\perp}\rVert^2.
\end{align}
Therefore,  {
 \begin{align}\label{f effect}
 \texttt{Term II}
& \overset{\eqref{eq:alpha_choice},\eqref{f bound}}{\leq} -  {\beta}\bigg(\frac{\mu}{8}-8\tau s\bigg) \frac{1}{m}\left(\frac{\mu}{4}-8s\tau\right)\lVert\bDelta^t_{\perp}\rVert^2+  {\frac{\beta}{m}}\bigg(\frac{\mu}{8}-8\tau s\bigg)^2\lVert\bDelta^t_{\perp}\rVert^2\notag\\
& \quad\,\,\, \leq  -  {\frac{\beta}{m}}\,\bigg(\frac{\mu}{8}-8\tau s\bigg)^2\lVert\bDelta^t_{\perp}\rVert^2\leq 0.
\end{align}}
\noindent $\bullet$ \texttt{Term III:} 
\begin{equation}\label{error ball original}
    \texttt{Term III}=  {\beta}\,\bigg(\frac{\mu}{8}-8\tau s\bigg)\,\tau\,\bigg(36s\lVert\hat{\nu}_{\textnormal{av}}\rVert_2^2+2 h^2_{\max}+\min\bigg\{\frac{2\eta}{\lambda}, 2R\bigg\}^2\bigg).
\end{equation}
Using  \eqref{objective rate}, \eqref{f effect}, and \eqref{error ball original} in \eqref{eta F}, we finally obtain: for all $t\geq T,$  
\begin{align*}
\begin{split}
        \eta_G^{t+1}&\leq\kappa\,\eta_G^{t}+  {\beta}\bigg(\frac{\mu}{8}-8\tau s\bigg)\,\tau\,\bigg(36s\lVert\hat{\nu}_{\textnormal{av}}\rVert^2+2 h^2_{\max}+\min\bigg\{\frac{2\eta}{\lambda}, 2R\bigg\}^2\bigg)\\
    & \leq 
    \kappa^{t-T} \eta_G^{T} + \tau\,\bigg(36s\lVert\hat{\nu}_{\textnormal{av}}\rVert^2+2 h^2_{\max}+\min\bigg\{\frac{2\eta}{\lambda}, 2R\bigg\}^2\bigg).
\end{split}
\end{align*}
 This completes the proof. $\hfill \square$

\bibliography{references}

\begin{thebibliography}{62}
\providecommand{\natexlab}[1]{#1}
\providecommand{\url}[1]{\texttt{#1}}
\expandafter\ifx\csname urlstyle\endcsname\relax
  \providecommand{\doi}[1]{doi: #1}\else
  \providecommand{\doi}{doi: \begingroup \urlstyle{rm}\Url}\fi

\bibitem[Agarwal et~al.(2012)Agarwal, Negahban, and
  Wainwright]{agarwal2012fast}
A.~Agarwal, S.~Negahban, and M.~J. Wainwright.
\newblock Fast global convergence of gradient methods for high-dimensional
  statistical recovery.
\newblock \emph{The Annals of Statistics}, pages 2452--2482, April 2012.

\bibitem[Bai and Ghosh(2019)]{eyedata}
R.~Bai and M.~Ghosh.
\newblock Normalbetaprime: Normal beta prime prior.
\newblock 2019.
\newblock URL \url{https://CRAN.R-project.org/package=NormalBetaPrime}.
\newblock R package version 2.2.

\bibitem[Bao and Xiong(2021)]{Bao2021OneRoundCE}
Y.~Bao and W.~Xiong.
\newblock One-round communication efficient distributed m-estimation.
\newblock In \emph{International Conference on Artificial Intelligence and
  Statistics}, 2021.

\bibitem[Battey et~al.(2018)Battey, Jianqing, Han, Junwei, and Ziwei]{Fan2018}
H.~Battey, F.~Jianqing, L.~Han, L.~Junwei, and Z.~Ziwei.
\newblock {Distributed testing and estimation under sparse high dimensional
  models}.
\newblock \emph{The Annals of Statistics}, 46\penalty0 (3):\penalty0 1352 --
  1382, June 2018.

\bibitem[Beck and Teboulle(2009)]{BeckTeboulleFISTA}
A.~Beck and M.~Teboulle.
\newblock A fast iterative shrinkage-thresholding algorithm for linear inverse
  problems.
\newblock \emph{SIAM Journal on Imaging Sciences}, 2:\penalty0 183--202,
  January 2009.

\bibitem[Becker et~al.(2011)Becker, Bobin, and Candes]{BeckerCandesNESTA11}
S.~Becker, J.~Bobin, and E.~J. Candes.
\newblock Nesta: a fast and accurate ﬁrst-order method for sparse recovery.
\newblock \emph{SIAM Journal on Imaging Sciences}, 4:\penalty0 1--39, April
  2011.

\bibitem[Bickel et~al.(2009)Bickel, Ritov, and Tsybakov]{10.1214/08-AOS620}
P.~J. Bickel, Y.~Ritov, and A.~B. Tsybakov.
\newblock {Simultaneous analysis of Lasso and Dantzig selector}.
\newblock \emph{The Annals of Statistics}, 37\penalty0 (4):\penalty0
  1705--1732, February 2009.

\bibitem[Bolte et~al.(2009)Bolte, Daniilidis, Ley, and Mazet]{Bolte09}
J.~Bolte, A.~Daniilidis, O.~Ley, and L.~Mazet.
\newblock Characterizeations of lojasiewicz inequalities: subgradient ﬂows,
  talweg, convexity.
\newblock \emph{The Transactions of the American Mathematical Society},
  322:\penalty0 3319--3363, June 2009.

\bibitem[Bredies and Lorenz(2008)]{Bredies08}
K.~Bredies and D.~A. Lorenz.
\newblock Linear convergence of iterative soft-thresholding.
\newblock \emph{Journal of Fourier Analysis and Applications}, 14:\penalty0
  813--837, September 2008.

\bibitem[Candes and Tao(2006)]{4016283}
E.~J. Candes and T.~Tao.
\newblock Near-optimal signal recovery from random projections: universal
  encoding strategies.
\newblock \emph{IEEE Transactions on Information Theory}, 52\penalty0
  (12):\penalty0 5406--5425, January 2006.

\bibitem[Chen and Ozdaglar(2012)]{chen2012fast}
A.~I. Chen and A.~Ozdaglar.
\newblock A fast distributed proximal-gradient method.
\newblock In \emph{2012 50th Annual Allerton Conference on Communication,
  Control, and Computing (Allerton)}, pages 601--608, April 2012.

\bibitem[Chen and Sayed(2012)]{Chen-Sayed}
J.~Chen and A.~H. Sayed.
\newblock Diffusion adaptation strategies for distributed optimization and
  learning over networks.
\newblock \emph{IEEE Transactions on Signal Processing}, 60\penalty0
  (8):\penalty0 4289--4305, August 2012.

\bibitem[Chen et~al.(2021)Chen, Liu, and Zhang]{chen2021first}
X.~Chen, W.~Liu, and Y.~Zhang.
\newblock First-order newton-type estimator for distributed estimation and
  inference.
\newblock \emph{Journal of the American Statistical Association}, pages 1--17,
  2021.

\bibitem[Daneshmand et~al.(2020)Daneshmand, Scutari, and
  Kungurtsev]{Daneshmand20}
A.~Daneshmand, G.~Scutari, and V.~Kungurtsev.
\newblock Second-order guarantees of distributed gradient algorithms.
\newblock \emph{SIAM Journal on Optimization}, 30\penalty0 (4):\penalty0
  3029--3068, January 2020.

\bibitem[de~Geer and Buhlmann(2009)]{Buhlmann09}
S.~Van de~Geer and P.~Buhlmann.
\newblock On the conditions used to prove oracle results for the lasso.
\newblock \emph{Electronic Journal of Statistics}, 3\penalty0 (4):\penalty0
  1360--1392, October 2009.

\bibitem[Donoho(1995)]{382009}
D.~L. Donoho.
\newblock De-noising by soft-thresholding.
\newblock \emph{IEEE Transactions on Information Theory}, 41\penalty0
  (3):\penalty0 613--627, November 1995.

\bibitem[Duchi et~al.(2008)Duchi, Shalev-Shwartz, Singer, and
  Chandra]{duchi2008efficient}
J.~Duchi, S.~Shalev-Shwartz, Y.~Singer, and T.~Chandra.
\newblock Efficient projections onto the $\ell_1$-ball for learning in high
  dimensions.
\newblock \emph{Proceedings of the 25th International Conference on Machine
  Learning}, pages 272--279, July 2008.

\bibitem[Hale et~al.(2008)Hale, Yin, and Zhang]{Hale08}
E.~T. Hale, W.~Yin, and Y.~Zhang.
\newblock Fixed-point continuation for $\ell_1$-minimization: methodology and
  convergence.
\newblock \emph{SIAM Journal on Optimization}, 19:\penalty0 1107--1130, October
  2008.

\bibitem[Hastie et~al.(2015)Hastie, Tibshirani, and
  Wainwright]{10.5555/2834535}
T.~Hastie, R.~Tibshirani, and M.~J. Wainwright.
\newblock \emph{Statistical Learning with Sparsity: the Lasso and
  Generalizations}.
\newblock Chapman \& Hall/CRC, 2015.

\bibitem[Horn(1988)]{10.1214/aos/1176350964}
P.~S. Horn.
\newblock {On the stochastic ordering of absolute univariate Gaussian random
  variables}.
\newblock \emph{The Annals of Statistics}, 16\penalty0 (3):\penalty0
  1327--1329, September 1988.

\bibitem[Jakoveti{\'c}(2019)]{Jakovetic:da}
D.~Jakoveti{\'c}.
\newblock A unification and generalization of exact distributed first-order
  methods.
\newblock \emph{IEEE Transactions on Signal and Information Processing over
  Networks}, 5:\penalty0 31--46, September 2019.

\bibitem[Jakoveti{\'c} et~al.(2011)Jakoveti{\'c}, Xavier, and
  Moura]{jakovetic2011cooperative}
D.~Jakoveti{\'c}, J.~Xavier, and JMF. Moura.
\newblock Cooperative convex optimization in networked systems: augmented
  lagrangian algorithms with directed gossip communication.
\newblock \emph{IEEE Transactions on Signal Processing}, 59\penalty0
  (8):\penalty0 3889--3902, July 2011.

\bibitem[Jakoveti{\'c} et~al.(2013)Jakoveti{\'c}, Moura, and Xavier]{6926737}
D.~Jakoveti{\'c}, JMF. Moura, and J.~Xavier.
\newblock Linear convergence rate of a class of distributed augmented
  lagrangian algorithms.
\newblock \emph{IEEE Transactions on Automatic Control}, 60\penalty0
  (4):\penalty0 922--936, July 2013.

\bibitem[Jakoveti{\'c} et~al.(2014)Jakoveti{\'c}, Xavier, and
  Moura]{Jakoveti2014FastDG}
D.~Jakoveti{\'c}, J.~Xavier, and JMF. Moura.
\newblock Fast distributed gradient methods.
\newblock \emph{IEEE Transactions on Automatic Control}, 59:\penalty0
  1131--1146, December 2014.

\bibitem[Jianqing et~al.(2021)Jianqing, Yongyi, and Kaizheng]{Fan2021}
F.~Jianqing, G.~Yongyi, and W.~Kaizheng.
\newblock Communication-efficient accurate statistical estimation.
\newblock \emph{Journal of the American Statistical Association}, 0\penalty0
  (0):\penalty0 1--11, August 2021.

\bibitem[Jordan et~al.(2018)Jordan, Lee, and Yang]{jordan2018communication}
M.~I. Jordan, J.~D. Lee, and Y.~Yang.
\newblock Communication-efficient distributed statistical inference.
\newblock \emph{Journal of the American Statistical Association}, 2018.

\bibitem[Lee et~al.(2015)Lee, Sun, Liu, and Taylor]{lee2015communication}
J.~D. Lee, Y.~Sun, Q.~Liu, and J.E. Taylor.
\newblock Communication-efficient sparse regression: a one-shot approach.
\newblock \emph{Journal of Machine Learning Research}, January 2015.

\bibitem[Lorenzo and Scutari(2016)]{LorenzoScutari-J'16}
P.~Di Lorenzo and G.~Scutari.
\newblock Next: In-network nonconvex optimization.
\newblock \emph{IEEE Transactions on Signal and Information Processing over
  Networks}, 2:\penalty0 1--1, February 2016.

\bibitem[Luo and Tseng(1993)]{LuoTseng93}
Z.~Q. Luo and P.~Tseng.
\newblock Error bounds and convergence analysis of feasible descent methods: a
  general approach.
\newblock \emph{Annals of Operations Research}, 46:\penalty0 157--178, March
  1993.

\bibitem[Marshall et~al.(2011)Marshall, Olkin, and Arnold]{marshall11}
A.~W. Marshall, I.~Olkin, and B.~C. Arnold.
\newblock \emph{Inequalities: Theory of Majorization and Its Applications},
  volume 143.
\newblock Springer, second edition, 2011.

\bibitem[Nedi{\'c} and Ozdaglar(2009)]{Nedic2009}
A.~Nedi{\'c} and A.~Ozdaglar.
\newblock Distributed subgradient methods for multi-agent optimization.
\newblock \emph{IEEE Transactions on Automatic Control}, 54\penalty0
  (1):\penalty0 48--61, January 2009.

\bibitem[Nedi{\'c} et~al.(2010)Nedi{\'c}, Ozdaglar, and Parrilo]{Nedic2010}
A.~Nedi{\'c}, A.~Ozdaglar, and P.~A. Parrilo.
\newblock Constrained consensus and optimization in multi-agent networks.
\newblock \emph{IEEE Transactions on Automatic Control}, 55\penalty0
  (4):\penalty0 922--938, April 2010.

\bibitem[Nedi\'c et~al.(2016)Nedi\'c, Olshevsky, and Shi]{nedich2016achieving}
A.~Nedi\'c, A.~Olshevsky, and W.~Shi.
\newblock Achieving geometric convergence for distributed optimization over
  time-varying graphs.
\newblock \emph{SIAM Journal on Optimization}, 27:\penalty0 2597--2633, July
  2016.

\bibitem[Nedi\'{c} et~al.(2018)Nedi\'{c}, Olshevsky, and
  Rabbat]{Nedic_Olshevsky_Rabbat2018}
A.~Nedi\'{c}, A.~Olshevsky, and M.~G. Rabbat.
\newblock Network topology and communication-computation tradeoffs in
  decentralized optimization.
\newblock \emph{Proceedings of the IEEE}, 106:\penalty0 953--976, September
  2018.

\bibitem[Nesterov(2007)]{Nesterov2007GradientMF}
Y.~Nesterov.
\newblock Gradient methods for minimizing composite objective function.
\newblock \emph{Research Papers in Economics}, January 2007.

\bibitem[Nesterov et~al.(2018)]{nesterov2018lectures}
Y.~Nesterov et~al.
\newblock \emph{Lectures on Convex Optimization}, volume 137.
\newblock Springer, 2018.

\bibitem[Olshevsky(2017)]{olshevsky2017linear}
A.~Olshevsky.
\newblock Linear time average consensus and distributed optimization on fixed
  graphs.
\newblock \emph{SIAM Journal on Control and Optimization}, 55\penalty0
  (6):\penalty0 3990--4014, December 2017.

\bibitem[Pan and Liu(2018)]{Pan18}
S.~Pan and Y.~Liu.
\newblock Metric subregularity of subdiﬀerential and {KL} property of
  exponent 1/2.
\newblock \emph{arxiv preprint, arXiv:1812.00558v3}, 322, 2018.

\bibitem[Qu and Li(2017)]{qu2016harnessing}
G.~Qu and N.~Li.
\newblock Harnessing smoothness to accelerate distributed optimization.
\newblock \emph{IEEE Transactions on Control of Network Systems}, 5:\penalty0
  1245--1260, April 2017.

\bibitem[Raskutti et~al.(2010)Raskutti, Wainwright, and Yu]{Ras}
G.~Raskutti, M.~J. Wainwright, and B.~Yu.
\newblock Restricted eigenvalue properties for correlated gaussian designs.
\newblock \emph{Journal of Machine Learning Research}, 11:\penalty0 2241--2259,
  August 2010.

\bibitem[Raskutti et~al.(2011)Raskutti, Wainwright, and Yu]{Yu2011}
G.~Raskutti, M.~Wainwright, and B.~Yu.
\newblock Minimax rates of estimation for high-dimensional linear regression
  over $\ell_q$ -balls.
\newblock \emph{IEEE Transactions on Information Theory}, 57:\penalty0
  6976--6994, November 2011.

\bibitem[Rosenblatt and Nadler(2016)]{rosenblatt2016optimality}
J.~Rosenblatt and B.~Nadler.
\newblock On the optimality of averaging in distributed statistical learning.
\newblock \emph{Information and Inference: A Journal of the IMA}, 5\penalty0
  (4):\penalty0 379--404, 2016.

\bibitem[Sahand et~al.(2012)Sahand, Ravikumar, Wainwright, and
  Yu]{Negahban_2012}
N.~Sahand, P.~Ravikumar, M.~J. Wainwright, and B.~Yu.
\newblock A unified framework for high-dimensional analysis of $m$-estimators
  with decomposable regularizers.
\newblock \emph{Statistical Science}, 27\penalty0 (4):\penalty0 538–557,
  November 2012.

\bibitem[Sayed(2014)]{Sayed}
A.~H. Sayed.
\newblock Adaptation, learning, and optimization over networks.
\newblock \emph{Foundations and Trends in Machine Learning}, 7:\penalty0
  311--801, January 2014.

\bibitem[Shamir et~al.(2014)Shamir, Srebro, and Zhang]{shamir2014communication}
O.~Shamir, N.~Srebro, and T.~Zhang.
\newblock Communication-efficient distributed optimization using an approximate
  newton-type method.
\newblock In \emph{International Conference on Machine Learning}, pages
  1000--1008. PMLR, 2014.

\bibitem[Shi et~al.(2014)Shi, Ling, Yuan, Wu, and Yin]{shi2014linear}
W.~Shi, Q.~Ling, K.~Yuan, G.~Wu, and W.~Yin.
\newblock On the linear convergence of the admm in decentralized consensus
  optimization.
\newblock \emph{IEEE Transactions on Signal Processing}, 62:\penalty0
  1750--1761, July 2014.

\bibitem[Shi et~al.(2015{\natexlab{a}})Shi, Ling, Wu, and Yin]{shi2015extra}
W.~Shi, Q.~Ling, G.~Wu, and W.~Yin.
\newblock {EXTRA: A}n exact first-order algorithm for decentralized consensus
  optimization.
\newblock \emph{SIAM Journal on Optimization}, 25\penalty0 (2):\penalty0
  944--966, November 2015{\natexlab{a}}.

\bibitem[Shi et~al.(2015{\natexlab{b}})Shi, Ling, Wu, and Yin]{shi2015proximal}
W.~Shi, Q.~Ling, G.~Wu, and W.~Yin.
\newblock A proximal gradient algorithm for decentralized composite
  optimization.
\newblock \emph{IEEE Transactions on Signal Processing}, 63\penalty0
  (22):\penalty0 6013--6023, November 2015{\natexlab{b}}.

\bibitem[Sun et~al.(2019)Sun, Daneshmand, and Scutari]{sun2019distributed}
Y~Sun, A~Daneshmand, and G~Scutari.
\newblock Distributed optimization based on gradient-tracking revisited:
  enhancing convergence rate via surrogation.
\newblock \emph{arXiv preprint, arXiv:1905.02637}, 2019.

\bibitem[Sun et~al.(2022)Sun, M., Scutari, and C.]{NetLASSO}
Y~Sun, Maros M., G~Scutari, and Guang C.
\newblock High-dimensional inference over networks: linearly convergence
  algorithms and statistical guarantees.
\newblock \emph{arXiv:2201.08507}, 2022.

\bibitem[Tibshirani(1996)]{Tibshirani96}
R.~Tibshirani.
\newblock Regression shrinkage and selection via the lasso.
\newblock \emph{Journal of the Royal Statistical Society Series B},
  58:\penalty0 267--288, January 1996.

\bibitem[Tseng and Yun(2009)]{TsengYun09}
P.~Tseng and S.~Yun.
\newblock A coordinate gradient descent method for nonsmooth separable
  minimization.
\newblock \emph{Mathematical Programming}, 117:\penalty0 387--423, March 2009.

\bibitem[Vershynin(2012)]{Vershynin2012IntroductionTT}
R.~Vershynin.
\newblock Introduction to the non-asymptotic analysis of random matrices.
\newblock In \emph{Compressed Sensing}, 2012.

\bibitem[Wainwright(2019)]{Wainwright-book}
M.~J. Wainwright.
\newblock \emph{High-Dimensional Statistics: A Non-asymptotic Viewpoint}.
\newblock Cambridge University Press, 2019.

\bibitem[Wang et~al.(2017)Wang, Kolar, Srebro, and Zhang]{wang2017efficient}
J.~Wang, M.~Kolar, N.~Srebro, and T.~Zhang.
\newblock Efficient distributed learning with sparsity.
\newblock In \emph{International Conference on Machine Learning}, pages
  3636--3645. PMLR, May 2017.

\bibitem[Wang et~al.(2018)Wang, Roosta, Xu, and Mahoney]{wang2018giant}
S.~Wang, F.~Roosta, P.~Xu, and W.~W. Mahoney.
\newblock Giant: Globally improved approximate newton method for distributed
  optimization.
\newblock \emph{Advances in Neural Information Processing Systems}, 31, 2018.

\bibitem[Wen et~al.(2017)Wen, Chen, and Pong]{Wen17}
B.~Wen, X.~Chen, and T.~Pong.
\newblock Linear convergence of proximal gradient algorithm with extrapolation
  for a class of nonconvex nonsmooth minimization problems.
\newblock \emph{SIAM Journal on Optimization}, 27:\penalty0 124--145, December
  2017.

\bibitem[Xu et~al.(2018)Xu, Zhu, Soh, and Xie]{Xu-TAC:hs}
J.~Xu, S.~Zhu, Y.~C. Soh, and L.~Xie.
\newblock {Convergence of asynchronous distributed gradient methods over
  stochastic networks}.
\newblock \emph{IEEE Transactions on Automatic Control}, 63\penalty0
  (2):\penalty0 434--448, July 2018.

\bibitem[Yuan et~al.(2016)Yuan, Ling, and Yin]{Yuan_2016}
K.~Yuan, Q.~Ling, and W.~Yin.
\newblock On the convergence of decentralized gradient descent.
\newblock \emph{SIAM Journal on Optimization}, 26\penalty0 (3):\penalty0
  1835–1854, January 2016.

\bibitem[Yuan et~al.(2020)Yuan, Alghunaim, Ying, and Sayed]{Yuan2020}
K.~Yuan, S.~Alghunaim, B.~Ying, and A.~H. Sayed.
\newblock On the influence of bias-correction on distributed stochastic
  optimization.
\newblock \emph{IEEE Transactions on Signal Processing}, 68:\penalty0
  4352--4367, July 2020.

\bibitem[Zeng and Yin(2018)]{WYin_ncvxDGD_SIAM2018}
J.~Zeng and W.~Yin.
\newblock On nonconvex decentralized gradient descent.
\newblock \emph{IEEE Transactions on Signal Processing}, 66\penalty0
  (11):\penalty0 2834--2848, June 2018.

\bibitem[Zhou and So(2017)]{ZhouSo17}
Z.~Zhou and A.~Man-Cho So.
\newblock A uniﬁed approach to error bounds for structured convex
  optimization problems.
\newblock \emph{Mathematical Programming}, 165:\penalty0 689--728, December
  2017.

\end{thebibliography}

\end{document}